\documentclass[12pt,a4paper]{article}
\usepackage[utf8]{inputenc}
\usepackage{booktabs}
\usepackage{geometry} 
\usepackage[T1]{fontenc}
\usepackage{booktabs}
\usepackage{array}
\usepackage{ragged2e}

\usepackage{mathtools}
\usepackage{caption}
\usepackage{array}
\usepackage{amsmath,amssymb,amsthm}
\usepackage{algorithm}
\usepackage{algorithmic}
\usepackage{graphicx}
\usepackage{booktabs}
\usepackage{mathtools}
\usepackage{caption}
\usepackage{hyperref}
\usepackage{geometry}
\usepackage{float}
\usepackage{tikz}
\usetikzlibrary{arrows.meta,positioning,fit,calc,shapes.geometric,shapes.multipart}
\usetikzlibrary{positioning,calc,arrows.meta}
\usetikzlibrary{positioning,arrows.meta,fit,backgrounds,decorations.pathreplacing}
\tikzset{
  arrow/.style={-Latex, line width=0.8pt},
  block/.style={draw, rounded corners=2pt, align=center, minimum height=6mm, inner sep=2pt, font=\small},
  op/.style={block, fill=gray!10},
  var/.style={block, fill=blue!5},
  gate/.style={block, fill=orange!12},
  delay/.style={block, fill=yellow!15},
  legendbox/.style={draw, rounded corners=2pt, inner sep=2pt, font=\scriptsize, fill=white},
  lbl/.style={font=\scriptsize, inner sep=1pt}
}
\usepackage{pgfplots}
\usepackage{pgfplots}
\usepackage{pgfplotstable}
\pgfplotsset{compat=1.18}
\usepackage{subcaption}
\usepackage{booktabs}
\usepackage[utf8]{inputenc}

\definecolor{colNOS}{RGB}{31,119,180} % Blue
\definecolor{coltGNN}{RGB}{255,127,14} % Orange
\definecolor{colGRU}{RGB}{44,160,44} % Green
\definecolor{colRNN}{RGB}{214,39,40} % Red
\definecolor{colMLP}{RGB}{148,103,189} % Purple
\usepackage{authblk}
\newcommand{\Eta}{\mathrm{E}}
\usepackage{subcaption}
\usepackage[english]{babel}
\newtheorem{theorem}{Theorem}
\newtheorem{corollary}{Corollary}[theorem]
\newtheorem{lemma}[theorem]{Lemma}

\pgfplotsset{compat=newest}
\title{Network-Optimised Spiking Neural Network for Event-Driven Networking}
\author{Muhammad Bilal \\School of Computing and Communications, Lancaster University, LA1 4WA Lancaster, United Kingdom\\ m.bilal@ieee.org }

\date{}

\begin{document}
\maketitle
\begin{abstract}
Delay-coupled systems often require low-latency decisions from sparse telemetry, where dense fixed-step neural inference is wasteful and can degrade near stability margins. We introduce Network-Optimised Spiking (NOS) \footnote{\url{https://mbilal84.github.io/nos-snn-networking/}}, a trainable two-state event-driven dynamical unit for delayed, graph-coupled streams, whose states map to a fast load variable and a slower recovery resource. NOS uses bounded excitability for finite buffers, explicit leak terms for service and damping, and graph-local coupling with per-link gates and communication delays, with differentiable resets compatible with surrogate-gradient training and neuromorphic execution. We prove existence and uniqueness of subthreshold equilibria, derive Jacobian-based stability conditions, and obtain a scalar network stability threshold that separates topology from node dynamics via a Perron-mode spectral condition. A stochastic arrival model aligned with telemetry smoothing explains increased variability as systems approach stability boundaries. On delayed graph forecasting and early-warning tasks from queue telemetry, NOS improves detection F1 and detection latency over MLP, RNN/GRU, and temporal GNN baselines under a common residual-based protocol, while providing calibration rules for resource-constrained deployments.
\end{abstract}

\section*{keywords}
Event-driven computation, spiking neural networks, delay-coupled dynamical systems, resource-constrained inference, spectral stability, networked control, queueing dynamics, distributed systems

\section{Introduction}\label{sec:intro}
Machine learning now underpins core networking tasks such as short-horizon traffic forecasting, anomaly detection, and routing optimisation. Standard deep models including multilayer perceptrons, recurrent networks, and graph neural networks are effective when training data are abundant and the inference budget can be treated as fixed. In operational networks, however, telemetry is often sparse and bursty, feedback is delayed, and performance can change sharply near stability margins. In this regime, dense fixed-step inference can waste computation and can also amplify timing errors, which matters when the target is early warning and safe operation rather than only average accuracy.

Spiking neural networks (SNNs) offer a complementary computational regime. Their event-driven updates can be sparse in time and energy, which aligns with packet-level dynamics at the edge and with tight power envelopes. Surveys and neuromorphic results report substantial efficiency gains for spiking workloads on dedicated substrates, encouraging designs that connect spiking dynamics to system semantics rather than to abstract latent states \cite{Chowdhury2025,Davies2018,Tia2019}. Foundational work on temporal coding and the computational power of spikes also supports spike-based processing for time-critical inference and control \cite{IJCNN2000}. In parallel, platforms such as Intel's Loihi and its toolchain demonstrate practical workflows for SNN deployment, including on-chip learning and synaptic delays \cite{Davies2018,LoihiToolchain2018}.

Despite this promise, classical SNN formulations are often a poor fit for networking. Many models prioritise biological analogy and introduce internal states that do not correspond to measurable quantities such as queue occupancy, service rate, or link delay \cite{GUO2023,Emre2019}. Their hard threshold-and-reset mechanisms create non-differentiable discontinuities that complicate gradient-based optimisation, which remains the main workhorse of modern machine learning \cite{Abbott2016,Zenke2014}. Topology and per-link heterogeneity are also frequently handled implicitly or uniformly, even though real networks are weighted, directed, and delay-bearing. Finally, common excitable dynamics can be effectively unbounded under sustained drive, producing unrealistic behaviour and numerical instability under high load \cite{Yao2024}. These issues are compounded by practical constraints on encoding, training cost, and deployment on resource-constrained neuromorphic or edge hardware, which helps explain why many evaluations remain disconnected from real topologies and traffic patterns \cite{GUO2023,yang2022}. Related dynamical models including Hodgkin--Huxley \cite{Hausser2000}, Izhikevich \cite{Izhikevich2003}, and G-networks \cite{Gelenbe1994} either impose heavy dynamics or do not provide a graph-aware, differentiable design with explicit queue semantics.

We address these gaps by introducing \emph{Network-Optimised Spiking} (NOS), a compact two-state event-driven unit designed for delayed, graph-coupled telemetry streams in networks. NOS maps its state directly to normalised queue occupancy and a recovery resource, uses bounded excitability to respect finite buffers, includes explicit leak terms for service and damping, and supports graph-local coupling with optional per-link gates and communication delays. Crucially, NOS replaces hard resets with differentiable reset mechanisms that preserve event timing while enabling surrogate-gradient training \cite{Zenke2014}. We also motivate deployment on neuromorphic substrates where synaptic delays and event-driven execution are native \cite{Davies2018}. Related SNN studies in robotics and complex-network settings consider topology and control, but their objectives differ from packet dynamics and queue semantics, and they do not provide the calibration hooks needed for network operation \cite{ACM2025Path,ACMISL2024,AICI2025}. An extended discussion of when SNNs are preferable to conventional deep and graph models for networking tasks is given in Section~\ref{subsec:snn-scope}, while Section~\ref{sec:classicsnn} reviews classical spiking paradigms and evaluates their suitability for queue semantics, topology, and trainability.

\paragraph{Contributions and findings.}
(i) \textbf{A semantics-aligned spiking unit:} we propose a two-state NOS model whose variables and parameters map to measurable network quantities, with bounded excitability, explicit service and damping leaks, and graph-local delayed coupling with optional per-link gating.
(ii) \textbf{Trainability via differentiable resets:} we introduce two differentiable reset strategies, an event-triggered exponential soft reset and a continuous pullback shaped by a sigmoid, to support surrogate-gradient optimisation while preserving spiking behaviour.
(iii) \textbf{Stability analysis at node and network scale:} we prove existence and uniqueness of subthreshold equilibria, derive Jacobian-based local stability tests, and obtain a scalar network stability threshold that scales with the Perron eigenvalue of the coupling matrix; saturation enlarges the stable region by reducing the effective excitability slope.
(iv) \textbf{Operational guidance:} we derive an operational margin that links damping to offered load, yielding a simple two-parameter contour for planning and capacity checks, and we provide calibration rules intended for delay-sensitive, resource-constrained deployments.

% \section*{Results}\label{sec2}

\section{Proposed model: Network-Optimised Spiking neural network (\textit{NOS})}
\label{sec:nos}

Networking workloads are shaped by graph structure, finite buffers, and service constraints, and they must be trainable from telemetry. We therefore begin with graph-local inputs that respect neighbourhoods and link heterogeneity,
\begin{align}
I_i(t) &= \sum_{j\in\mathcal{N}(i)} w_{ij}\,S_j(t),
\label{eq:nos_start_input}
\end{align}
where \(S_j(t)\) is a presynaptic event train and \(w_{ij}\) encodes link capacity, policy weight, or reliability. Three implementation choices motivate the \emph{NOS} design. First, the reset must be continuous and differentiable so gradient-based training is possible. Second, state variables and parameters should map to observables such as normalised queue length, service rate, and recovery time. Third, subthreshold integration needs explicit stochastic drive so gradual load accumulation and burstiness are both represented. 

Implementation-level choices, detailed calibration rules, and additional parameter-tuning heuristics, including the pipeline in Fig.~\ref{fig:nos-pipeline}, are summarised in  \emph{methods}~\ref{sec:nos-design-guidance}.

\subsection{Variable reinterpretation}
\label{subsec:nos-reinterpretation}

We define a two-variable \emph{NOS} unit whose excitability and recovery states encode the short-term behaviour of a networked queue. The fast state \(v_i(t)\) represents a normalised queue or congestion level, with \(v_i\!\in[0,1]\) mapping empty to full buffer. The slow state \(u_i(t)\) represents a recovery or slowdown resource that builds during bursts and relaxes thereafter.

\(
v \;\mapsto\; \text{normalised queue length or congestion level}
\)

\(
u \;\mapsto\; \text{recovery or slowdown resource}.
\)

Here \(v=1\) corresponds to a full buffer. The state \(u\) summarises pacing, token replenishment, or rate-limiter cool-down that follows a burst. This choice allows direct initialisation from traces: \(v\) from queue occupancy, \(u\) from measured relaxation time, and the link between them from decay after micro-bursts. It also clarifies units. Both \(v\) and \(u\) evolve in the same time scale, which keeps linearisation and stability analysis transparent. 

\subsection{State variables and dynamics}
\label{subsec:nos-dynamics}
For each node $i$ we adopt a bounded excitability function and explicit damping:
\begin{align}
\begin{split}
\frac{dv_i}{dt} = f_{\mathrm{sat}}(v_i) + \beta v_i + \gamma - \xi\,u_i + I_i(t) - \lambda v_i - \\ \chi \bigl(v_i - v_{\mathrm{rest}}\bigr),
\label{eq:nos-v-eq}
\end{split}
\end{align}

\begin{align}
\begin{split}
\frac{du_i}{dt} &= a \bigl(b v_i - u_i\bigr) - \mu u_i \;=\; a b\,v_i - (a+\mu)\,u_i .
\label{eq:nos-u-eq}
\end{split}
\end{align}
Here $v$ is the normalised queue level (dimensionless). The recovery state $u$ is also taken dimensionless, and the scalar $\xi>0$ maps recovery into a rate contribution in \eqref{eq:nos-v-eq} (so $\xi$ has units of inverse time, or per-bin under a discrete time base). In all experiments we absorb this scaling into the definition of $u$ and fix $\xi=1$, so $\xi$ is omitted from subsequent notation.
Throughout, time $t$ is measured in sampling bins of width $\Delta t$ (default $\Delta t = 5\,\mathrm{ms}$).
Derivatives $d(\cdot)/dt$ are therefore \emph{per-bin} rates, and all coefficients in \eqref{eq:nos-v-eq}--\eqref{eq:nos-u-eq}
are dimensionless per-bin quantities.
Physical rates in $\mathrm{s}^{-1}$ are obtained by
\(
r_{\mathrm{phys}} = \frac{r}{\Delta t}.
\)

with bounded excitability
\begin{equation}
f_{\mathrm{sat}}(v) \;=\; \frac{\alpha v^2}{1+\kappa v^2},
\label{eq:f_sat_def_main}
\end{equation}
Equation \eqref{eq:nos-v-eq} separates four operational effects. The term \(f_{\mathrm{sat}}(v_i)\) (see \eqref{eq:nos-fsat}) captures the convex rise of backlog during rapid arrivals while remaining bounded. The linear part \(\beta v_i+\gamma\) fits residual slope and offset that are visible in practice after removing coarse trends. The coupling \(-\xi u_i\) models recovery drag that slows re-accumulation immediately after events. Finally, \(-\lambda v_i - \chi(v_i-v_{\mathrm{rest}})\) represents service and small-signal damping. Equation \eqref{eq:nos-u-eq} integrates recent congestion with sensitivity \(ab\) and relaxes at rate \(a+\mu\). Thus, the \textit{NOS} model reflects three intuitive aspects of queue behaviour: the rise under increasing load, the draining through service, and the short-lived slowdown that follows bursts. Fig.~\ref{fig:nos-unit-schematic} illustrates a single \emph{NOS} unit, highlighting its graph-local inputs, excitability dynamics, recovery variable, stochastic threshold, and differentiable reset pathway. Spike generation and reset principles are discussed in \emph{methods}~\ref{subsec:nos-resets}. Operational interpretations and data-driven initialisation rules for all \textit{NOS} parameters are summarised in Table~\ref{tab:param_combined} in \emph{methods}~\ref{defvals}, with further design guidance in \emph{methods}~\ref{sec:nos-design-guidance}.

%%%%%% \emph{NOS} Unit
\begin{figure}[t]
\centering
% Width locked to \textwidth (A4-safe); height grows to fit
\resizebox{\textwidth}{!}{%
\begin{tikzpicture}[x=0.12\textwidth, y=0.08\textwidth, >=stealth]

% ---- styles ----
\tikzstyle{blk}=[draw, rounded corners=2pt, minimum width=2.8cm, minimum height=1.0cm, align=center]
\tikzstyle{blkfill}=[blk, fill=gray!6]
\tikzstyle{sum}=[draw, circle, minimum size=0.7cm, inner sep=0pt]
\tikzstyle{dbl}=[draw, double, double distance=0.6pt, rounded corners=2pt, minimum width=4.5cm, minimum height=1.4cm, align=center, fill=gray!10]
\tikzstyle{txt}=[font=\small]

% ---- Row 1: presynaptic sources (top) ----
\node[blk] (s1) at (0,10) {$S_{j_1}(t)$};
\node[blk] (s2) at (2.5,10) {$S_{j_2}(t)$};
\node[blk] (s3) at (5,10) {$S_{j_3}(t)$};

% ---- Row 2: per-link gates and weights ----
\node[blk] (g1) at (0,8.5) {$w_{ij_1}\,g\!\big(q_{ij_1}\big)$};
\node[blk] (g2) at (2.5,8.5) {$w_{ij_2}\,g\!\big(q_{ij_2}\big)$};
\node[blk] (g3) at (5,8.5) {$w_{ij_3}\,g\!\big(q_{ij_3}\big)$};

% ---- Row 3: delays ----
\node[blk] (d1) at (0,7) {delay $\tau_{ij_1}$};
\node[blk] (d2) at (2.5,7) {delay $\tau_{ij_2}$};
\node[blk] (d3) at (5,7) {delay $\tau_{ij_3}$};

% arrows from sources to gates to delays
\draw[->] (s1.south) -- (g1.north);
\draw[->] (g1.south) -- (d1.north);
\draw[->] (s2.south) -- (g2.north);
\draw[->] (g2.south) -- (d2.north);
\draw[->] (s3.south) -- (g3.north);
\draw[->] (g3.south) -- (d3.north);

% ---- Row 4: summation and exogenous noise ----
\node[sum] (sumI) at (2.5,5.5) {$\sum$};
\node[txt] at (2.5,6.2) {$I_i(t)$};

% arrows from delays to summation (curved to avoid crossing)
\draw[->] (d1.south) .. controls (0.5,6.2) and (1.5,6.0) .. (sumI.west);
\draw[->] (d2.south) -- (sumI.north);
\draw[->] (d3.south) .. controls (4.5,6.2) and (3.5,6.0) .. (sumI.east);

\node[blk] (eta) at (5.5,5.5) {$\eta_i(t)$};
\draw[->] (eta.west) -- (sumI.east);

% annotation of I_i(t)
\node[txt, align=left, anchor=west] at (6.5,6.2)
{$I_i(t)=\displaystyle\sum_{j\in\mathcal{N}(i)} w_{ij}\,g\!\big(q_{ij}(t)\big)\,S_j\!\big(t-\tau_{ij}\big)\;+\;\eta_i(t)$};

% ---- Row 5: v-dynamics (centre lane) ----
\node[blkfill, minimum width=8.5cm, minimum height=1.8cm, anchor=center] (vdyn) at (2.5,3.8)
{\begin{minipage}{0.65\textwidth}\centering
$\dot v_i = f_{\mathrm{sat}}(v_i) + \beta v_i + \gamma - u_i + I_i(t) - \lambda v_i - \chi\big(v_i - v_{\mathrm{rest}}\big)$\\[3pt]
$f_{\mathrm{sat}}(v)=\dfrac{\alpha v^2}{\,1+\kappa v^2\,}$
\end{minipage}};
\draw[->] (sumI.south) -- (vdyn.north);

% ---- Row 6: u-dynamics ----
\node[blkfill, minimum width=8.5cm, minimum height=1.2cm, anchor=center] (udyn) at (2.5,2.2)
{$\dot u_i = a\big(b\,v_i - u_i\big) - \mu\,u_i$};

% u -> v routed on LEFT lane
\coordinate (uLeft) at ($(udyn.west)+(-0.8,0)$);
\coordinate (vLeft) at ($(vdyn.west)+(0,-0.1)$);
\draw[->] (udyn.west) -- (uLeft) |- (vLeft);
\node[txt, anchor=west] at ($(vLeft)+(0.1,0.15)$) {$-\,u_i$};

% ---- Row 7: threshold (diamond, lowered to avoid overlap) ----
\node[draw, diamond, aspect=2.0, minimum width=3.8cm, minimum height=1.0cm, anchor=center] (th) at (2.5,0.3)
{\begin{minipage}{2.5cm}\centering
$\,v_{\mathrm{th}}(t)=v_{\mathrm{th,base}}+\sigma\,\xi(t)\,$
\end{minipage}};

% v -> th on CENTRE lane
\draw[->] (vdyn.south) -- (th.north);

% ---- Row 8: soft reset (bottom) ----
\node[dbl, anchor=center] (rst) at (2.5,-2.2)
{\begin{minipage}{0.45\textwidth}\centering
\textbf{soft reset}\\[2pt]
$\,v \leftarrow c + (v-c)\,e^{-r_{\text{reset}}\,\Delta t},\qquad u \leftarrow u + d$\\[2pt]
{\scriptsize or continuous pullback: $-\,r_{\text{reset}}\,\sigma_{\kappa_\sigma}\!\big(v-v_{\mathrm{th}}\big)\,(v-c)$ in $\dot v$}
\end{minipage}};

% wiring: v -> threshold -> reset
\draw[->] (th.south) -- (rst.north);

% spike output (right, straight)
\draw[->] (th.east) -- ++(3.5,0) node[txt, right] {$S_i(t)$};

% reset feedbacks: v on RIGHT lane, u on LEFT lane
\draw[->] (rst.east) .. controls (6.8,-2.2) and (6.8,4.2) .. (vdyn.east)
  node[txt, right, xshift=0.2cm, yshift=-0.5cm] {$v \leftarrow c + (v-c)\,e^{-r_{\text{reset}}\,\Delta t}$};

\draw[->] (rst.west) .. controls (-1.8,-2.2) and (-1.8,2.5) .. (udyn.west)
  node[txt, left, xshift=-0.2cm, yshift=0.8cm] {$u \leftarrow u + d$};

% ---- legend ----
\node[txt, align=left, anchor=west] at (0,-3.6)
{Per-link weight and gate $w_{ij}\,g(q_{ij})$, delay $\tau_{ij}$, and shot-noise $\eta_i(t)$ feed a node with bounded excitability $f_{\mathrm{sat}}$,\\
recovery $u_i$, stochastic threshold, and differentiable reset. The threshold emits $S_i(t)$; reset updates feed back into $v$ and $u$.};

\end{tikzpicture}%
}
\caption{NOS unit with graph-local inputs, weighted per-link gates and delays, exogenous shot-noise, bounded excitability, recovery dynamics, stochastic threshold, and differentiable reset. Three non-overlapping lanes: left for $u\!\to\!v$, centre for $v\!\to\!v_{\mathrm{th}}$, right for reset $\to v$.}
\label{fig:nos-unit-schematic}
\end{figure}
\subsection{Bounded excitability}
\label{subsec:nos-bounded-excitability}

We replace the unbounded quadratic drive in Izhikevich model by a saturating nonlinearity
\begin{align}
f_{\mathrm{sat}}(v) &= \frac{\alpha v^2}{1+\kappa v^2}, \qquad \alpha>0,\ \kappa>0.
\label{eq:nos-fsat}
\end{align}
For small queues this behaves approximately quadratically, while for large queues it saturates at \(\alpha/\kappa\), which enforces a finite growth rate and reflects finite buffer capacity.
Detailed small–signal approximations, derivative bounds, and their use in stability margins are given in \emph{methods}~\ref{app:fsat-bounds}.

\subsection{Graph-local inputs, explicit delays, and per-link queues}
\label{subsec:nos-graph}

Topology enters through delayed presynaptic events (neighbourhood wiring, per-link delays, and optional link-state gates that modulate influence during congestion, scheduling, or wireless fading). Starting from \eqref{eq:nos_start_input}, we incorporate delays and exogenous noise as
\begin{align}
I_i(t) &= \sum_{j\in\mathcal{N}(i)} w_{ij}\,S_j\bigl(t-\tau_{ij}\bigr) + \eta_i(t),
\label{eq:nos-input-delayed}
\end{align}
where \(\mathcal{N}(i)\) is the graph neighbourhood of node \(i\), \(S_j(t-\tau_{ij})\) is the presynaptic event train as it arrives at node \(i\), \(w_{ij}\) encodes a link’s nominal capacity, policy weight, or reliability, \(\tau_{ij}\) is a link delay, and \(\eta_i(t)\) is stochastic drive. 
The statistical properties of these stochastic drives, and their impact on variance and delay tails, are quantified through experiments in the \emph{methods}, particularly \emph{methods}s~\ref{noise_avalanche} and~\ref{subsec:mm1}.
In this work we treat $\eta_i(t)$ as a bursty arrival process and obtain $\tau_{ij}$ and $w_{ij}$ from simple measurements such as round trip times and link utilisation statistics. The full statistical specification of the exogenous drive, the calibration of delays and weights, and the optional link state gates based on $q_{ij}(t)$ are given in \emph{methods}~\ref{app:nos-stochastic} and~\ref{sec:delayed-gates}.

\section{Equilibrium and local stability analysis}
\label{sec:stability-local}

Using the properties of the bounded excitability function in \emph{methods}~\ref{app:fsat-bounds}, the existence and uniqueness conditions for the subthreshold equilibrium and the detailed stability proofs are developed in \emph{methods}~\ref{app:stability-details}. 
At equilibrium a \textit{NOS} node holds a steady queue level set by arrivals, service, and graph-local coupling. Solving the stationarity conditions $\dot v_i=\dot u_i=0$ yields a scalar fixed-point equation balancing bounded excitability, linear drain, and a constant baseline from $\gamma+\chi v_{\mathrm{rest}}+I_i^*$. Under the same service–dominance requirement as in Corollary~\ref{cor:nos-algebra}, this equation admits a unique operating point $v_i^*\in[0,1]$ for each admissible mean drive $I_i^*$; the full derivation and small-signal approximations are given in \emph{methods}~\ref{app:equilibrium-eq}.

\subsection{Equilibrium equations}
\label{subsec:equilibrium-eq}

For node $i$ the subthreshold equilibrium $(v_i^*,u_i^*)$ is obtained by solving $\dot v_i=\dot u_i=0$. Eliminating $u_i^*$ gives a scalar balance between bounded excitability, linear drain and a baseline term that absorbs $\gamma$, $v_{\mathrm{rest}}$ and the mean drive $I_i^*$. Under a mild service dominance condition the resulting fixed point is unique in $[0,1]$ for each admissible $I_i^*$ (\emph{methods}~\ref{app:equilibrium-eq}).

\subsection{Jacobian and linear stability}
\label{sec:Jstability-local}

Linearising a \textit{NOS} node about a subthreshold equilibrium $(v^*,u^*)$ produces a $2\times 2$ Jacobian whose entries depend on the small signal slope $f'(v^*)$ and the coefficients $(\lambda,\chi,a,b,\mu,\beta)$. Standard trace and determinant inequalities then yield local stability conditions that tie the damping provided by $\lambda$ and $\chi$ to the maximum admissible slope of $f_{\mathrm{sat}}$ (\emph{methods}~\ref{app:Jstability-local}).

\subsection{Network coupling and global stability}
\label{sec:global_stability}

In isolation a \textit{NOS} unit settles to a load–service balance set by its local parameters. In a network the steady input $I_i^*$ is shaped by neighbours and policy weights, so existence and robustness of equilibria depend on both the operating point and the coupling matrix $W$. Bounds on $I_i^*$ in terms of matrix norms of $W$ lead to sufficient existence conditions that expose how heavy fan–in or large policy weights shrink headroom for a given steady drive.

A convenient scalar “operational margin” summarises these existence bounds in terms of subthreshold damping $\chi$ and the maximum steady input $I_{\max}=\max_i I_i^*$. Positive margin corresponds to safe operation with headroom for burstiness; margin near zero signals proximity to loss of equilibrium. The trade-off between damping and sustained load is illustrated in Fig.~\ref{fig:op_heat}, while the full inequalities and derivations are given in \emph{methods}~\ref{app:global_stability}, and the associated bifurcation
structure and parameter trends (saddle--node versus Hopf onsets) are
summarised in \emph{methods}~\ref{sec:bifurcationOP}.

\begin{figure}[t]
\centering
\includegraphics[width=.85\linewidth]{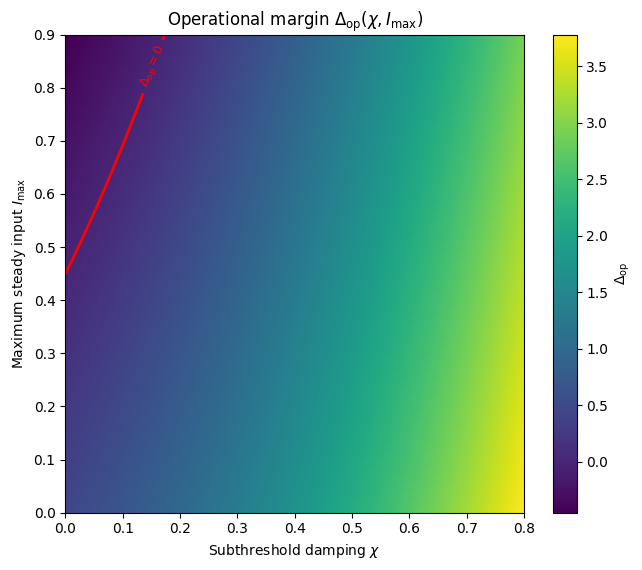}
\caption{Operational margin $\Delta_{\mathrm{op}}(\chi, I_{\max})$ as a function of subthreshold damping $\chi$ and maximum steady input $I_{\max}=\max_i I_i^*$. The red contour marks $\Delta_{\mathrm{op}}=0$; points above this curve (larger $I_{\max}$ for the same $\chi$) lie beyond the small-signal existence bound. Parameters $\{\alpha,\kappa,\lambda,\beta,a,b,\mu,\gamma,v_{\mathrm{rest}}\}$ are held fixed as in the text. 
% \emph{methods}~\ref{app:global_stability} gives the analytical definition of $\Delta_{\mathrm{op}}$ and one-dimensional sweeps versus $I_{\max}$, together with sensitivity to $\rho(W)$ and $g$.
}
\label{fig:op_heat}
\end{figure}

\subsection{Block Jacobian, network spectrum, and stability threshold}
\label{sec:stability-threshold}

Linearising the coupled \textit{NOS} network about a subthreshold equilibrium yields a $2N\times 2N$ block Jacobian, whose network contribution enters through the linearised drive sensitivity $\delta I \approx G\,\delta v$. Under homogeneous scaling, one may take $G=gW$. A corresponding vector-norm bound is given in methods~\ref{app:dI-G}.
Projecting onto the Perron mode of $W$ reduces the stability check to a $2\times 2$ effective Jacobian (methods~\ref{app:stability-threshold}). Routh--Hurwitz then yields a scalar coupling threshold in terms of $k:=g\,\rho(W)$:
\begin{equation}
g\,\rho(W) \;<\; \Lambda - f'_{\mathrm{sat}}(v^*) ,
\label{eq:net-small-signal-main}
\end{equation}
equivalently $g\,\rho(W)+f'_{\mathrm{sat}}(v^*)<\Lambda$, where $\Lambda$ is the net drain defined from the local coefficients (cf.~\eqref{eq:nos-cor-service}). Parameters are non-dimensionalised so that $T\,\Lambda$ and $T\,f'_{\mathrm{sat}}(v^*)$ are dimensionless and the operational margin \(
\Delta_{\mathrm{net}} := \Lambda - f'_{\mathrm{sat}}(v^*) - g\,\rho(W). \) is comparable across sites (methods~\ref{app:nos-scaling}). This separates node physics from topology and matches the operational proxy used in Section~\ref{sec:net-stability}.

 The full block Jacobian, Perron-mode reduction, numerical spectral validation, and Gershgorin-style bounds are given in \emph{methods}~\ref{app:stability-threshold} and \ref{sec:gershgorin}. Finally, \emph{methods}~\ref{sec:bifurcationOP} characterises the saddle--node and Hopf onsets and shows that the instability threshold obeys the spectral collapse relation in~\eqref{eq:collapse_gstar}.

\section{Network-level stability and topology dependence}
\label{sec:net-stability}

We simulate \textit{NOS} networks with an event driven integrator that updates delayed spikes, state variables and thresholds for every node at each step. Pseudocode and implementation details, including optional stochastic arrivals and avalanche counting, are provided in the  \emph{methods}~\ref{pseudocode}.

\subsection{Topologies and experiments}
\label{sec:net-topology}
We tested \emph{NOS} on large networks of size $N=250$ with different 
canonical topologies:

\paragraph{Chain topology.} Nodes are arranged in a linear sequence, 
representing serial bottlenecks such as routers on a backbone link. 
Propagation is slowed by distance, requiring stronger coupling for instability.

\paragraph{Star topology.} One hub connects to all leaves, representing 
a single aggregation point such as a central switch. 
This configuration is fragile, as overload of the hub destabilises 
the entire structure.

\paragraph{Scale-free topology.} Networks generated by the 
Barabási–Albert model capture heterogeneous connectivity with hubs and peripheries. This structure shows intermediate behaviour: hubs can initiate cascades, but leaves constrain their spread.

Figure~\ref{fig:NOS_panels} summarises the network-level phenomena captured by \textit{NOS}. Panel~(a) shows that when link delays are included ($\tau_{ij}>0$), the stability boundary shifts: oscillations emerge at lower coupling $k$, reflecting how communication latency destabilises otherwise stable traffic.  
Panel~(b) compares absolute resilience across topologies: chains require stronger coupling to destabilise, stars are fragile because of their single hub, and scale-free networks lie in between.  
Panel~(b$''$) reframes this relative to the Perron spectral prediction. Here chains deviate strongly, sitting well above the $1/\rho(W)$ baseline, while stars and scale-free graphs align more closely. This means that spectral tools capture hub-dominated topologies reliably but underestimate the resilience of homogeneous chains.  
Additional synchrony measures based on a Kuramoto type order parameter and finite size trends are reported in \emph{methods}~\ref{sec:synchrony}, shows  $\langle R \rangle$ versus $k = g \,\rho(W)$ (Fig.~\ref{fig:R_vs_k}) and the finite-size sharpening of the transition around the Perron prediction $k^\star$.
Together these results demonstrate that \textit{NOS} bridges dynamical-systems formalism and networking dynamics.  
It shows where simple spectral theory suffices (hub-dominated topologies) and where explicit simulation is required (homogeneous or delayed structures).

\begin{figure}[h!]
\centering
\includegraphics[width=0.65\textwidth]{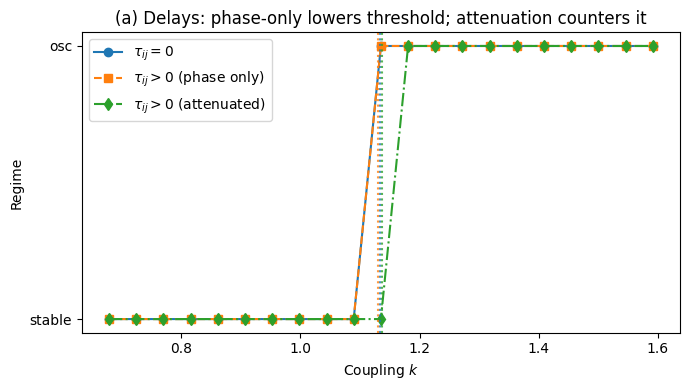}
\includegraphics[width=0.65\textwidth]{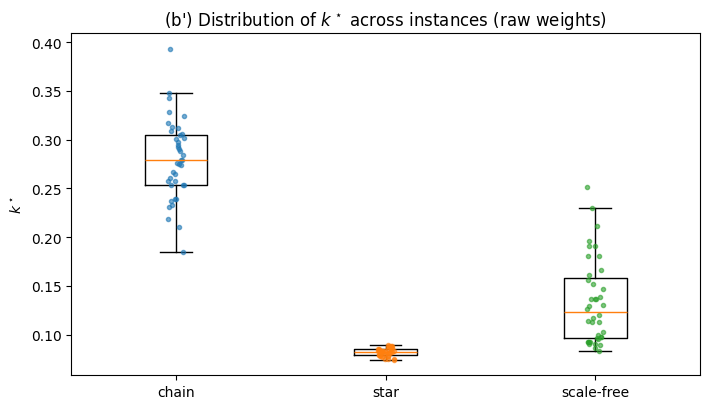}
\includegraphics[width=0.65\textwidth]{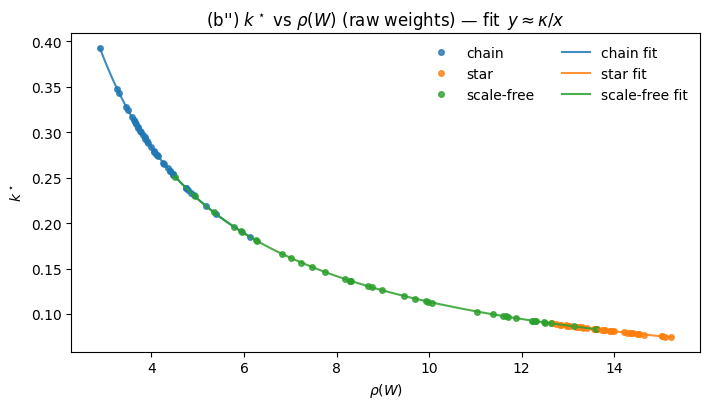}
\caption{Network-level dynamics in \textit{NOS}.  
(\textbf{a}) Delay effects: introducing link delays $\tau_{ij}>0$ shifts the stability boundary, causing oscillations to emerge at lower coupling $k$.  
(\textbf{b}) Absolute topology effects: chain networks destabilise last (resilient), stars first (fragile), with scale-free in between.  
(\textbf{b$''$}) Relative to Perron-mode prediction: chain networks deviate strongly above the $1/\rho(W)$ line, while star and scale-free graphs track the prediction closely.}
\label{fig:NOS_panels}
\end{figure}

\subsection{Predictive baseline: comparison with machine learning models}
\label{sec:ML-comparison}

Many operational networks lack reliable labels, so short-horizon forecasting with residual-based event inference is the standard route. We follow that practice. A complementary zero-shot comparison, in which all methods operate without supervised training and use only arrival-based calibration, is reported in \emph{methods}~\ref{subsec:zeroshot}. We compare \textit{NOS} with tGNN, GRU, RNN, and MLP under three canonical graphs: scale-free, star, and chain. Full details of the evaluation protocol are given in \emph{methods}~\ref{sec:evaluation-protocol}.

% ===================== FIGURE GROUP 1: TEST-RANGE ALIGNMENTS =====================
\begin{figure*}[h!]
  \centering
  \begin{subfigure}[t]{0.85\textwidth}
    \includegraphics[width=\linewidth]{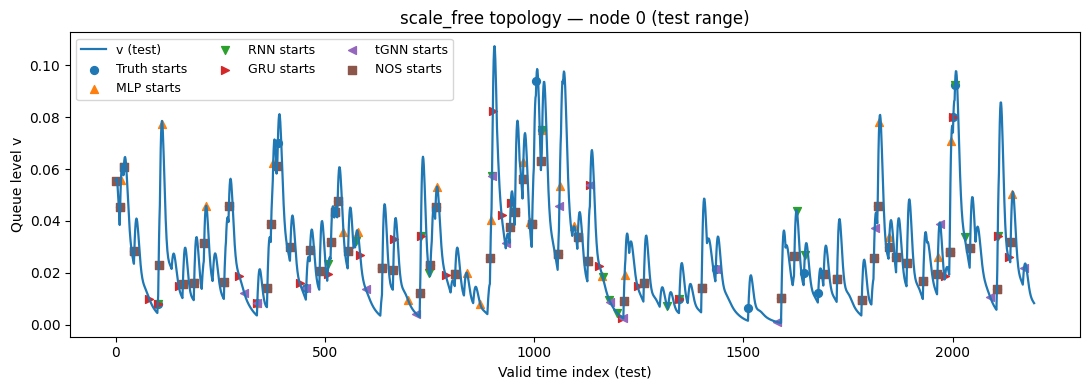}
    \caption{Scale-free. Truth and model start markers on the test range.}
    \label{fig:range-scale}
  \end{subfigure}\hfill
  \begin{subfigure}[t]{0.85\textwidth}
    \includegraphics[width=\linewidth]{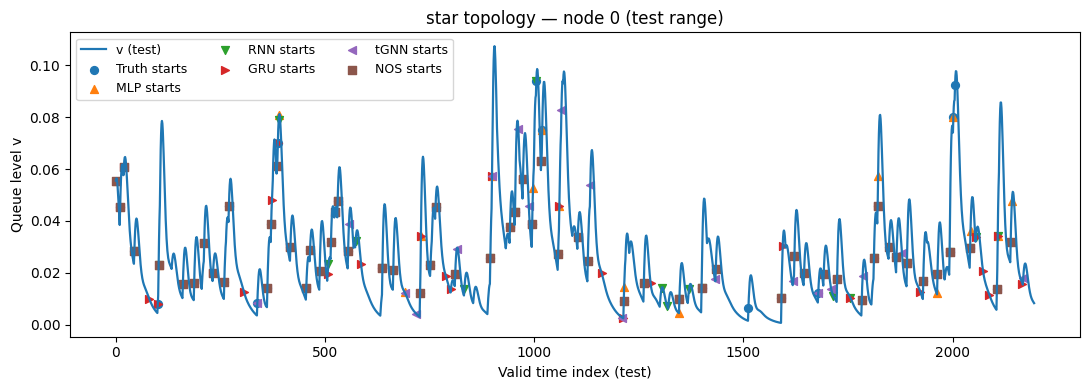}
    \caption{Star. Truth and model start markers on the test range.}
    \label{fig:range-star}
  \end{subfigure}\hfill
  \begin{subfigure}[t]{0.85\textwidth}
    \includegraphics[width=\linewidth]{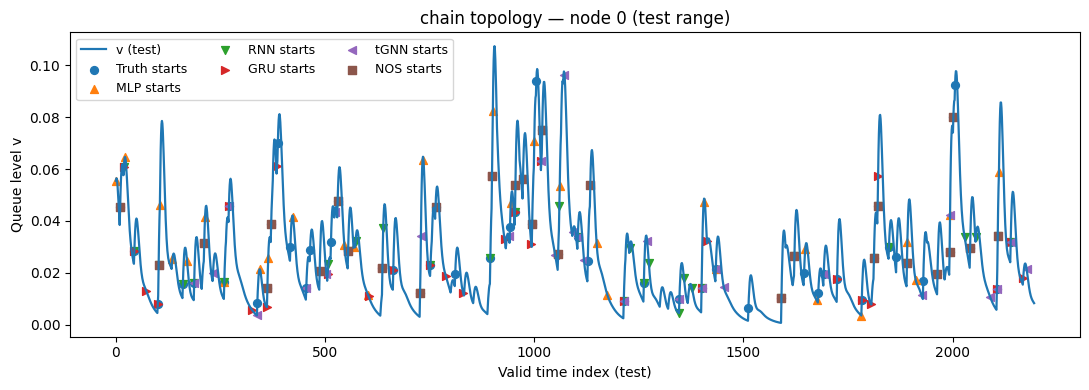}
    \caption{Chain. Truth and model start markers on the test range.}
    \label{fig:range-chain}
  \end{subfigure}
  \caption{Test-range alignments under train-calibrated residual thresholds. A marker at the truth dot is correct, to the right is late, to the left is early. Extra markers indicate false positives, and missing markers near truth indicate false negatives. \textit{NOS} remains close to truth across all three topologies while avoiding small oscillations.}
  \label{fig:ranges}
\end{figure*}

Figure~\ref{fig:ranges} presents the test-range series with start markers for each method. A marker at the truth dot indicates a correct start, a marker to the right indicates a late start, and a marker to the left indicates an early start. Extra markers away from a truth dot are false positives, while missing markers in the vicinity of a truth dot are false negatives. Read across the three panels from scale-free to star to chain and the pattern is consistent. On the scale-free graph, bursts arrive in clusters and ramps are steep. \textit{NOS} remains close to the truth throughout these clusters and avoids small oscillations. The tGNN calls tend to appear inside the ramps, which shows up as slight right shifts. RNN produces occasional early calls on moderate slopes, and GRU sits between RNN and tGNN. MLP is the least selective and places extra markers on low-amplitude wiggles. Moving to the star graph, the hub concentrates load while spokes are quieter. This helps tGNN, yet \textit{NOS} remains more selective at the spokes and keeps the timing tight when the hub rises. On the chain graph, onsets are clearer and repeat through the range. \textit{NOS} lands on the main rises with few stray calls in the troughs, while tGNN shows a mild lag, RNN is a little early on shallow slopes, and GRU again lies between them. The visual story across the three topologies is that \textit{NOS} aligns closely without over-firing when the background is quiet.

Figure~\ref{fig:metrics} gathers the aggregate scores. The bars report F1, precision, and recall under the same train-calibrated residual protocol. The curves report one-step error as MAE and RMSE and the median start-latency in milliseconds. Higher bars indicate better detection quality. Lower curves indicate better forecasting accuracy and earlier response. Across all three topologies \textit{NOS} attains the highest F1 by holding strong precision while lifting recall. It also yields the lowest MAE and RMSE and the smallest latency. RNN keeps good precision but loses recall and responds later. GRU and tGNN are intermediate and more sensitive to the graph structure. MLP has higher errors and longer delays. The metric panels match what the range plots suggested: \textit{NOS} is both earlier and more selective, and its forecasts give cleaner residuals for thresholding.

\begin{figure}[h!]
  \centering
  \begin{subfigure}[t]{0.5\textwidth}
    \includegraphics[width=\linewidth]{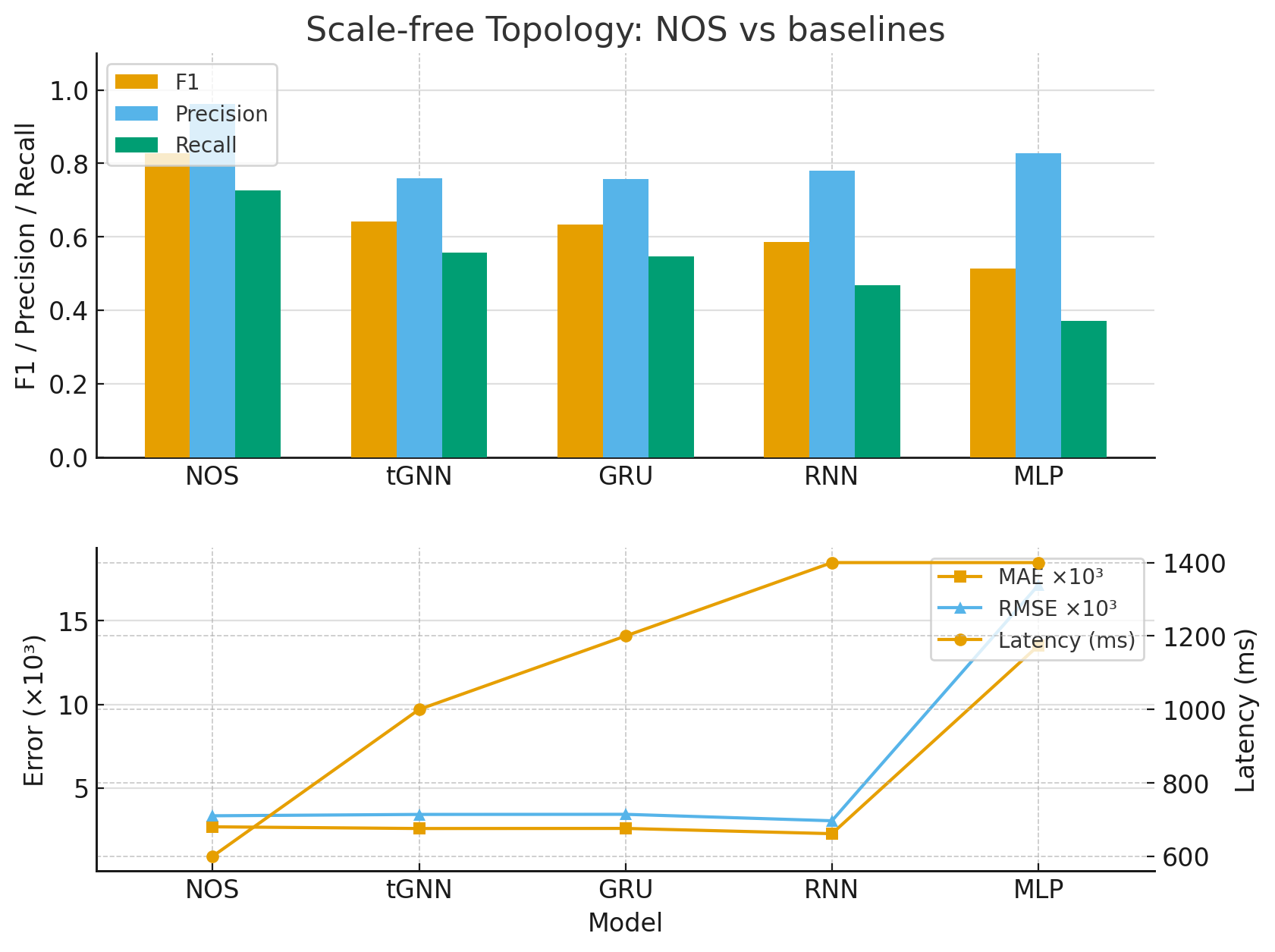}
    \caption{Scale-free. F1, precision, recall, MAE, RMSE, and median start-latency.}
    \label{fig:bars-scale}
  \end{subfigure}\hfill
  \begin{subfigure}[t]{0.5\textwidth}
    \includegraphics[width=\linewidth]{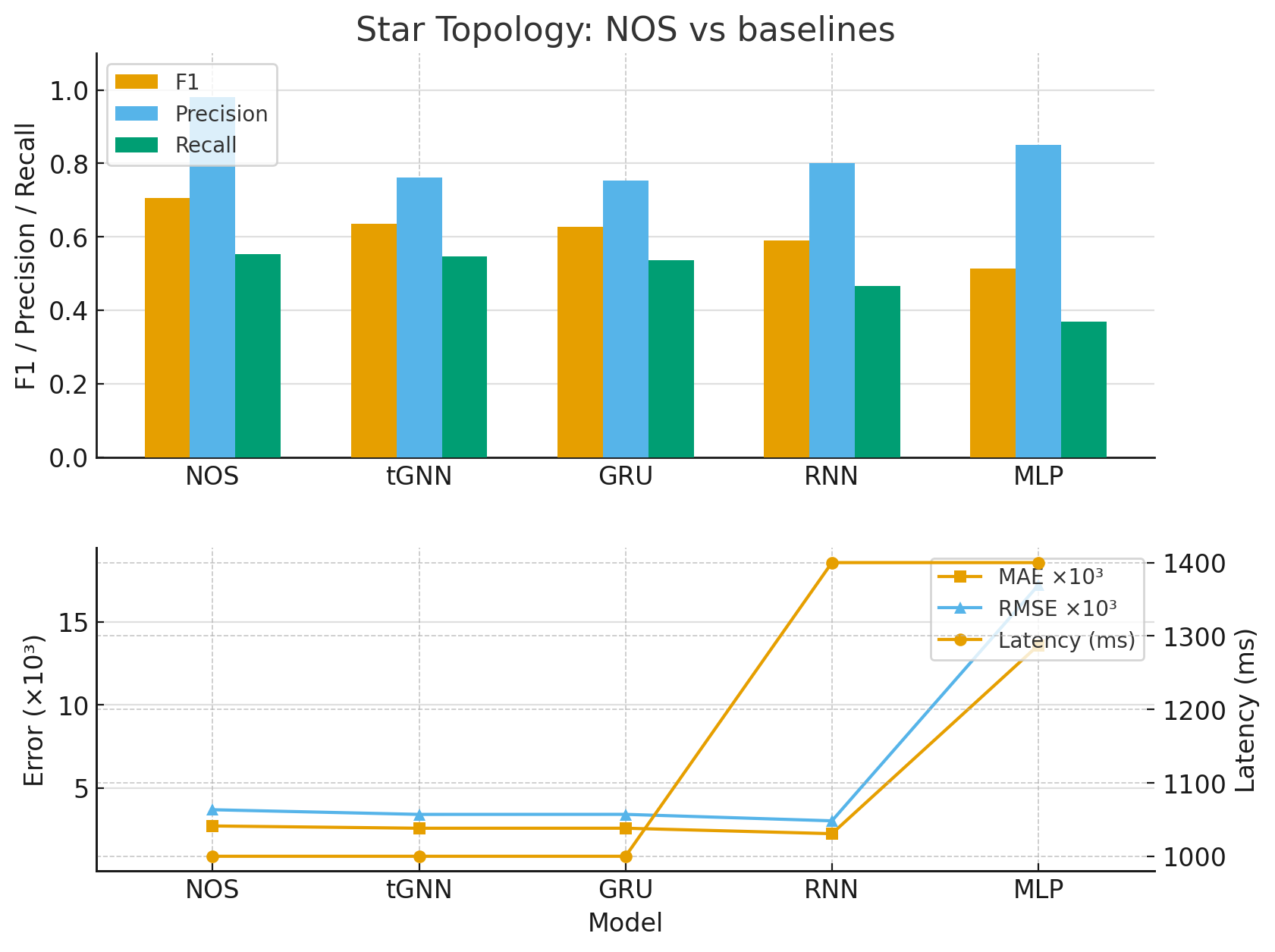}
    \caption{Star. F1, precision, recall, MAE, RMSE, and median start-latency.}
    \label{fig:bars-star}
  \end{subfigure}\hfill
  \begin{subfigure}[t]{0.5\textwidth}
    \includegraphics[width=\linewidth]{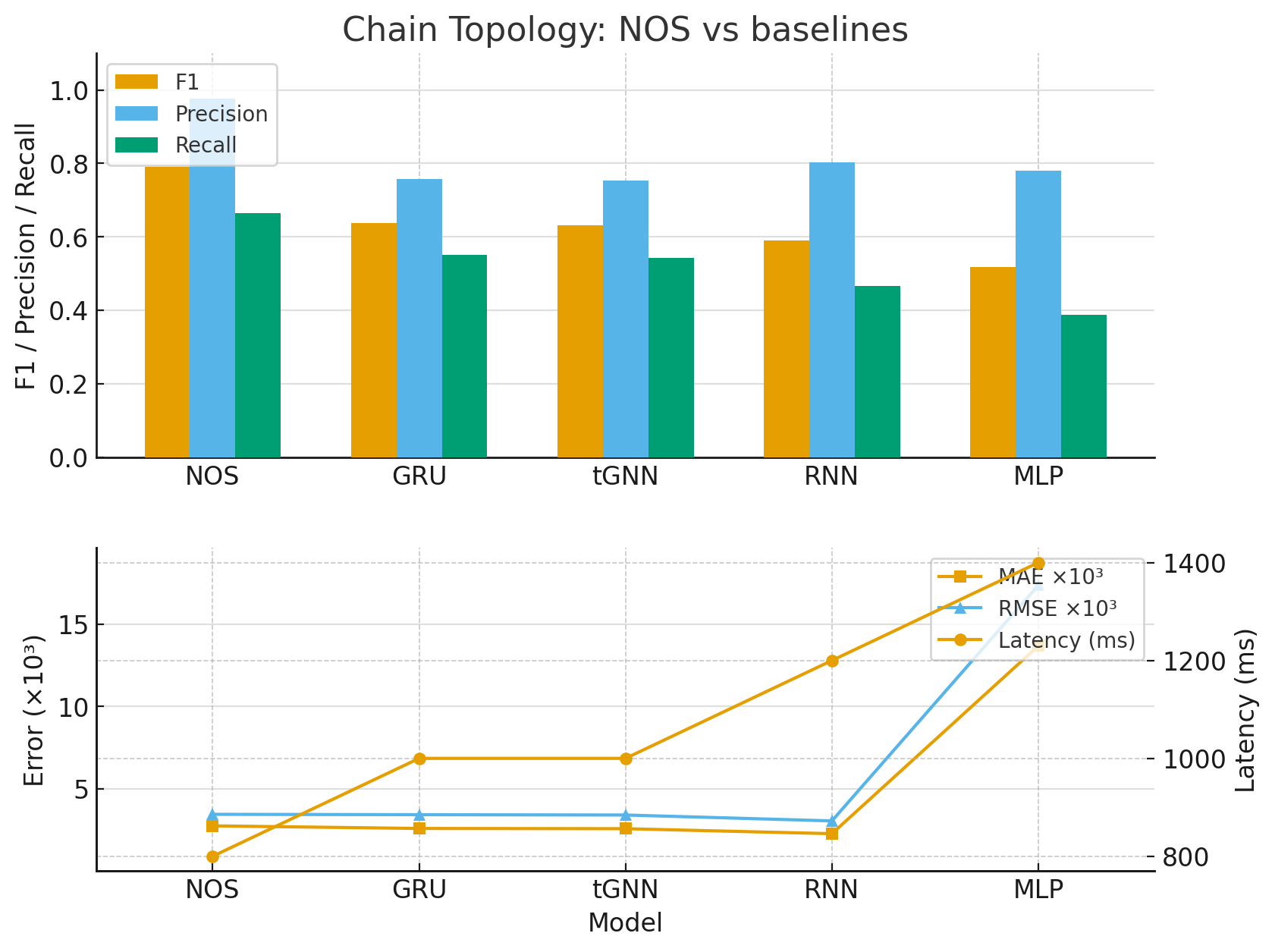}
    \caption{Chain. F1, precision, recall, MAE, RMSE, and median start-latency.}
    \label{fig:bars-chain}
  \end{subfigure}
  \caption{Summary metrics under the label-free residual protocol. Higher bars are better for F1, precision, and recall. Lower curves are better for MAE, RMSE, and latency. \textit{NOS} achieves the best F1, the lowest forecast error, and the earliest starts across all three topologies.}
  \label{fig:metrics}
\end{figure}

\section{Discussion}
\textit{NOS} treats congestion monitoring and control as a compact spiking process whose ingredients reflect how networks are actually run. The queue proxy \(v\) is normalised to buffer fullness, the recovery state \(u\) encodes recent stress, excitability is bounded to keep queues within a finite envelope, and the linear drains \(\lambda\) and \(\chi\) sit on the same device clocks as schedulers and EWMA filters. Thresholds implement budgeted alarms and resets are smooth enough for gradient-based training while still producing crisp, scheduler-scale events. Inputs are graph-local and include explicit delays and per-link gates, so timing and residual capacity appear as first-class quantities rather than side effects of architecture.

The comparison with classical spiking paradigms clarifies what is gained. Random Neural Networks exploit product forms and queueing intuition but hinge on Poisson assumptions and stationary averages. Hodgkin--Huxley and Izhikevich offer rich excitable dynamics, yet their variables and parameters are tied to ion channels and generic cortical behaviour rather than to queue occupancy, service, and link delay. Across these models, either the physics is faithful and heavy or it is light and disconnected from networking semantics. \emph{NOS} occupies a middle ground. It keeps a two-dimensional state with saturating excitability \(f_{\mathrm{sat}}\) and a recovery loop, and it assigns queue-level meaning to every parameter that appears in the small-signal slope, stability tests, and reset dynamics. This mapping underpins the operational rules in the \emph{methods}.

The analytical results show that stability can be expressed in terms that are familiar to operators. At node level, the Routh--Hurwitz conditions reduce to a balance between the local slope
$f'_{\mathrm{sat}}(v^*)$ and a drain margin
$\Lambda := \min\{(a+\mu),\,ab/(a+\mu)\}-(\beta-\lambda-\chi)$.
At network level, coupling enters through $k=g\rho(W)$, giving the proxy $k+f'_{\mathrm{sat}}(v^*)<\Lambda$.
At network level, coupling enters only through \(g\,\rho(W)\), so the proxy
\(
g\,\rho(W)\,+f'_{\mathrm{sat}}(v^*) < \Lambda
\)
separates topology from node physics. This yields three practical levers: adjust \(\rho(W)\) by reweighting or sparsifying heavy fan-in, adjust the gain \(g\) on stressed regions through scheduling and policy, or adjust the local margin \(\Lambda\) through service, damping, and recovery parameters. The bifurcation and synchrony analyses support this scalar picture, capturing both saddle--node tipping and Hopf-like oscillations, while finite-size experiments show that thresholds sharpen with network size.

Queueing baselines anchor \emph{NOS} against classical theory and device-level intuition. In open loop, light-load calibration aligns the \emph{NOS} mean with the M/M/1 prediction, while bounded excitability and soft pull-back bend earlier near saturation and yield shorter tails under bursty MMPP arrivals. From a networking viewpoint, \emph{NOS} trades a small increase in point error at high load for a systematic reduction in deep-queue probability, which is where SLO violations and drop storms arise. In closed loop, NOS-based controllers produce marking signals that are decisive and low in jitter, avoiding both noisy oscillations of raw queue marking and the lag induced by heavy low-pass filtering. These behaviours are consistent with the headroom conditions and with the linear decay times extracted from the local Jacobian.

The forecasting experiments indicate that \emph{NOS} is competitive with compact learned baselines when labels are scarce. A fluid predictor driven by the same arrivals provides a strong MAE reference because it effectively reproduces the queueing simulator. \emph{NOS} does not attempt to mimic every backlog micro-fluctuation and therefore carries a larger point error, yet it delivers high event skill for burst onset detection and produces compact responses around excursions. Temporal GNNs, MLPs, and GRUs trained in the same label-free regime tend to trade off noise suppression against timing accuracy. In contrast, the bounded, reset-driven dynamics of \emph{NOS} yield sharp onset timing without prolonged false positives between bursts. The protocol is deliberately strict, with train-only calibration and validation-based thresholds, which reduces the likelihood that gains arise from tuning artefacts.
Noise sensitivity and avalanche analyses view \emph{NOS} through the lens of stochastic drive. Shot-noise arrivals with realistic smoothing produce firing statistics whose dependence on amplitude, coupling, and recovery matches the linear response theory. As the network approaches the coupling threshold, the same noise induces higher firing rates, larger inter-spike-interval variability, and heavier avalanche tails. The fitted exponents and synchrony measures indicate that coupling can push the system toward congestion cascades along dominant influence paths. Operationally, this supports the same levers suggested by the deterministic analysis: increase service and damping, saturate earlier through larger \(\kappa\), and reduce effective coupling on noisy links.

An explicit path to neuromorphic deployment follows from the design. Address--event routing realises sparse, delayed coupling with the same \(W\) and \(\tau_{ij}\) that appear in the analysis. Soft resets can be implemented as leaky pull-backs without algebraic jumps, and fixed-point quantisation respects the queue normalisation used in software. The stability proxy carries over once finite-precision errors are bounded, allowing the same margin checks to be applied on mixed-signal or digital neuromorphic substrates. This alignment between model, analysis, and hardware is attractive for edge and in-network deployments where energy per decision should scale with activity rather than with link count.

A few limitations follow from our modelling choices. The arrival model uses compound Poisson shot noise with exponential smoothing, which captures burstiness at typical telemetry timescales but does not represent long-range dependence or heavy-tailed flow mixing in full generality. The experiments focus on single-queue observables and fixed topologies, and the queueing baselines largely reflect M/M/1-style behaviour. Joint adaptation with transport protocols, multi-queue schedulers, and cross-layer policies remains open. 
% Training relies on surrogate gradients and truncated BPTT windows aligned with the recovery time; although the homotopy schedule and clipping stabilise optimisation, a more systematic study of learnability under aggressive thresholds and high coupling would be useful.

\section{Conclusion}
We introduced \emph{Network-Optimised Spiking} (NOS), a two-state event-driven unit designed for delay-coupled, graph-structured telemetry with explicit queue semantics. NOS combines bounded excitability, explicit service and damping leaks, delayed graph-local coupling with optional per-link gates, and differentiable resets that support surrogate-gradient training. We provided node- and network-level stability conditions, including a scalar spectral proxy based on \(g\,\rho(W)\) that separates topology from local dynamics, and we connected stochastic drive under telemetry smoothing to increased variability near stability boundaries. Empirically, under a strict label-free residual protocol, NOS improves early-warning quality and detection latency over compact learned baselines across canonical topologies, while retaining calibration handles that are interpretable to operators.

Looking ahead, three directions are immediate. Richer traffic and topology models, including multi-tenant datacentre fabrics and mobile radio access networks, would test robustness of the stability proxy and forecasting claims. Hierarchical designs that couple NOS units with GNN or RL controllers could combine fast local reflexes with global policy. Finally, hardware-in-the-loop studies on contemporary neuromorphic platforms can quantify energy, latency, and tail behaviour under realistic deployment constraints.

\section*{Impact Statement}
NOS supports stability-aware, event-driven inference for delayed, graph-coupled streams, with parameters that can be calibrated from telemetry and checked against simple spectral margins. It can reduce detection latency and severe congestion events in resource-constrained deployments by computing mainly when activity changes. 
% Risks include applying stability proxies outside their assumptions or under distribution shift; we mitigate this by reporting margins and validating calibration under changing traffic and topology.

\appendix

\section*{Methods}\label{sec11}

\section{Spiking Neural Networks for Networking: Scope and Complementarity}\label{subsec:snn-scope}
Spiking models describe a time-varying state and emit discrete events when a threshold is crossed. In the biological literature this is written in terms of a membrane potential \(V(t)\); for networking the same idea can be read as a state that advances under inputs and intrinsic dynamics until an event is produced. The event acts as the computational primitive.
Intelligent networking has been led by deep learning, reinforcement learning, and graph neural networks trained on aggregates or labelled flows. These methods assume synchronous batches and centralised compute. Spiking neural networks use event-driven updates, which aligns with irregular packet arrivals, bursty congestion, and anomaly signatures that depart from baseline rather than average behaviour. The practical question is when this alignment matters enough to prefer an SNN, and when conventional ML or GNNs remain the better tool, \ref{tab:snn-ml-map} summarizes it.
\begin{table*}[h!]
\centering
\caption{Comparative suitability of SNNs and conventional ML/GNNs across networking problem domains.}
\label{tab:snn-ml-map}
\renewcommand{\arraystretch}{1.25}
\setlength{\tabcolsep}{6pt}
\begin{tabular}{
  p{3.7cm}
  >{\RaggedRight\arraybackslash}p{5.3cm}
  >{\RaggedRight\arraybackslash}p{5.3cm}}
\toprule
\textbf{Problem domain} & \textbf{More suitable with SNNs} & \textbf{More suitable with ML/GNNs} \\
\midrule
Anomaly detection (e.g., DDoS, worm outbreaks) &
Event-driven detection of sudden bursts; low-latency alarms \cite{Maciag2021,Baessler2022} &
Flow-level detection from large corpora; supervised and deep models \cite{Ahmed2016,Jiang2023} \\
\addlinespace
Routing optimisation &
Local adaptation in ad hoc or sensor settings using local rules; energy- and event-efficient deployment on neuromorphic hardware \cite{Davies2018,Furber2014} &
Global path optimisation and traffic engineering with topology-wide objectives; GNN/RL controllers \cite{XuTealSIGCOMM2023} \\
\addlinespace
QoS prediction (throughput, delay) &
Edge-side reaction to micro-bursts under power limits (event-driven inference) \cite{Davies2018} &
Predictive modelling from aggregates; offline/online regression and forecasting \cite{Giannakou2020,Krishna2024,xu2023} \\
\addlinespace
Traffic classification &
Temporal signatures (periodicity, jitter) via spike timing and event patterns \cite{rasteh2022,plat2017} &
Feature-rich flow classification with labels using CNN/RNN/GNN \cite{Nguyen2008,Zhao2021,basit2022} \\
\addlinespace
Wireless resource allocation &
Local duty-cycling, sensing, and lightweight on-device decision-making under tight energy budgets \cite{Furber2014} &
Centralised spectrum allocation and large-scale optimisation with RL/GNN \cite{LuongDRL2019,Zhi2022,IQBAL2023} \\
\addlinespace
Network resilience and fault detection &
Distributed detection of local failures and sudden state changes in streams \cite{Maciag2021,Baessler2022} &
Correlating failures across topology; resilience modelling and recovery for SDN/CPS \cite{Menaceur2023,SegoviaFerreiraACMCSUR2024,Liu2024} \\
\bottomrule
\end{tabular}
\end{table*}

\paragraph{Event-driven versus batch-oriented learning.}
Packet systems are inherently event-driven. Conventional pipelines buffer into fixed windows before inference, adding latency and energy. SNNs update on arrival, so a single event can trigger a decision without waiting for a batch. This is attractive for anomaly alerts and micro-burst monitoring where early reaction is valuable.

\paragraph{Energy and deployment constraints.}
Edge devices and in-network elements often face tight power and memory budgets. Neuromorphic platforms (e.g., Loihi and SpiNNaker) report energy advantages for spiking workloads, which supports in situ decisions at low latency \cite{Davies2018,SpiNNaker}. In contrast, data-centre controllers and core routers, with ample power and global context, typically profit from conventional ML and GNNs.

\paragraph{Temporal coding and traffic structure.}
Some networking tasks hinge on timing patterns rather than sheer volume: jitter in interactive media, periodic bursts in games, or hotspot prediction from inter-arrival times. SNNs encode such structure in spike timing and phase. Recurrent deep models can capture time as well, but temporal coding is native to spiking dynamics.

\paragraph{Local learning and autonomy.}
Networks are distributed. Routers, base stations, and sensors often adapt with limited scope. SNNs admit local rules such as spike-timing-dependent plasticity (STDP), enabling on-node adaptation. Global optimisation over topology and multi-objective trade-offs remains a strength of GNNs and reinforcement learning.

\paragraph{Complementary roles.}
SNNs and conventional AI are complementary. SNNs provide efficient, event-driven reflexes for local, time-critical decisions under tight budgets. Deep learning and GNNs act as global planners that integrate structure and statistics. A practical design places SNNs in edge nodes or programmable switches for reflexive response, coordinated with ML/GNN controllers for end-to-end policy.

\section{Classical spiking paradigms and networking suitability}
\label{sec:classicsnn}
This section reviews three classical spiking neural networks and examines what they offer, and their suitability for networking tasks where queue semantics, topology, and trainability plays important role.

\subsection{Random Neural Network (Gelenbe)}
\label{subsec:RNNG}
The Random Neural Network (RNN) treats each neuron as an integer-valued counter driven by excitatory and inhibitory Poisson streams and yields a product-form steady state \cite{gelenbe1989random}. Let $q_i=\Pr\{k_i>0\}$ denote the activity of neuron $i$. In steady state
\begin{align}
q_i &= \frac{\Lambda_i^{+}}{r_i+\Lambda_i^{-}}, \label{eq:rnn:q}
\\
\Lambda_i^{+} &= \lambda_i^{+} + \sum_{j} q_j r_j p_{ji}^{+}, \qquad
\Lambda_i^{-} = \lambda_i^{-} + \sum_{j} q_j r_j p_{ji}^{-}, \label{eq:rnn:lambda}
\end{align}
and the joint distribution factorises as $\Pr[k_1,\ldots,k_N]=\prod_i (1-q_i) q_i^{k_i}$. 

The queueing affinity and fixed points make calibration attractive. However, backbone and datacentre traces are overdispersed and often long-memory, so a stationary product form understates burst clustering and tail risk. This follows from the Poisson and exponential assumptions of the model, which wash out temporal correlation and refractoriness. Product-form results emphasise stationarity, whereas operations rely on transient detection and control. Parameters \(r_i\) and \(p_{ji}^{\pm}\) do not tie directly to per-link capacity, delay, or gating. Under strong excitation, the signal-flow equations can fail to admit a valid solution with \(q_i<1\), leading to instability or saturation; when a solution exists it is unique \cite{gelenbe1989random, GelenbeNECO1990}.

\subsection{Hodgkin--Huxley}
\label{subsec:HH}
The Hodgkin--Huxley (HH) equations provide biophysical fidelity through voltage-gated conductances and reproduce spikes, refractoriness and ionic mechanisms \cite{Hausser2000}. The membrane and current relations are
\begin{align}
C_m \frac{dV}{dt} &= -\big(I_{\mathrm{Na}}+I_{\mathrm{K}}+I_L\big) + I_{\mathrm{ext}}, \label{eq:hh:mem}
\end{align}

\begin{align}
\begin{split}
I_{\mathrm{Na}} = \bar g_{\mathrm{Na}}\, m^3 h \,(V-E_{\mathrm{Na}}),
I_{\mathrm{K}} = \bar g_{\mathrm{K}}\, n^4 \,(V-E_{\mathrm{K}}),\\
I_L = \bar g_L (V-E_L), \label{eq:hh:curr}
\end{split}
\end{align}

\begin{align}
\frac{dx}{dt} &= \alpha_x(V)(1-x) - \beta_x(V)x, \quad x\in\{m,h,n\}. \label{eq:hh:gates}
\end{align}
where \(V\) is membrane potential; \(C_m\) membrane capacitance. \(I_{\mathrm{Na}}, I_{\mathrm{K}}, I_L\) are sodium, potassium, and leak currents; \(I_{\mathrm{ext}}\) is external or synaptic input. \(\bar g_{\mathrm{ion}}\) are maximal conductances and \(E_{\mathrm{ion}}\) reversal potentials. \(m,h,n \in [0,1]\) are gating variables with voltage-dependent rates \(\alpha_x(V), \beta_x(V)\).

Hodgkin–Huxley offers mechanistic fidelity and a broad repertoire of excitable behaviour, and recent results show that HH neurons can be trained end to end with surrogate gradients, achieving competitive accuracy with very sparse spiking on standard neuromorphic benchmarks \cite{HHExtremeSparsity}. This confirms feasibility for learning and suggests potential efficiency from activity sparsity. For network packet-level control, the obstacles are practical rather than conceptual. The model comprises four coupled nonlinear differential equations with voltage-dependent rate laws, which typically require small timesteps and careful numerical treatment; in practice this raises computational cost and stiffness at scale, and training can fail without tailored integrators and schedules \cite{HHExtremeSparsity}. There is also no direct mapping from biophysical parameters to queue occupancy, service rate, or link delay, so calibration from network telemetry lacks interpretability. Taken together, HH remains a benchmark for cellular electrophysiology \cite{Hausser2000}, but it is a poor fit for networking tasks that requires lightweight, semantically mapped states and trainable dynamics.

\subsection{Izhikevich}
\label{subsec:Izhikevich}
The Izhikevich model balances speed and diversity of firing patterns \cite{Izhikevich2003}. Its dynamics and reset are
\begin{align}
\begin{split}
\frac{dv}{dt} = 0.04 v^2 + 5v + 140 - u + I, \\
\frac{du}{dt} = a(bv-u), \label{eq:izh:dyn}
\end{split}
\end{align}
\begin{align}
\begin{split}
\text{if } v \ge 30\,\mathrm{mV}: \quad  v \leftarrow c, \quad u \leftarrow u + d. \label{eq:izh:reset}
\end{split}
\end{align}
where, \(v\) is the membrane potential (mV). \(u\) is a recovery variable with the same units as \(v\) that aggregates \(K^+\) activation and \(Na^+\) inactivation effects. \(I\) is the input drive expressed as an equivalent voltage rate (mV\,ms\(^{-1}\)). The spike condition is \(v \ge 30\) mV, after which \(v \leftarrow c\) and \(u \leftarrow u+d\). \(a\) is the inverse time constant of \(u\) (ms\(^{-1}\)); \(b\) is the coupling gain from \(v\) to \(u\) (dimensionless); \(c\) is the post-spike reset level of \(v\) (mV); \(d\) is the post-spike increment of \(u\) (mV).

The Izhikevich model achieves breadth of firing patterns with two state variables and a hard after-spike reset \eqref{eq:izh:dyn}–\eqref{eq:izh:reset}. This efficiency is documented through large-scale pulse-coupled simulations and parameter recipes for cortical cell classes. For networking, three constraints matter. First, the quadratic drift in \eqref{eq:izh:dyn} is not intrinsically saturating, which is mismatched to finite-buffer behaviour and can drive unrealistically rapid growth between resets under heavy input. Second, the reset in \eqref{eq:izh:reset} is discontinuous, which hinders gradient-based optimisation. Third, the parameters$(a,b,c,d)$ are phenomenological and the connectivity is generic, so queue occupancy, service rate, link delay, and per-link quality are not first-class. These features explain the model’s value for neuronal dynamics and its limited interpretability for packet-level control.

\subsection{Synthesis}
\label{subsec:syn}
Across the three formalisms the evidence points to the same gaps. Either the model is correct but heavy, or it is fast but carries non-differentiable resets and unbounded excitability. In all cases the parameters do not align with queue occupancy, service, and delay. Topology and per-link quality are usually implicit rather than explicit. For packet networks this limits interpretability, stability under high load, and the ability to train or adapt in situ. These observations set the design brief for NOS: a compact two-state unit with finite-buffer saturation, a service-rate leak and a differentiable reset, graph-local inputs with delays and optional gates, and parameters that map to observable network quantities.

\begin{table*}[h]
\centering
\caption{Classical paradigms versus networking requirements.}
\label{tab:classical-compare}
\renewcommand{\arraystretch}{1.15}
\setlength{\tabcolsep}{6pt}
\begin{tabular}{p{3.4cm}p{3.3cm}p{3cm}p{3.5cm}}
\toprule
 & \textbf{RNN (Gelenbe)} & \textbf{Hodgkin--Huxley} & \textbf{Izhikevich} \\
\midrule
State/parameters map to queue, service, delay & Limited (probabilistic rates) & No & No \\
Trainability for control (gradients) & Moderate at steady state & Low & Low (non-differentiable reset) \\
Topology and per-link attributes explicit & Limited & No & No \\
Finite-buffer realism / stability under load & Stationary focus & Realistic but heavy & Unbounded drive \\
Edge-side compute/energy & Low–moderate & High & Low \\
\bottomrule
\end{tabular}
\end{table*}

%%%%% \ref{defvals} is already cited
\section{Admissible parameter region and calibration conventions}
\label{defvals}
\textit{NOS} parameters are summarised in Table~\ref{tab:param_combined}, while Table~\ref{tab:nos_ranges} records the admissible region explored in our experiments. The entries are expressed on the experimental scale used throughout the paper, with queue fullness normalised to \(v\in[0,1]\) and a sampling bin of \(5\) ms. Rates given “per bin” convert to per second by multiplying by \(200\). We separate structural constraints from initialisation choices so that the table serves reproducibility without implying narrow tuning.

The bounds have three roles. First, they enforce the normalisation and sampling assumptions used by the model and code. Second, they keep the sufficient stability condition from Lemma~2 satisfied with margin for the operating regimes we test. Third, they define a common hyperparameter space for stress tests across the three graph families. Within this region we draw values uniformly unless stated and refine only on the validation split.

Coupling and reset are disambiguated. The network coupling index is \(k_{\text{net}}=g\,\rho(W)\), where \(\rho(W)\) is the spectral radius after normalisation to \(\rho(W)=1\) and \(g\) is the scalar applied during experiments. The symbol \(k_{\mathrm{reset}}\) denotes the sharpness of the soft reset gate and is independent of \(k_{\text{net}}\). Delays \(\tau_{ij}\) are given in milliseconds and reflect link propagation on the chosen binning. Shot-noise parameters \((\nu_i,A)\) match the noise sensitivity experiments and allow drive strength to approach the near-threshold regime without violating the stability bound.

The ranges for \(\alpha,\kappa,\beta,\gamma\) control the shape and slope of the bounded excitability \(f_{\mathrm{sat}}\). The leak terms \(\lambda\) and \(\chi\) set the small-signal decay about \(v_{\mathrm{rest}}\) and are bounded to keep subthreshold dynamics stable at the \(5\) ms resolution. Recovery parameters \((a,b,\mu)\) govern post-burst relaxation and are given as per-bin rates. Threshold \(v_{\mathrm{th}}\), reset depth \(c\), recovery jump \(d\), and pullback speed \(\rho\) determine spiking and refractory behaviour. The mapping from arrivals to effective input uses an offset \(I_0\) and a dimensionless gain, and may be optionally low-pass filtered with \(\tau_s\).

These bounds are admissible rather than prescriptive. Defaults used in the main text fall within Table~\ref{tab:nos_ranges} and are reported in the parameter-initialisation table. Sensitivity studies widening each bound by a factor of two preserve the relative ordering of methods on F1 and latency, with the largest trade-offs arising from \(\lambda\) and \(v_{\mathrm{th}}\) as expected from precision–recall balance. All residual scalers and thresholds are fitted on train only, and validation is used to select per-node thresholds within the same admissible region.

\begin{table*}[htbp]
\centering
\caption{Operational interpretation and \emph{data-driven initialisation} of \textit{NOS} parameters from observed network metrics. Defaults are starting values and a bin width of \(5\)\,ms; they are not hard limits.}
\label{tab:param_combined}
\fontsize{9}{11}\selectfont
\begin{tabular}{p{0.15\textwidth} p{0.33\textwidth} p{0.32\textwidth} p{0.16\textwidth}}
\toprule
Parameter & Observable(s) & Initialisation rule & Typical default \\
\midrule
\(v\) & Queue length or buffer occupancy & Affine map full buffer \(\to 1\), empty \(\to 0\) & — \\
\(v_{\mathrm{th}}\) & False–alarm budget on residuals & Choose so \(\Pr(z>z_\tau)=p_{\mathrm{FP}}\) on train; set threshold at \(z_\tau\) & \(0.60\) \\
\(\lambda\) & Mean service rate \(\mu_{\mathrm{svc}}\) & Per-bin rate: \(\lambda \approx \mu_{\mathrm{svc}}\Delta t\) & \(0.18\) per bin \\
\(\chi\) & Small-signal damping & Fit AR(1) on subthreshold \(v\): \(v_{t+1}\!\approx\!(1-\lambda-\chi)v_t+\dots\) & \(0.02\)–\(0.05\) per bin \\
\(\beta\) & Slope of \(dv/dt\) vs \(v\) & Regress \(\Delta v/\Delta t\) on \(v\) on train residuals & \(0.05\)–\(0.8\) \\
\(\gamma\) & Baseline load & Intercept from the same regression, or set mean residual \(\approx 0\) & \(0.06\)–\(0.10\) \\
\(\alpha,\kappa\) & Nonlinear ramp steepness & Fit \(f_{\mathrm{sat}}(v)=\alpha v^2/(1+\kappa v^2)\) to rising edges & \(\alpha=0.7\), \(\kappa=1.0\) \\
\(a\) & Post-burst relaxation time \(\tau_{\mathrm{rec}}\) & \(a \approx \Delta t/\tau_{\mathrm{rec}}\) (per bin) from exponential fit & \(1.1\) per bin \\
\(b\) & Recovery sensitivity to \(v\) & Tune so \(u\) tracks \(v\) during decay without overshoot & \(1.0\)–\(1.2\) \\
\(\mu\) & Passive decay of \(u\) & Small regulariser on \(u\) drift; set by validation & \(0.05\)–\(0.15\) per bin \\
\(c\) & Post-event baseline of \(v\) & Median \(v\) in a quiet window after events & \(0.10\) \\
\(d\) & Recovery jump on event & Choose to match observed refractory depth & \(0.20\)–\(0.30\) \\
\(k_{\mathrm{reset}}\) & Desired sharpness of reset & Large enough to mimic hard reset without instability & \(14\) \\
\(r_{\text{reset}}\) & Reset time constant \(\tau_{\mathrm{reset}}\) & \(r_{\text{reset}} \approx \Delta t/\tau_{\mathrm{reset}}\) (per bin) & \(5\) per bin \\
\(I_0\), gain & Arrivals \(\to I\) mapping & Regress \(I\) on smoothed arrivals: \(I\!\approx\!I_0+\text{gain}\cdot \text{arrivals}\) & \(I_0=0.10\), gain \(=1.0\) \\
\(\tau_s\) & Burst smoothing need & Low-pass if arrivals are spiky; choose cut-off by validation & \(0\)–\(10\) ms \\
\(w_{ij}\) & Link rates/priorities & Proportional to nominal bandwidth or policy weight; normalise \(\rho(W)=1\) & — \\
\(g\) & Desired coupling index \(k_{\text{net}}\) & Set \(g = k_{\text{net}}/\rho(W)\); pick \(k_{\text{net}}\in[0.8,1.6]\) for stress & topology-specific \\
\(\tau_{ij}\) & Per-link delay & From RTT or profiling; use queue-free component & \(0\)–\(25\) ms \\
\bottomrule
\end{tabular}
\vspace{2pt}

\footnotesize
Per-bin conversion uses \(\Delta t=5\)\,ms. For per-second rates multiply by \(200\).
Defaults are the values used unless stated; sensitivity to \(\{\lambda,v_{\mathrm{th}},\alpha,\kappa\}\) is discussed in this \emph{methods}.
Further design and training guidance, including the telemetry pipeline, reset selection, surrogate gradients, and BPTT window choice, is provided in \emph{methods}~\ref{sec:nos-design-guidance}.
\end{table*}

\begin{table*}[t]
\centering
\caption{Admissible parameter region used in the experiments. State $v$ is dimensionless; time is measured in bins with $\Delta t=5$\,ms. We use the recovery coupling $\xi$ in \eqref{eq:nos-v-eq} to map the (dimensionless) recovery state $u$ into a rate term, and fix $\xi=1$ in all experiments (absorbed into the definition of $u$). Per-bin rates convert to physical rates as $r_{\mathrm{phys}} = r / \Delta t$ (so $1$ per bin $=200$ s$^{-1}$). Values were sampled within these regions; defaults are given elsewhere for reproducibility.}
\label{tab:nos_ranges}
\fontsize{9}{11}\selectfont
\begin{tabular}{p{0.18\textwidth} p{0.33\textwidth} p{0.41\textwidth}}
\toprule
Parameter & Interpretation & Admissible range (units) \\
\midrule
$\alpha$ & excitability scale in $f_{\mathrm{sat}}$ & $[0.4,\,1.0]$ (per bin) \\
$\kappa$ & saturation knee of $f_{\mathrm{sat}}$ ($\kappa>0$) & tested $\{0.5,\,2.0\}$ (dimensionless) \\
$\beta$ & linear excitability gain & $[-0.10,\,0.8]$ (per bin) \\
$\gamma$ & constant drive (baseline load) & $[0.00,\,0.15]$ (per bin) \\
$\lambda$ & service/leak on $v$ & $[0.10,\,0.30]$ (per bin) \\
$\chi$ & subthreshold damping about $v_{\mathrm{rest}}$ & $[0.00,\,0.08]$ (per bin) \\
$\xi$ & recovery-to-$\dot v$ coupling (maps $u\mapsto$ rate) & fixed to $1$ (per bin) \\
$a$ & recovery rate & $[0.6,\,1.8]$ (per bin) \\
$b$ & recovery sensitivity to congestion & $[0.6,\,1.6]$ (dimensionless) \\
$\mu$ & passive recovery decay & $[0.00,\,0.35]$ (per bin) \\
$v_{\mathrm{th}}$ & firing threshold & $[0.50,\,0.68]$ (dimensionless $v$ units) \\
$k_{\mathrm{reset}}$ & sharpness of reset gate & $[10,\,20]$ (dimensionless) \\
$\rho$ & pullback/reset speed & $[3,\,8]$ (per bin) \\
$c$ & post-event baseline level & $[0.0,\,0.2]$ (dimensionless $v$ units) \\
$d$ & recovery jump on event (adds to $u$) & $[0.1,\,0.4]$ (dimensionless $u$ units) \\
$I_0$ & drive offset (arrivals $\to I$) & $[0.08,\,0.16]$ (per bin) \\
\textit{gain} & drive scale (arrivals $\to I$) & $[0.8,\,1.2]$ (dimensionless) \\
$\tau_s$ & optional drive smoothing & $[0,\,3]$ bins (equiv.\ $[0,\,15]$ ms) \\
$g$ & coupling on $W$ (network) & choose so $k_{\text{net}}=g\,\rho(W)\in[0,\,1.8]$ (dimensionless) \\
\midrule
Shot noise & arrival model for synthetic tests &
rate $\nu_i\in[10,\,50]$ Hz (mapped using $\Delta t=5$ ms), amplitude $A\in[0.3,\,0.9]$ (per bin) \\
Delays & link propagation & $\tau_{ij}\in[0,\,5]$ bins (equiv.\ $[0,\,25]$ ms) \\
$W$ normalisation & spectral scaling & $\rho(W)=1$ before applying $g$ \\
\bottomrule
\end{tabular}
\end{table*}

\footnotetext{Per-bin rates correspond to a $5$ ms bin width. Multiply by $200$ to convert to s$^{-1}$.}

\section{Design principles and engineering guidance}
\label{sec:nos-design-guidance}

The \textit{NOS} configuration is \emph{straightforward}: it follows a small set of choices that mirror device behaviour in operational networks. Figure~\ref{fig:nos-pipeline} summarises how a \emph{NOS} unit is used in practice: telemetry pre-processing aligns with the smoothing in \eqref{eq:nos-shot}–\eqref{eq:nos-kernel-norm}; the state $(v,u)$ evolves by \eqref{eq:nos-v-eq}–\eqref{eq:nos-u-eq} with resets \eqref{eq:nos-exp-reset}/\eqref{eq:nos-cont-reset}; events and residuals drive local detection and control; and spikes propagate over $W$ with delays \eqref{eq:nos-input-delayed} to form neighbours’ inputs. This sets the stage for the engineering choices that follow (reset time scale, EWMA leak $\chi$, gain scheduling $g$ vs.\ $\rho(W)$, and threshold calibration).

\begin{figure*}[h!]
\centering
\resizebox{\textwidth}{!}{%
\begin{tikzpicture}[
  >=LaTeX,
  node distance=13mm and 18mm,
  every node/.style={font=\small},
  box/.style={draw, rounded corners=2pt, align=center, minimum height=8.5mm, inner xsep=3mm, inner ysep=2mm},
  thinline/.style={line width=0.6pt},
  arrow/.style={-{Latex[length=2.4mm,width=1.6mm]}, line width=0.6pt},
  dashedarrow/.style={-{Latex[length=2.2mm,width=1.6mm]}, line width=0.6pt, dashed}
]

% ===== nodes (placed first so we can fit a background later) =====
\node[box] (tel) {Telemetry\\\footnotesize packets, queues, rates};
\node[box, below=of tel] (nos) {\textit{NOS} unit\\\footnotesize $v$ (queue), $u$ (recovery)};
\node[box, below=of nos] (det) {Detection / forecast\\\footnotesize spikes, residuals};
\node[box, below=of det] (ctl) {Control\\\footnotesize marking, pacing};
\node[box, below=18mm of ctl] (cpl) {Coupling to neighbours\\\footnotesize $W$, gates $g(q_{ij})$, delays $\tau_{ij}$};

\node[box, left=28mm of nos] (eta) {Exogenous drive\\\footnotesize $\eta_i(t)$ (shot noise)};

% annotation at right (free-floating)
\node[align=left, font=\scriptsize, text width=45mm, right=28mm of nos] (noteA)
{bounded excitability $f_{\mathrm{sat}}$\\
 service leak $\lambda$, damping $\chi$\\
 stoch.\ threshold $v_{\mathrm{th}}(t)$\\
 soft reset $r_{\text{reset}}$};

% spectral hint + coupling note
\node[left=22mm of cpl, align=center, font=\scriptsize, text width=50mm] (spec)
  {design: scale $W$ s.t.\ $\rho(W)$ meets a chosen stability margin};
\node[below=7mm of cpl, align=center, font=\scriptsize, text width=55mm] (noteB)
  {forms neighbour inputs\quad $I_j(t)=\sum_i w_{ji}\,g(q_{ji})\,S_i(t-\tau_{ji})$};

% ===== background group (light fill + dashed border) =====
\begin{pgfonlayer}{background}
  \node[
    draw,
    rounded corners=3pt,
    dashed,
    fill=gray!8,        % light fill
    inner sep=5mm,
    fit=(tel) (nos) (det) (ctl)
  ] (grp) {};
\end{pgfonlayer}

% ===== connectors and arrows (drawn after background) =====
\draw[thinline] (spec.east) -- (cpl.west);
\draw[thinline] (cpl.south) -- (noteB.north);

\draw[arrow] (tel) -- node[right, font=\scriptsize] {preproc./norm.} (nos);
\draw[arrow] (eta.east) -- node[above, font=\scriptsize, pos=0.55] {$\eta_i(t)$} (nos.west);
\draw[arrow] (nos) -- node[right, font=\scriptsize] {$\{v,u\}$, spikes to detector} (det);
\draw[arrow] (det) -- (ctl);

% curved spikes to coupling (clean landing at top edge)
\draw[arrow]
  (nos.east) .. controls ($(nos.east)+(22mm,0)$) and ($(cpl.north)+(28mm,10mm)$)
  .. ($(cpl.north)+(0mm,0)$)
  node[pos=0.45, right, font=\scriptsize, xshift=1mm] {$S_i(t)$};

% feedback (single dashed arrow on the left)
\draw[dashedarrow]
  (ctl.west) -- ++(-16mm,0) |- ($(tel.west)$)
  node[above, pos=0.65, font=\footnotesize, xshift=-1mm] {policy, thresholds};

% to neighbours
\draw[dashedarrow] (cpl.east) -- ++(18mm,0)
  node[right, align=left, font=\scriptsize] {to neighbours\\(updates their $I_j(t)$)};

% subtle leader from \emph{NOS} to annotation (offset to avoid the curved arrow)
\draw[thinline]
  ($(nos.east)+(3mm,3mm)$) -- ++(6mm,0) |- ($(noteA.west)+(0,3mm)$);

\end{tikzpicture}
}
\caption{Distributed flow around a \textit{NOS} unit. Telemetry and exogenous shot noise drive the unit. Its state \((v,u)\) feeds detection/control; spikes \(S_i(t)\) also propagate to neighbours via weighted, gated, and delayed couplings to form their inputs \(I_j(t)\). Control policies feed back to telemetry normalisation (dashed).}
\label{fig:nos-pipeline}
\end{figure*}
%%%%

For resets, one may use either an event–based exponential smoothing at threshold crossings or a continuous sigmoidal pullback that activates only above threshold (cf.\ \S\ref{subsec:nos-resets}); both are differentiable and avoid algebraic jumps, so gradient flow remains stable and attribution stays clear. In deployment terms, the return–to–baseline time is matched to the hardware drain or scheduler epoch, which means the model cools down on the same clock as the switch or NIC. Subthreshold prefiltering enters as a mild leak $\chi\,(v-v_{\mathrm{rest}})$ that behaves like an EWMA on residual queues: it reveals early warnings while damping tiny oscillations caused by counter jitter, yet preserves micro–bursts that matter for control.

Exogenous burstiness and measurement noise are included explicitly. We use shot–noise drive as in \S\ref{sec:nos-stochastic} and allow slow adaptation of tolerance and recovery so the unit becomes temporarily conservative during sustained stress. Concretely, the threshold may wander within a bounded band and the recovery rate may track recent activity
\begin{align}
v_{\mathrm{th}}(t) &= v_{\mathrm{th,base}} + \sigma\,\xi(t), 
\qquad
a(t) \;=\; a_0 + \kappa\,\overline{S}_i(t),
\label{eq:nos-adapt-params}
\end{align}
where $\xi(t)$ is bounded noise that sets the false–alarm \emph{budget} and $\overline{S}_i(t)$ is the short–window spike (arrival) rate. These adaptations match operator practice such as temporary pacing, token–bucket refill, or ECN–driven headroom.

Training follows standard practice while keeping the networking semantics intact. We employ surrogate gradients at the threshold with a fast–sigmoid derivative
\begin{align}
\widehat{\sigma}'(x) \;=\; \frac{1}{\bigl(1+\alpha_{\mathrm{sg}}|x|\bigr)^2},
\qquad \alpha_{\mathrm{sg}} \approx 5,
\label{eq:nos-surrogate}
\end{align}
which keeps gradients finite and centred. To approach crisp events without destabilising optimisation, we use a homotopy on the pullback sharpness $k$ that starts smooth and tightens once optimisation stabilises,
% \begin{align}
% \kappa_\sigma(t) \;=\; \kappa_{\sigma(0)}\;+\; \bigl(\kappa_{\sigma,\mathrm{final}}-k_0\bigr)\,\frac{t}{T_{\mathrm{h}}},
% \qquad k_0 \!\approx\! 1,\;\; \kappa_{\sigma,\mathrm{final}}\!\in\![20,100],
% \label{eq:nos-k-homotopy}
% \end{align}

\begin{align}
\begin{split}
\kappa_\sigma(t)
\;=\;
\kappa_{\sigma,0}
+\bigl(\kappa_{\sigma,\mathrm{final}}-\kappa_{\sigma,0}\bigr)\,
\min\!\left\{1,\frac{t}{T_{\mathrm{h}}}\right\},
\\
\kappa_{\sigma,0}\approx 1,\;\;
\kappa_{\sigma,\mathrm{final}}\in[20,100].
\label{eq:nos-k-homotopy}
\end{split}
\end{align}

which preserves gradient quality early and matches event timing late. We also clip gradients by global norm in the range $[0.5,2.0]$ while $k$ grows. An adaptive optimiser such as Adam is suitable; the learning rate is reduced as $k$ increases to maintain stable steps near the threshold. Truncated BPTT should cover the dominant recovery timescale so that $u$’s memory is learned rather than aliased,
\begin{align}
\begin{split}
T_{\mathrm{BPTT}} \;\gtrsim\; \frac{c_{\!*}}{a+\mu},
\\ c_{\!*}\in[3,5]\ \text{(capturing about 95\%–99\%)},
\label{eq:nos-bptt-window}
\end{split}
\end{align}
which aligns with the cool–down used by paced drain or token–bucket refill in devices. After supervised pre–training, slow local updates can be enabled to track diurnal load or policy shifts.

The network view stays explicit. Delays $\tau_{ij}$ and gates $g\!\bigl(q_{ij}(t)\bigr)$ ensure that path timing and residual capacity modulate influence at the right hop and time. The small–signal proxy in \eqref{eq:nos-reading-threshold} makes the stability levers visible to operations: one can lower $\rho(W)$ by reweighting or sparsifying couplings, reduce the global gain $g$ by scheduling, or increase the net drain $\Lambda$ by raising service $\lambda$, damping $\chi$, or recovery $a$ (with $b$ and $\mu$ setting trade–offs). These are the same knobs used in practice to keep queues bounded and alerts meaningful, and they connect directly to the analysis that follows.

%%%%
\subsection{Neuromorphic implementation details}
\label{sec:neuromorphic}

\paragraph{Fixed-point state and quantisation.}
Let $(v,u)$ be stored in $Q_{m.n}$ fixed point with full-buffer scaling $V$ and time base $T$ from \eqref{app:nos-scaling}. The forward map and clipping are
\begin{align}
\begin{split}
v_q = \frac{1}{2^n}\,\mathrm{clip}\!\left(\mathrm{round}\!\left(\frac{v}{V}\,2^n\right),\,0,\,2^m-1\right), 
\\
u_q \;=\; \frac{1}{2^n}\,\mathrm{clip}\!\left(\mathrm{round}\!\left(\frac{u}{V}\,2^n\right),\,0,\,2^m-1\right).
\label{eq:nm-quant-v}
\end{split}
\end{align}
With this scaling, $\tilde v\in[0,1]$ remains interpretable as normalised queue level; overflow is prevented by the saturation in $f_{\mathrm{sat}}$.

\paragraph{Discrete-time soft reset.}
On digital substrates the exponential reset \eqref{eq:nos-exp-reset} is applied per tick $\Delta t$ only when a spike occurs. Let $\tilde v_{t+1}$ denote the state after the continuous update but before any reset. Then
\begin{align}
\begin{split}
v_{t+1} = (1-S_i(t))\,\tilde v_{t+1}
          + S_i(t)\Big[c +  \big(\tilde v_{t+1}-c\big)\,e^{-r_{\text{reset}}\Delta t}\Big],\\
u_{t+1} = u_t + d\,S_i(t).
\label{eq:nm-exp-disc}
\end{split}
\end{align}
where $S_i(t)\in\{0,1\}$ indicates a threshold crossing (a spike) at time $t$ for neuron $i$.
The factor $e^{-r_{\text{reset}}\Delta t}$ is realised by a multiply–accumulate or by a bit-shift if $r_{\text{reset}}\Delta t$ is drawn from a small table.

\paragraph{Logistic pullback via LUT or PWL.}
The continuous pullback \eqref{eq:nos-cont-reset} uses $\sigma_{\kappa_\sigma}(x)=1/(1+e^{-\kappa_\sigma x})$. On hardware we use a lookup table or piecewise-linear approximation $\widehat{\sigma}_{\kappa_\sigma}$ on a bounded interval $x\in[-x_{\max},x_{\max}]$:
\begin{align}
\bigl|\sigma_{\kappa_\sigma}(x)-\widehat{\sigma}_{\kappa_\sigma}(x)\bigr| &\le \varepsilon_{\mathrm{LUT}} \quad \text{for } |x|\le x_{\max}.
\label{eq:nm-sig-bound}
\end{align}
The induced drift error on the pullback term is then bounded by $r_{\text{reset}}\,|v-c|\,\varepsilon_{\mathrm{LUT}}$ and can be kept below a configured fraction of the net drain in \eqref{eq:nos-reading-lambda}.

\paragraph{Graph coupling, delays, and AER routing.}
Sparse coupling $W=[w_{ij}]$ is realised by address–event routing tables; per-link delays are integer tick buffers
\begin{align}
\tilde \tau_{ij} &= \mathrm{round}\!\left(\frac{\tau_{ij}}{T}\right) \in \mathbb{N},
\label{eq:nm-delay-quant}
\end{align}
which preserves the causality of \eqref{eq:nos-input-delayed}. Gates $g\!\bigl(q_{ij}\bigr)$ are evaluated where the link queue state is maintained; their outputs modulate the synaptic multiplier before accumulation.

\paragraph{Fabric-rate and budget constraints.}
Let $R_{\max}$ be the sustainable spike throughput per core and let $\bar r_j$ be the observed rate of $S_j$. A conservative admission constraint that avoids router overload is
\begin{align}
\sum_{j} \mathbf{1}\{w_{ij}\neq 0\}\,\bar r_j &\le R_{\max} \quad \text{for each core hosting neuron } i,
\label{eq:nm-rate-budget}
\end{align}
or, in matrix form for a deployment-wide check, $\|W\|_{0,1}\, \bar r_{\max} \le R_{\max}$, where $\|W\|_{0,1}$ counts nonzeros per row and $\bar r_{\max}=\max_j \bar r_j$.

\paragraph{Operational margin under quantisation.}
Let $\Delta_{\mathrm{net}}$ be the margin in \eqref{eq:nos-reading-margin}. Finite precision induces errors in $f'_{\mathrm{sat}}(v^*)$, in $W$, and in gains. If $\delta_{f'}$, $\delta_W$, and $\delta_g$ bound these relative errors, then a sufficient robustness condition is

\begin{align}
\begin{split}
\Delta_{\mathrm{net}}
&>\; f'_{\mathrm{sat}}(v^*)\,\delta_{f'}
\;+\; g\,\rho(W)\,\bigl(\delta_W+\delta_g\bigr)
\\
&\quad+\; \varepsilon_{\mathrm{LUT}}\,\rho(W)\,\mathbb{E}\{|v-c|\},
\label{eq:nm-margin}
\end{split}
\end{align}

which ensures that the stability proxy \eqref{eq:nos-reading-threshold} still holds after rounding and table approximations.

\paragraph{Networking alignment and reporting.}
Choose $(V,T)$ so that $V$ matches the full-buffer level used by the QoS policy and $T$ matches the scheduler epoch or drain time. Report energy per spike and energy per detection alongside CPU baselines, and log $\tilde\tau_{ij}$ histograms to confirm that delay quantisation preserves path ordering. These checks keep the deployed \emph{NOS} consistent with the queueing semantics and timing used in the main text.

\section{Spike generation and resets}
\label{subsec:nos-resets}

Two differentiable reset strategies are introduced, an event-triggered
exponential soft reset and a continuous pullback shaped by a sigmoid, both enabling gradient-based training while preserving spiking dynamics. A spike indicates that the local congestion proxy exceeded tolerance and triggers a control action such as ECN marking, rate reduction, or an alert to the controller. To reflect tolerance bands and measurement noise we allow a stochastic threshold with bounded dispersion:
\begin{align}
v_i(t) &\ge v_{\mathrm{th}}(t) \;=\; v_{\mathrm{th,base}} + \sigma\,\xi(t),
\label{eq:nos-threshold}
\end{align}
where $\xi(t)$ is zero-mean, bounded (e.g., clipped Gaussian or uniform), and $\sigma$ is tuned to the false-alarm budget on residuals. This construction separates policy (the base threshold) from instrumentation noise (the dispersion).

\paragraph{Event-based exponential soft reset.}
Right after a threshold crossing, the unit should return toward a baseline without a hard discontinuity. For a forward-Euler step of size $\Delta t$ we apply
\begin{align}
v_i &\leftarrow c + \bigl(v_i - c\bigr) e^{-r_{\text{reset}} \Delta t}, 
\qquad 
u_i \leftarrow u_i + d,
\label{eq:nos-exp-reset}
\end{align}
which makes $v_i$ relax exponentially toward $c$ with rate $r_{\text{reset}}$ and increments $u_i$ by $d$ to encode refractory depth. In networking terms, \eqref{eq:nos-exp-reset} models a paced drain after a trigger together with a short-lived slow-down in effective send or admit rate. The pair $(r_{\text{reset}},d)$ is set from post-burst relaxation fits on $v$ and the desired “cool-down” depth in $u$.

\begin{figure*}[t]
\centering
\begin{subfigure}[b]{0.32\textwidth}
  \centering
  \includegraphics[width=\linewidth]{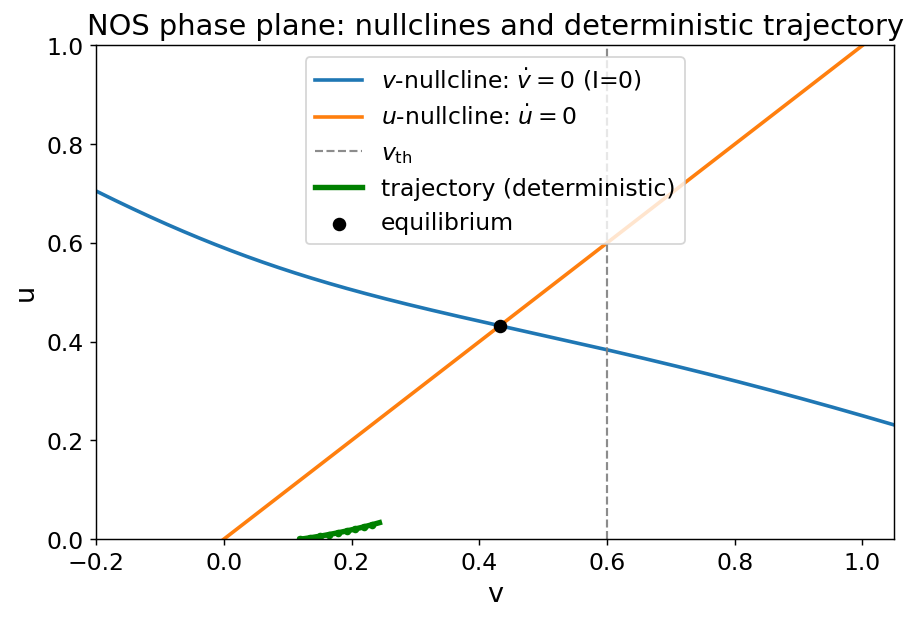}
  \caption{Deterministic  (\(I\!=\!0\)).}
  \label{fig:pp-clean}
\end{subfigure}\hfill
\begin{subfigure}[b]{0.32\textwidth}
  \centering
  \includegraphics[width=\linewidth]{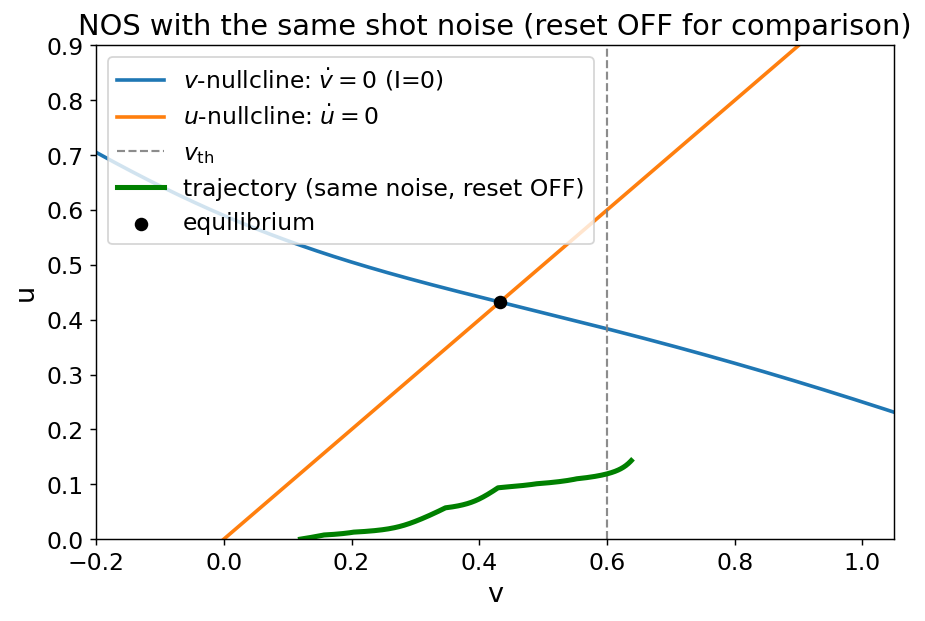}
  \caption{Shot noise, reset \emph{off}.}
  \label{fig:pp-noisy-off}
\end{subfigure}\hfill
\begin{subfigure}[b]{0.32\textwidth}
  \centering
  \includegraphics[width=\linewidth]{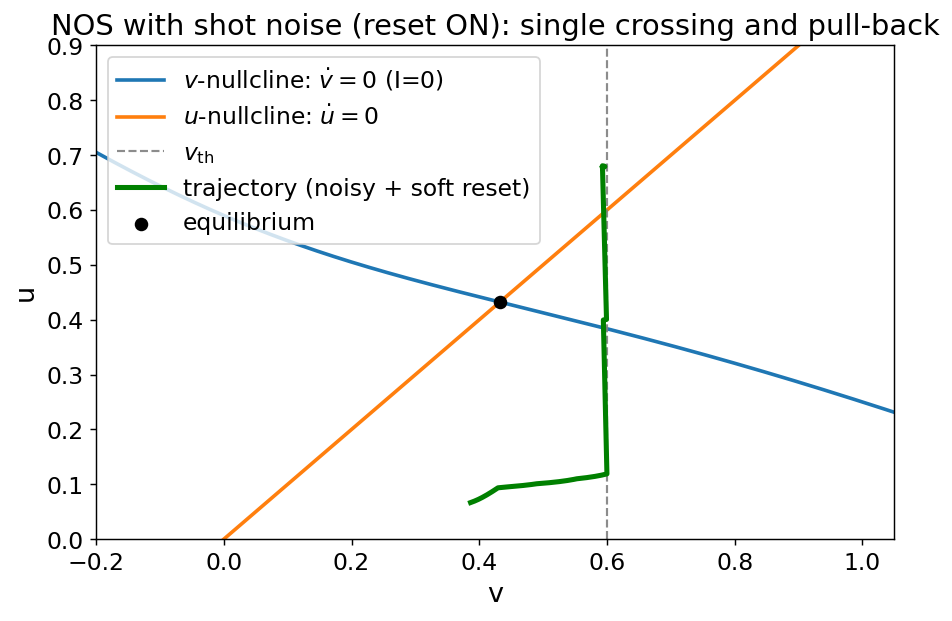}
  \caption{Shot noise, soft reset \emph{on}.}
  \label{fig:pp-noisy-on}
\end{subfigure}

\caption{%
NOS phase–plane diagnostics with nullclines \( \dot v=0 \) and \( \dot u=0 \), threshold \(v_{\mathrm{th}}\), equilibrium, and short trajectories. 
\textbf{(a)} Baseline geometry: the parabolic \(v\)-nullcline intersects the linear \(u\)-nullcline at a stable equilibrium; the threshold sits to the right. 
\textbf{(b)} With identical exogenous shot noise but reset disabled, the trajectory creeps toward the threshold and remains above the \(u\)-nullcline, indicating slow accumulation of \emph{effective} load that can lead to prolonged excursions. 
\textbf{(c)} With the same noise and the exponential soft reset active, the trajectory crosses the threshold once, is immediately pulled back toward \(c\), and rises in \(u\), which encodes a temporary conservative phase. 
Operationally, panel (c) mirrors paced drain after an alarm: timing is preserved by the threshold crossing, the return to baseline follows the configured time scale \(r_{\text{reset}}^{-1}\), and the recovery increment \(d\) captures the short-lived send-rate reduction.
}
\label{fig:nos-phaseplane-three}
\end{figure*}

\paragraph{Continuous differentiable pullback.}
Alternatively, we add to the $v$-dynamics in \eqref{eq:nos-v-eq} a smooth term that engages only above threshold:
\begin{align}
\begin{split}
   -r_{\text{reset}}\,\sigma_{\kappa_\sigma}\!\bigl(v_i - v_{\mathrm{th}}(t)\bigr)\,\bigl(v_i - c\bigr),\\
\sigma_{\kappa_\sigma}(x)=\frac{1}{1+e^{-\kappa_\sigma x}},\;\; r_{\text{reset}}>0.
\label{eq:nos-cont-reset} 
\end{split}
\end{align}
For $v_i<v_{\mathrm{th}}(t)$ the multiplier $\sigma_{\kappa_\sigma}$ is near $0$ and the pullback is negligible. Once $v_i$ exceeds the threshold, $\sigma_{\kappa_\sigma}$ rises toward $1$ and an exponential attraction to $c$ turns on as a continuous part of $\dot v$. Training proceeds with a homotopy in $\kappa_\sigma$, starting smooth and sharpening as optimisation stabilises, which preserves gradient quality while converging to crisp event timing.
For reference, the nullclines of the no–input skeleton are
\begin{align}
\dot v=0 &: \;
u \;=\; f_{\mathrm{sat}}(v) + (\beta-\lambda-\chi)\,v + \gamma + \chi v_{\mathrm{rest}},
\label{eq:nos-v-null}\\
\dot u=0 &: \;
u \;=\; \frac{ab}{a+\mu}\,v 
\label{eq:nos-u-null}
\end{align}

Figure~\ref{fig:nos-phaseplane-three} clarifies the role of the smooth pullback in a networking setting. In panel~\ref{fig:pp-clean} the skeleton geometry is benign: the parabolic \(v\)-nullcline \eqref{eq:nos-v-null} meets the linear \(u\)-nullcline \eqref{eq:nos-u-null} at a stable point well to the left of the threshold. In panel~\ref{fig:pp-noisy-off}, identical shot–noise drive nudges \(v\) upward but, with the reset disabled, there is no restoring action triggered at the moment of crossing; the trajectory can linger near \(v_{\mathrm{th}}\), which in practice sustains high send pressure on downstream links. In panel~\ref{fig:pp-noisy-on}, the same bursts cause a single threshold crossing followed by an immediate, continuous return toward \(c\). The concurrent rise in \(u\) encodes a short conservative phase, matching paced drain or token–bucket refill after an alarm. Because \eqref{eq:nos-cont-reset} is differentiable, these operator–visible effects are trainable: \(r_{\text{reset}}^{-1}\) is tuned to device drain time or scheduler epoch, \(c\) to the desired post–event baseline, and \(d\) to the depth of temporary slow–down. The model therefore reproduces both the timing of decisions and the engineered cool–down that follows, without introducing discontinuities that harm gradient-based learning or blur attribution.

\section{Local slope bounds and saturation properties}
\label{app:fsat-bounds}

For convenience we collect here the elementary properties of the bounded excitability function
\begin{align}
f_{\mathrm{sat}}(v) = \frac{\alpha v^2}{1+\kappa v^2}, \qquad \alpha>0,\ \kappa>0,
\tag{\ref{eq:nos-fsat} revisited}
\end{align}
used in Section~\ref{subsec:nos-bounded-excitability}.
These properties underpin the existence and uniqueness conditions for the subthreshold equilibrium and the small–signal stability margins.

We first record its small–signal expansion and limiting value:
\begin{align}
f_{\mathrm{sat}}(v) &\sim \alpha v^2 \quad \text{as } v\to 0,
\label{eq:nos-fsat-small}
\end{align}
and enforces a finite ceiling for heavy load, ensuring bounded behaviour
\begin{align}
\lim_{v\to\infty} f_{\mathrm{sat}}(v) &= \frac{\alpha}{\kappa}.
\label{eq:nos-fsat-limit}
\end{align}
The derivative is globally bounded, which improves numerical stability and yields clean gain conditions, which means, it recovers a quadratic slope for small $v$ while enforcing saturation for large $v$ :
\begin{align}
f'_{\mathrm{sat}}(v) &= \frac{2\alpha v}{\bigl(1+\kappa v^2\bigr)^2},
\label{eq:nos-fsat-deriv}
\end{align}
is globally bounded, with maximum
\begin{align}
\max_{v} f'_{\mathrm{sat}}(v) &= \frac{3\sqrt{3}}{8}\,\frac{\alpha}{\sqrt{\kappa}}
\quad \text{at } v=\frac{1}{\sqrt{3\kappa}}.
\label{eq:nos-fsat-deriv-max}
\end{align}
The quantity in \eqref{eq:nos-fsat-deriv-max} is the steepest admissible local growth of the excitability term and appears in the service–dominance inequalities used for equilibrium uniqueness and Jacobian stability tests.

The consequences of this bounded excitability under stochastic drive
are examined in \emph{methods}~\ref{subsec:nos-noise} and ~\ref{noise_avalanche}.

%%%%%%%%%

\section{Noise sensitivity under stochastic drive}
\label{subsec:nos-noise}

We model arrivals as shot noise with exponential smoothing. For node $i$,
\begin{align}
\begin{split}
\eta_i(t) = \sum_{n=1}^{N_i(t)} A_{i,n}\,e^{-(t-t_{i,n})/\tau_s}\,H(t-t_{i,n}),
\\ N_i(t)\sim \mathrm{Poisson}(\nu_i t),
\label{eq:shotnoise}
\end{split}
\end{align}
where $\nu_i$ is the event rate, $A_{i,n}$ are i.i.d.\ amplitudes, $\tau_s$ is a smoothing time, and $H$ is the Heaviside step. The process is stationary with mean, variance, and autocovariance
\begin{align}
\begin{split}
\mathbb{E}[\eta_i] = \nu_i\,\mathbb{E}[A_i]\,\tau_s,
\qquad
\mathrm{Var}[\eta_i] = \nu_i\,\mathbb{E}[A_i^2]\,\frac{\tau_s}{2},
\\
\mathrm{Cov}_{\eta_i}(\tau) = \nu_i\,\mathbb{E}[A_i^2]\,\frac{\tau_s}{2}\,e^{-|\tau|/\tau_s}.
\label{eq:shot_moments}
\end{split}
\end{align}
Its power spectrum is Lorentzian,
\begin{align}
S_{\eta_i}(\omega) &= \frac{2\,\nu_i\,\mathbb{E}[A_i^2]\,\tau_s}{1+\omega^2 \tau_s^2}.
\label{eq:shot_psd}
\end{align}
In networking terms, $(\nu_i,\mathbb{E}[A_i],\tau_s)$ are read from the same counters and windows used in telemetry: $\nu_i$ is the burst-start rate, $\mathbb{E}[A_i]$ is the mean per-burst mass after smoothing, and $\tau_s$ matches the prefilter that removes counter jitter while preserving micro-bursts.

\paragraph{Small-signal sensitivity.}
Linearising \textit{NOS} at a subthreshold equilibrium $(v^*,u^*)$ with $\bar d=f'_{\mathrm{sat}}(v^*)+\beta-\lambda-\chi$ gives
\begin{align}
\begin{split}
\frac{d}{dt}
\begin{bmatrix}\delta v\\ \delta u\end{bmatrix}
=
\begin{bmatrix}\bar d & -1\\ ab & -(a+\mu)\end{bmatrix}
\begin{bmatrix}\delta v\\ \delta u\end{bmatrix}
+
\begin{bmatrix}1\\ 0\end{bmatrix}\eta_i(t),
\\
H_v(s) \;=\; \frac{\delta V(s)}{\Eta(s)} \;=\; \frac{s+(a+\mu)}{(s-\bar d)\,(s+(a+\mu))+ab}.
\label{eq:lin_transfer} 
\end{split}
\end{align}
The DC gain and the variance of $\delta v$ under \eqref{eq:shot_psd} are
\begin{align}
\begin{split}
H_v(0) = \frac{a+\mu}{ab-(a+\mu)\bar d},
\\
\sigma_v^2 \;=\; \frac{1}{2\pi}\int_{-\infty}^{\infty} \bigl|H_v(i\omega)\bigr|^2\,S_{\eta_i}(\omega)\,d\omega.
\label{eq:lin_dc_var}
\end{split}
\end{align}
As the node approaches its local margin ($ab\downarrow (a+\mu)\bar d$) or the network approaches the Perron-mode threshold ($k\to k^\star$), $|H_v(0)|$ increases and the integral in \eqref{eq:lin_dc_var} grows, so the same shot-noise trace produces larger queue excursions and more threshold crossings.

\paragraph{Firing statistics and cascades.}
We summarise sensitivity by the mean firing rate $\bar f$, the inter-spike-interval coefficient of variation (CV), and avalanche sizes $S$ (contiguous above-threshold activity aggregated over neighbours). With parameters fixed and only the noise amplitude varied, three robust effects appear, consistent with \eqref{eq:lin_transfer}–\eqref{eq:lin_dc_var} and bounded excitability:
(i) $\bar f$ rises with amplitude in all regimes, with a steeper slope near the coupling threshold because small increments in $\eta_i$ map to larger $\delta v$;
(ii) CV increases near threshold, reflecting irregular, burst-dominated spike trains as variance grows;
(iii) the tail of $P(S)$ becomes heavier with amplitude, indicating extended congestion cascades when fluctuations push multiple nodes above threshold before recovery completes.

\paragraph{Networking interpretation.}
Higher $\nu_i$ or larger $\mathbb{E}[A_i^2]$ means burstier ingress; increasing $\tau_s$ models longer burst coherence. The levers that reduce noise sensitivity are the same that enlarge deterministic headroom: higher service and subthreshold damping ($\lambda$, $\chi$), faster recovery ($a+\mu$), and earlier saturation (larger $\kappa$, which lowers $f'_{\mathrm{sat}}(v^*)$ and hence $\bar d$). Each reduces $H_v(0)$ and the variance integral in \eqref{eq:lin_dc_var}. With per-link gates, attenuating $w_{ij}$ on noisy edges reduces the effective drive entering $I_i$ and moves the network away from the Perron threshold.

\paragraph{Empirical summary (Fig.~\ref{fig:noise_stats}).}
The rate curves show $\bar f$ versus amplitude $A$ for subcritical ($k=0.9$) and near-threshold ($k=1.36$) regimes; the steeper slope near $k=1.4$ matches the growth of $|H_v(0)|$. The CV curves increase with $A$ and peak near threshold, consistent with \eqref{eq:lin_dc_var}. The avalanche distributions broaden as $A$ increases; heavier tails near $k\approx k^\star$ reflect larger correlated excursions. These trends are stable across $\rho_s\in\{10,20,50\}$\,Hz and the reported $\tau_s$ range, and they align with the linear response in \eqref{eq:lin_transfer}. Further quantitative tail fits are provided below in~\ref{noise_avalanche}.

\begin{figure*}[h!]
\centering
\includegraphics[width=0.5\textwidth]{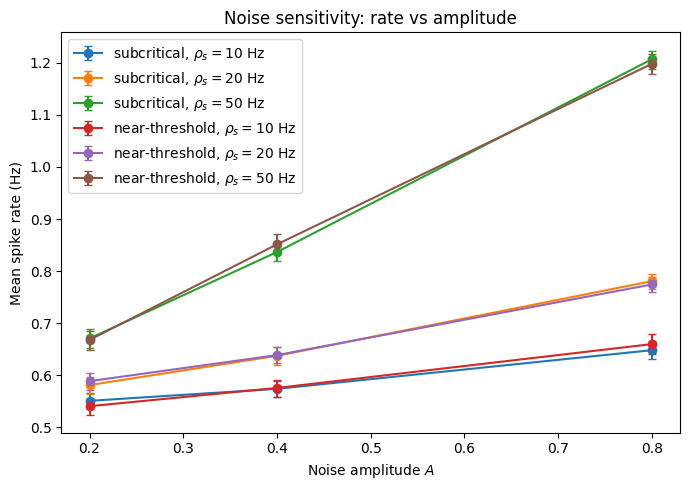}\hfill
\includegraphics[width=0.5\textwidth]{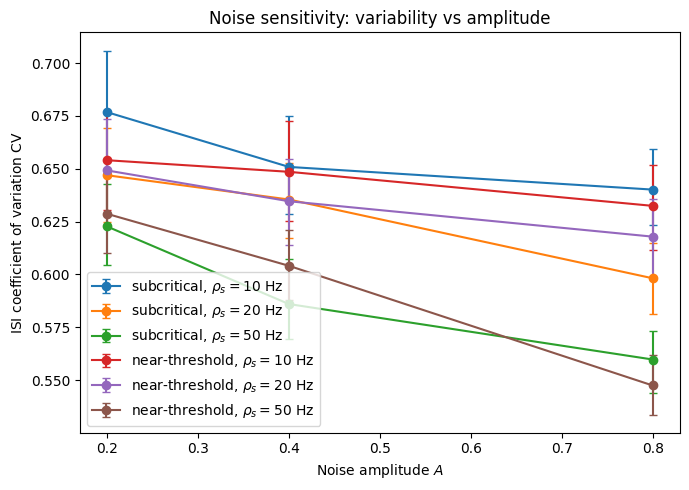}
\includegraphics[width=0.5\textwidth]{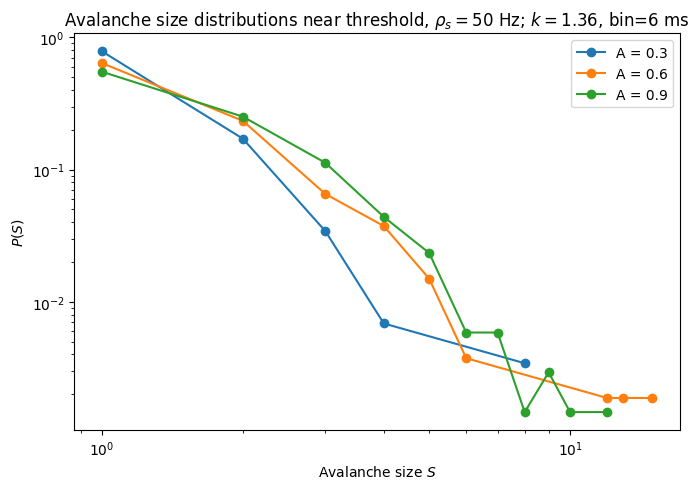}
\caption{Noise sensitivity in \textit{NOS}. 
(\textbf{Top left}) Mean firing rate $\bar f$ versus amplitude $A$ for subcritical ($k=0.9$) and near–threshold ($k=1.36$) regimes at $\rho_s\in\{10,20,50\}$\,Hz.  
(\textbf{Top right}) Inter–spike interval CV versus $A$ under the same conditions. Error bars show bootstrap 95\% confidence intervals.  
(\textbf{Bottom}) Avalanche size distributions $P(S)$ near threshold ($k=1.36$, $\rho_s=50$\,Hz). Heavier tails appear as $A$ increases, corresponding to extended cascades.  
}
\label{fig:noise_stats}
\end{figure*}

\subsection{Statistical robustness under noise}
\label{noise_avalanche}

For each condition we estimated the mean firing rate $\bar f$ and the inter-spike-interval coefficient of variation (CV) using non-parametric bootstrap resampling (hundreds of replicates). Bootstrap 95\% intervals show that both $\bar f$ and CV increase with shot-noise amplitude $A$, particularly near threshold. In \textit{NOS}, stronger stochastic drive raises overall excitability and increases irregularity, mirroring more variable queue dynamics.

Avalanche size distributions $P(S)$ were then analysed. Figure~\ref{fig:avalanche_fits} shows empirical complementary cumulative distributions with fitted power-law and log-normal overlays above a data-driven tail threshold $x_{\min}$. Tail parameters were estimated by maximum likelihood with truncation at $x_{\min}$; goodness was assessed by the Kolmogorov–Smirnov distance and model choice by AIC. Across conditions, AIC favours the power-law, indicating scale-free-like behaviour near threshold. Fitted values are reported in Table~S\ref{tab:avalanche_fits}.

The exponents display a clear trend: at low amplitude ($A=0.3$) the power-law exponent is $\hat\alpha\approx 8.7$ (steep decay, small avalanches). At $A=0.6$ the exponent decreases to $\hat\alpha\approx 6.4$ (broader tail). At $A=0.9$ the exponent drops to $\hat\alpha\approx 3.3$ (heavy tail and a higher chance of large cascades). Thus, near the bifurcation, stronger stochastic drive allows small fluctuations to trigger system-wide congestion events.

\begin{figure}[h]
\centering
\includegraphics[width=.8\columnwidth]{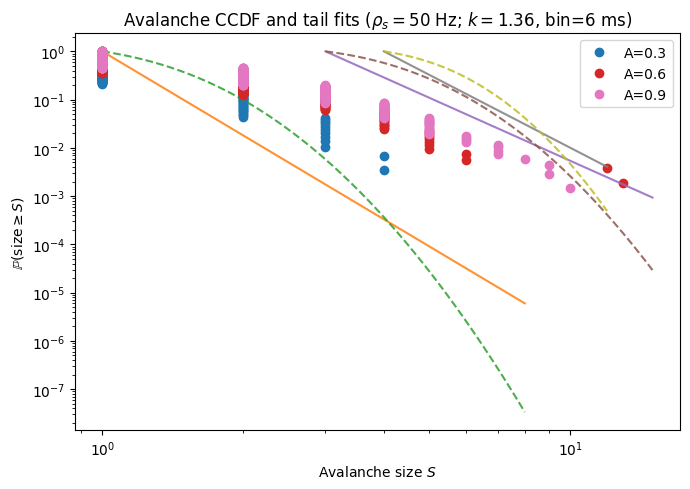}
\caption{Avalanche complementary CDFs near threshold ($k=1.36$, $\rho_s=50$\,Hz). Empirical distributions (points) with fitted power-law (solid) and log-normal (dashed) overlays above the estimated $x_{\min}$ for each amplitude $A$. Model selection used AIC; numerical results appear in Table~S\ref{tab:avalanche_fits}.}
\label{fig:avalanche_fits}
\end{figure}

\begin{table*}[h]
\centering
\caption{Avalanche tail fits near threshold ($k=1.36$, $\rho_s=50$\,Hz). For each amplitude $A$, we report the tail threshold $x_{\min}$, tail size $n$, maximum-likelihood parameters, KS distance (for the truncated tail), log-likelihood (LL), AIC, and the preferred model.}
\label{tab:avalanche_fits}
\fontsize{9}{11}\selectfont
\setlength{\tabcolsep}{4pt}
\renewcommand{\arraystretch}{1.1}
\begin{tabular}{ccccccccccccccc}
\toprule
$A$ & $x_{\min}^{\text{PL}}$ & $n_{\text{PL}}$ & $\hat\alpha$ & KS$_{\text{PL}}$ & LL$_{\text{PL}}$ & AIC$_{\text{PL}}$ &
$x_{\min}^{\text{LN}}$ & $n_{\text{LN}}$ & $\hat\mu$ & $\hat\sigma$ & KS$_{\text{LN}}$ & LL$_{\text{LN}}$ & AIC$_{\text{LN}}$ & Pref. \\
\midrule
0.3 & 2.0 & 63  & 8.73 & 0.76 & 14.03   & -26.07  & 2.0 & 63  & 0.82 & 0.27 & 0.76 & -34.20 & 72.39   & PL \\
0.6 & 4.0 & 31  & 6.43 & 0.58 & -27.24  & 56.48   & 4.0 & 31  & 1.57 & 0.31 & 0.58 & -46.15 & 96.31   & PL \\
0.9 & 1.0 & 681 & 3.25 & 0.53 & -431.13 & 864.26  & 1.0 & 681 & 0.44 & 0.53 & 0.53 & -684.68 & 1373.35 & PL \\
\bottomrule
\end{tabular}
\end{table*}

\paragraph{Avalanche definition and fitting.}
Avalanche sizes $S$ were computed from population spike counts in 5\,ms bins (the experimental $\Delta t$): an avalanche is a contiguous run of nonzero bins, with size equal to the total spikes in the run. For each condition, $x_{\min}$ minimised the KS distance between empirical and model CDFs (tail only). Power-law tails used the continuous-tail MLE
$\hat\alpha = 1 + n \big/ \sum_{i=1}^n \log(x_i/x_{\min})$
with model $F(x)=1-(x/x_{\min})^{1-\hat\alpha}$ for $x\ge x_{\min}$. Log-normal tails were fit by MLE on $\log x$ with truncation at $x_{\min}$. Model selection used AIC with parameter counts 1 (power-law) and 2 (log-normal).

\section{Queueing baseline comparison}
\label{subsec:mm1}

\paragraph{Experimental setup.}
We compare \textit{NOS} against canonical queueing baselines in two modes. 
(\emph{i}) \textbf{Open--loop}: we drive M/M/1 (analytic mean, $\rho<1$), simulated M/M/1/$K$, and a single \emph{NOS} unit with \emph{the same} exogenous arrivals and report observables in the \emph{same unit} (packets). 
Light–load calibration aligns means: we choose an input gain and output scale so that the \emph{NOS} mean at small $\rho$ matches the M/M/1 mean $\mathbb{E}[L]=\rho/(1-\rho)$. 
Tail behaviour is then tested under a common MMPP burst sequence (same ON/OFF epochs), which is the operational regime of interest.
(\emph{ii}) \textbf{Closed--loop}: we attach three controllers to the same offered–load trace and compare the resulting queue trajectories and marking signals $p(t)$: 
A) NOS–based (marking derived from \emph{NOS} subthreshold state with soft reset), 
B) \emph{queue} controller (marking from the raw queue), and 
C) \emph{LP–queue} controller (marking from a lightly low–pass filtered queue). 
All three use the same logistic nonlinearity for $p(t)$ so differences are due to the state fed to the nonlinearity.

\begin{figure}[h!]
  \centering
  \includegraphics[width=.6\textwidth]{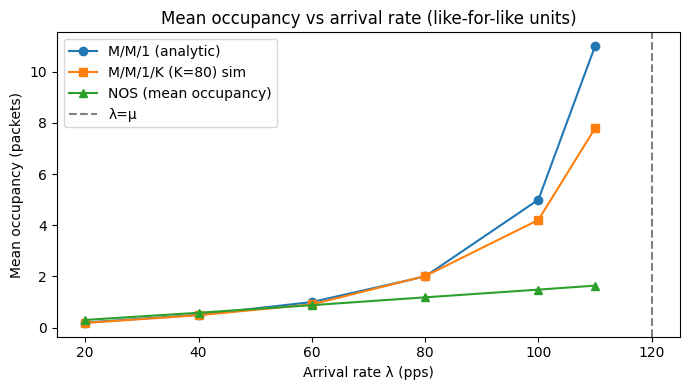}\hfill
  \includegraphics[width=.6\textwidth]{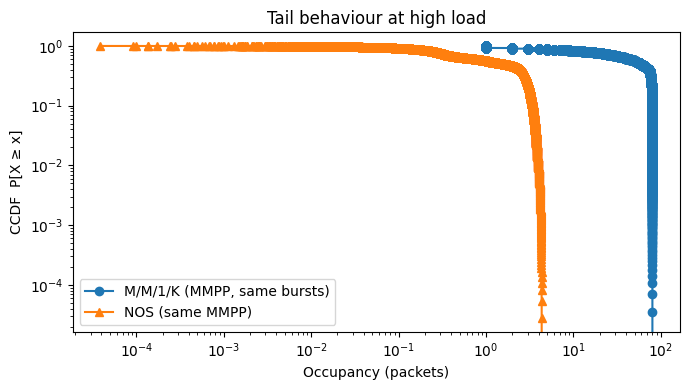}
  \caption{Open--loop comparisons in packet units. 
  (\textbf{Left}) Mean occupancy versus arrival rate $\lambda_{\mathrm{arr}}$: M/M/1 (analytic, $\rho<1$), simulated M/M/1/$K$, and \emph{NOS} (calibrated at light load). \emph{NOS} bends earlier near saturation due to bounded excitability and soft reset. 
  (\textbf{Right}) Tail CCDF under a common MMPP burst sequence: M/M/1/$K$ displays a slow tail, whereas \emph{NOS} truncates cascades via continuous pull–back, yielding a much lighter tail.}
  \label{fig:openloop}
\end{figure}
\paragraph{How to read the open–loop panels (networking view).}
Figure~\ref{fig:openloop} (left) shows mean occupancy versus arrival rate $\lambda_{\mathrm{arr}}$ with a fixed service $\mu_{\mathrm{srv}}$. 
The blue curve (M/M/1) is the textbook reference that diverges as $\rho=\lambda_{\mathrm{arr}}/\mu_{\mathrm{srv}}\uparrow1$. 
The orange curve (M/M/1/$K$) follows M/M/1 until blocking becomes material. 
The green curve (NOS) rises with load at small $\rho$ (by calibration) but \emph{bends earlier} as $\rho$ grows; bounded excitability and continuous pull–back reduce the subthreshold slope and drain excursions. 
Operationally, this mirrors how modern AQMs aim to keep average queues low near saturation rather than letting them diverge smoothly.

Figure~\ref{fig:openloop} (right) plots the complementary CDFs of occupancy under the \emph{same} MMPP burst sequence. 
M/M/1/$K$ develops a long tail during ON periods (large probability of deep queues), whereas \emph{NOS} exhibits a much sharper tail because soft reset and leak truncate cascades. 
In networking terms, \emph{NOS} converts the same burst train into fewer deep–queue events---directly improving tail latency and reducing drop/mark storms. 
The aggressiveness of this truncation is tunable through the subthreshold leak $\chi$, the reset rate $r_{\mathrm{reset}}$, and the saturation knee $\kappa$; operators can place the tail against SLO targets (e.g., $\mathbb{P}\{Q\!\ge\! q_0\}\le\varepsilon$).

\paragraph{How to read the closed–loop panels (networking view).}
Figure~\ref{fig:closedloop} contrasts three marking strategies under the same offered–load traces. 
The top panel shows controller outputs $p(t)$. 
NOS (blue) is decisive and low–jitter; the queue controller (orange) is noisy (mark jitter follows queue noise); the LP–queue controller (green) is smoother but lags, so it reacts late to bursts. 
The middle panel shows a step in offered load: \emph{NOS} snaps to target with minimal ringing; the queue controller oscillates; the LP–queue controller overshoots and then settles slowly. 
The bottom panel shows a burst train: \emph{NOS} produces short, clipped spikes; the queue controller shows noisy peaks; the LP–queue controller permits taller spikes because filtering delays the reaction. 
Practically, \emph{NOS} shrinks the time spent at high queue levels and reduces burst amplification without a long averaging window, which improves tail latency and fairness while avoiding prolonged ECN saturation.

\paragraph{Summary for operators.}
The baselines anchor \emph{NOS} against established theory. 
When the objective is to \emph{match the classical infinite–buffer mean} at light load, M/M/1 gives the right reference and \emph{NOS} can be calibrated to agree. 
When the objective is to \emph{manage burst risk and stability} under finite resources and coupling, NOS’ bounded excitability and differentiable pull–back are strictly more faithful: they reduce deep–queue probability, yield crisper ECN/marking, and shorten congestion episodes without heavy filtering. 
The same knobs that appear in our stability proxies---$\chi$, $\lambda_{\mathrm{arr}}$–like drain, $\kappa$, and reset parameters---map directly to deployed policy controls (AQM leak, scheduler epoch, ramp shaping, and post–alarm conservatism), making the model both predictive and actionable.
Having established mechanistic fidelity and control behaviour, we next ask how \emph{NOS} fares as a label-free forecaster against compact ML/SNN baselines.

\begin{figure*}[t]
  \centering
  \includegraphics[width=.7\textwidth]{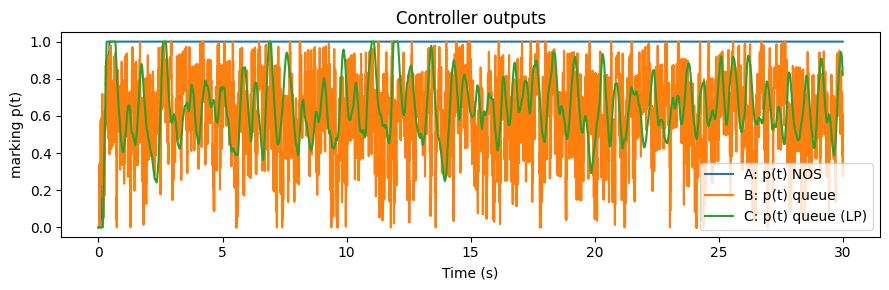}\\(4pt]
  \includegraphics[width=.7\textwidth]{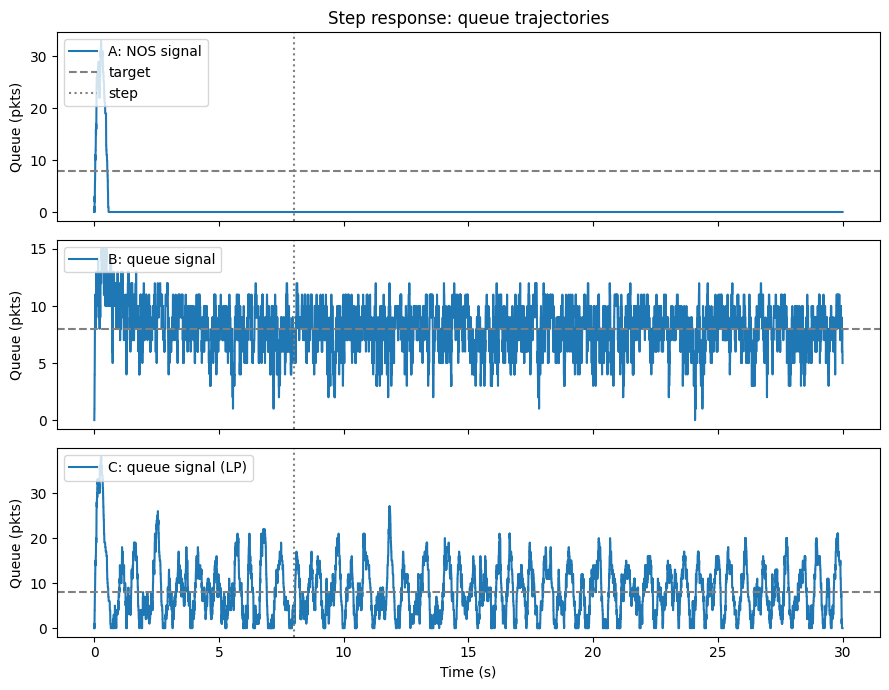}\\(4pt]
  \includegraphics[width=.7\textwidth]{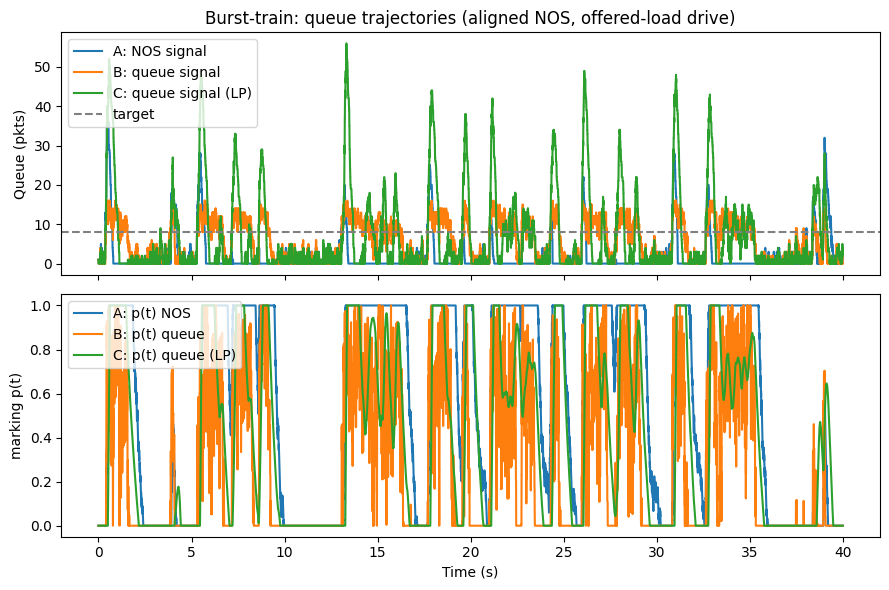}
  \caption{Closed--loop behaviour under identical offered–load traces. 
  (\textbf{Top}) Marking signals $p(t)$ from three controllers: \emph{NOS} (blue), raw queue (orange), and low–pass filtered queue (green). 
  (\textbf{Middle}) Step response: \emph{NOS} returns to target fastest with little ringing; queue–only is jitter–dominated; LP–queue trades jitter for lag and overshoot. 
  (\textbf{Bottom}) Burst train: \emph{NOS} clips spikes and shortens congestion episodes; queue–only is noisy; LP–queue allows taller spikes due to filter delay.}
  \label{fig:closedloop}
\end{figure*}

\section{Stochastic arrivals: shot noise model and calibration}
\label{app:nos-stochastic}

Network arrivals are bursty, short-correlated within flows, and heavy-tailed across flows. To expose this variability to the model while staying compatible with measurement practice, we represent exogenous drive at node $i$ by a compound Poisson shot-noise process smoothed at an operator-chosen time scale:
\begin{align}
\eta_i(t) &= \sum_{n=1}^{N_i(t)} A_{i,n}\,\kappa_s\!\bigl(t-t_{i,n}\bigr),
\label{eq:nos-shot}
\end{align}
where \(N_i(t)\) is a Poisson process of rate \(\nu_i\), the amplitudes \(A_{i,n}\) reflect burst size, and \(\kappa_s\) is a causal exponential smoothing kernel such as
\begin{align}
\kappa_s(t) &= e^{-t/\tau_s}\,H(t), \label{eq:nos-kernel}\\
H(t) &= 
\begin{cases}
0, & t<0,\\
1, & t\ge 0,
\end{cases} \label{eq:nos-heaviside}
\end{align}
so \(H(t)\) enforces causality (no effect before the burst) and the factor \(e^{-t/\tau_s}\) sets an exponential decay with time scale \(\tau_s\) (the burst’s influence fades smoothly). When needed, a normalised variant keeps unit area:
\begin{align}
\kappa_s^{\mathrm{norm}}(t) &= \frac{1}{\tau_s}\,e^{-t/\tau_s}\,H(t). \label{eq:nos-kernel-norm}
\end{align}
The normalisation is set so that \(A/(v_{\mathrm{th}}-v_{\mathrm{rest}})\) falls in a practical range for the workload. This construction aligns with the EWMA-style smoothing used in telemetry pipelines and permits direct calibration from packet counters.

Equation~\eqref{eq:nos-shot} mirrors the way operators pre-process packet counters and flow logs: arrivals are bucketed, lightly smoothed, and scaled. The rate $\nu_i$ matches the observed burst-start rate at node $i$ over the binning interval; the amplitude $A_{i,n}$ matches per-burst byte or packet mass after the same smoothing; and $\tau_s$ is set to the shortest window that suppresses counter jitter without hiding micro-bursts. The normalisation is chosen so that a single exogenous shot changes $v$ by a fraction $\epsilon$ of the threshold margin $m=v_{\mathrm{th}}-v_{\mathrm{rest}}$, with $\epsilon\in[10^{-3},\,2\times 10^{-1}]$. The lower bound clears telemetry noise; the upper bound prevents a single burst from forcing a reset. In discrete time this corresponds to $A\,\Delta t/m\in[10^{-3},\,2\times 10^{-1}]$; under the exponential kernel $\kappa_s(t)=e^{-t/\tau_s}H(t)$ it corresponds to $A\,\tau_s/m\in[10^{-3},\,2\times 10^{-1}]$. Using~\eqref{eq:nos-shot} inside $I_i(t)$ in~\eqref{eq:nos_start_input} makes the event drive consistent with the graph structure and delays defined in~\eqref{eq:nos-input-delayed}. In effect, traffic statistics enter \textit{NOS} with the same time base and smoothing that appear in the telemetry pipeline, which simplifies both training and post-deployment troubleshooting.
\section{Delayed gates and per-link queues}
\label{sec:delayed-gates}

When links maintain their own queues, we gate the coupling by a per-link occupancy \(q_{ij}(t)\). A simple queue-level description is
\begin{equation}
\dot q_{ij} = -\zeta q_{ij} + \Phi\bigl(S_j(t)\bigr) - \Psi\bigl(q_{ij}\bigr),
\label{eq:qij}
\end{equation}
and we replace \(w_{ij}S_j(t-\tau_{ij})\) by \(w_{ij} g\bigl(q_{ij}\bigr) S_j(t-\tau_{ij})\). This produces a two-layer representation: node integrator plus per-link capacity gating. Here \(\Phi\) maps presynaptic activity to arrivals on \((j\!\to\! i)\), \(\Psi\) represents service or RED/ECN-like bleed, and \(g\) is a bounded, monotone gate that reduces influence when the link is loaded. This separates phenomena cleanly: node excitability remains interpretable in queue units, while link congestion only modulates the strength and timing of coupling.

\paragraph{Per–link queue gating (dimensionally explicit).}
Let \(q_{ij}(t)\) denote the normalised occupancy of link \((i,j)\), with \(q_{ij}\in[0,1]\).
Choose a single time constant \(\tau_q>0\) and a bounded, dimensionless encoding
\(\sigma_s:\mathbb{R}\to[0,1]\) of the presynaptic drive \(S_j(t)\).
We model the queue as a leaky accumulator of arrivals:
\begin{equation}
\tau_q \,\dot q_{ij}(t) \;=\; -\,q_{ij}(t) \;+\; \sigma_s\!\big(S_j(t)\big),
\label{eq:qij_refined}
\end{equation}
so that \([\,\tau_q\,]=\mathrm{time}\) and \(q_{ij}\) is dimensionless. The effective coupling is then
\begin{equation}
w_{ij}^{\mathrm{eff}}(t) \;=\; w_{ij}\, g\!\big(q_{ij}(t)\big), \qquad
g:[0,1]\to[0,1],
\label{eq:weff}
\end{equation}
with gates such as
\begin{align}
g(x) &= (1-x)^p \quad \;\; \text{(capacity fade, \(p>1\))},
\label{eq:nos-g-power}\\
g(x) &= \frac{1}{1+\exp\{k(x-\theta)\}} \;\; \text{(logistic gate)}.
\label{eq:nos-g-logistic}
\end{align}
One then replaces \(w_{ij}S_j(t-\tau_{ij})\) in \eqref{eq:nos-input-delayed} by \(w_{ij}^{\mathrm{eff}}(t)S_j(t-\tau_{ij})\). This cleanly separates node adaptation from link capacity effects. In practice it lets an operator disclose where instability originates: node excitability, link queues, or topology. For completeness, the original form
\eqref{eq:qij}
is dimensionally consistent provided
\([\,\zeta\,]=\mathrm{time}^{-1}\) and
\([\,\Phi(S)\,]=[\,\Psi(q)\,]=\mathrm{time}^{-1}\);
\eqref{eq:qij_refined} corresponds to the special case
\(\zeta=\tau_q^{-1}\), \(\Phi(S)=\tau_q^{-1}\sigma_s(S)\), and \(\Psi\equiv 0\).

\paragraph{Networking interpretation and design guidance.}
Equation~\eqref{eq:nos-input-delayed} enforces hop-level causality: bursts from \(j\) affect \(i\) only after \(\tau_{ij}\), so back-pressure and wavefront effects appear on the correct timescale. The queue model in \eqref{eq:qij}–\eqref{eq:qij_refined} captures how a busy uplink or radio fades coupling without requiring packet loss. The drain term \(\Psi\) can represent ECN or token-bucket service, while the gate \(g(\cdot)\) reduces the effective link weight as occupancy rises. The power gate in \eqref{eq:nos-g-power} suits wired interfaces where capacity fades smoothly with load. The logistic gate in \eqref{eq:nos-g-logistic} matches thresholded policies such as priority drops, wireless duty cycling, or scheduler headroom rules. Because \(w_{ij}^{\mathrm{eff}}(t)\) multiplies \(S_j\!\bigl(t-\tau_{ij}\bigr)\) in \eqref{eq:nos-input-delayed}, congestion localises: only the affected links attenuate, which keeps attribution simple in operations dashboards.

Calibration follows telemetry. Choose \(\tau_{ij}\) from the queue-free component of RTT. Set \(\tau_q\) near the interface drain time or the QoS smoothing window. Normalise \(w_{ij}\) by nominal link rates or policy priorities, then scale \(W=\{w_{ij}\}\) so its spectral radius meets the desired small-signal index; under stress the gate \(g\) reduces the effective spectral radius, enlarging the stable operating region. In evaluation, this decomposition helps explain behaviour: if instability appears with empty \(q_{ij}\) the cause is node excitability or topology; if it vanishes when \(g\) engages, the bottleneck is link capacity or scheduling.

\subsection{ Delayed events and per-link gates in practice}
\label{app:delayed-gates-notes}
This note elaborates on the roles of $\tau_{ij}$ and $g\!\bigl(q_{ij}(t)\bigr)$ in \eqref{eq:nos-input-delayed} and \eqref{eq:qij_refined}.

\paragraph{Timing.}
The shift $S_j\!\bigl(t-\tau_{ij}\bigr)$ represents propagation and processing delay on $(j,i)$. It aligns presynaptic events with the time packets reach node $i$. Delay changes phase in closed loops, which affects stability margins and the onset of oscillations.

\paragraph{Conditioning.}
The gate $g\!\bigl(q_{ij}(t)\bigr)$ encodes a per-link condition such as residual capacity, scheduler priority, or queue occupancy. As $q_{ij}$ grows, the effective weight $w_{ij}^{\mathrm{eff}}(t)=w_{ij}g\!\bigl(q_{ij}(t)\bigr)$ shrinks, so a congested link contributes less drive downstream. This mimics policing, AQM, or rate limiting and prevents the model from reinforcing overloaded paths.

\paragraph{Examples.}
\begin{itemize}
\item \textit{Datacentre incast.} Multiple senders $j$ spike toward a top-of-rack switch $i$. Per-link queues $q_{ij}$ rise, $g(q_{ij})$ falls, and $w_{ij}^{\mathrm{eff}}$ shrinks, so $I_i(t)$ reflects effective rather than nominal bandwidth.
\item \textit{WAN hop with high delay.} A large $\tau_{ij}$ shifts $S_j$ by tens of milliseconds. The delayed drive prevents premature reactions at $i$ and yields a stability test that depends on both $\rho(W)$ and the phase introduced by $\tau_{ij}$.
\item \textit{Wireless fading.} Treat short-term SNR loss as a temporary rise in $q_{ij}$ or as a change in the gate parameters. The same $S_j$ then has smaller impact at $i$, consistent with link adaptation that lowers the offered load.
\end{itemize}

\paragraph{Operator workflow.}
Use passive RTT sampling or active probes to estimate $\tau_{ij}$. Fit $\tau_q$ to the decay of $q_{ij}$ when input drops or to the effective averaging window of the QoS policy. Set $w_{ij}$ proportional to nominal capacity or administrative weight, then scale $W$ to the target small-signal index. Validate $g$ by replaying traces with known congestion episodes and checking that $w_{ij}^{\mathrm{eff}}$ attenuates the expected links while leaving others unchanged.

\section{Equilibrium and stability: detailed derivations}
\label{app:stability-details}

\subsection{Equilibrium equations (derivations)}
\label{app:equilibrium-eq}

For node $i$, let $(v_i^*,u_i^*)$ denote a subthreshold equilibrium ($\dot v_i=\dot u_i=0$). At such a steady state, the \eqref{eq:nos-u-eq} implies a linear relation between recovery and queue level:
\begin{align}
u_i^* \;=\; \frac{a b}{a+\mu}\,v_i^* .
\label{eq:eq_u}
\end{align}
Substituting \eqref{eq:eq_u} into \eqref{eq:nos-v-eq} gives
\begin{align}
\begin{split}
   0 \;=\; f_{\mathrm{sat}}(v_i^*) \;+\; \bigl(\beta - \lambda - \chi\bigr)\,v_i^* \; +\; \gamma \\ \;+\; \chi v_{\mathrm{rest}} \;-\; u_i^* \;+\; I_i^* ,
\label{eq:eq_v_corrected} 
\end{split}
\end{align}

where $I_i^*$ denotes the long-run mean of the graph-local drive at node $i$ (including delayed contributions). Eliminating $u_i^*$ yields a scalar fixed-point equation for $v_i^*$:
\begin{align}
\begin{split}
f_{\mathrm{sat}}(v_i^*) \;+\; \Bigl(\beta - \lambda - \chi - \frac{a b}{a+\mu}\Bigr)\,v_i^* \;+\; \gamma 
\\ \;+ \chi v_{\mathrm{rest}} \;+\; I_i^* \;=\; 0 .
\label{eq:equilibrium_f1}
\end{split}
\end{align}

Equation \eqref{eq:equilibrium_f1} balances three effects: a convex rise with load through $f_{\mathrm{sat}}$, linear drain through the net service and damping, and the constant baseline from $\gamma+\chi v_{\mathrm{rest}}+I_i^*$. The mean input $I_i^*$ aggregates neighbour activity over the network with delays; under stationary traffic it is well defined because the delayed sum in \eqref{eq:nos-input-delayed} averages to a constant.

\paragraph{Quadratic small-signal approximation.}
For analytic visibility, replace $f_{\mathrm{sat}}(v)\approx \alpha v^2$ in a small-signal regime and define
\begin{align}
L \;:=\; \beta - \lambda - \chi - \frac{ab}{a+\mu}, 
\;\;
C \;:=\; \gamma + \chi v_{\mathrm{rest}} + I_i^* .
\label{eq:equil_LC_defs}
\end{align}
Then \eqref{eq:equilibrium_f1} reduces to
\begin{align}
\alpha (v_i^*)^2 \;+\; L\,v_i^* \;+\; C \;=\; 0 ,
\label{eq:equil_quadratic}
\end{align}
with discriminant
\begin{align}
D_i \;=\; L^2 \;-\; 4\alpha C .
\label{eq:disc}
\end{align}
If $D_i\ge 0$ and the admissible root lies in $[0,1]$, it approximates $v_i^*$ and clarifies how the operating point moves with $C$. In networks, a convenient sufficient condition for a unique operating point is that the net small-signal slope is negative on $[0,1]$, i.e.\ $f'_{\mathrm{sat}}(v)+L<0$ throughout the admissible range; the corollary in \ref{subsec:nos-equilibrium} provides a direct algebraic check. In the quadratic approximation, the sign of $D_i$ in \eqref{eq:disc} separates three regimes: $D_i<0$ (no real root under the quadratic approximation), $D_i=0$ (tangent balance at a single $v_i^*$), and $D_i>0$ (two mathematical roots, of which the physically admissible one lies in the queue range and is then checked against the monotonicity and stability conditions below).

\paragraph{Existence and uniqueness on $[0,1]$.}
The monotonicity criterion introduced earlier applies with $L$ as in \eqref{eq:equil_LC_defs}: a sufficiently \emph{negative} net slope $f'_{\mathrm{sat}}(v)+L$ over the admissible interval (for example $v\in[0,1]$) ensures a unique solution. For the saturating choice
\begin{align}
f_{\mathrm{sat}}(v) \;=\; \frac{\alpha v^2}{1+\kappa v^2}, \qquad \alpha>0,\ \kappa>0,
\label{eq:f_sat_def}
\end{align}
the slope
\begin{align}
f_{\mathrm{sat}}'(v) \;=\; \frac{2\alpha v}{\bigl(1+\kappa v^2\bigr)^2}
\label{eq:fprime}
\end{align}
is nonnegative and bounded on $[0,1]$, with $\max_{v\in[0,1]} f'_{\mathrm{sat}}(v) \le \tfrac{3\sqrt{3}}{8}\tfrac{\alpha}{\sqrt{\kappa}}$.
If
\begin{align}
\sup_{v\in[0,1]}\bigl(f'_{\mathrm{sat}}(v)+L\bigr) \;<\; 0,
\label{eq:eq_monotone_cond}
\end{align}
then $F$ is strictly decreasing on $[0,1]$ and has at most one root there. If in addition $F(0)\,F(1)\le 0$, a unique $v_i^*\in[0,1]$ exists. Condition \eqref{eq:eq_monotone_cond} is met whenever
\begin{align}
L \;<\; -\,\frac{3\sqrt{3}}{8}\,\frac{\alpha}{\sqrt{\kappa}},
\label{eq:eq_global_sufficient}
\end{align}
which is the small-signal service-dominance requirement: linear damping $\lambda+\chi+\tfrac{ab}{a+\mu}$ must dominate the maximal small-signal excitability slope (offsetting any positive linear gain $\beta$).

\paragraph{Networking interpretation.}
Equation \eqref{eq:equilibrium_f1} balances offered load against service and damping. The term $I_i^*$ aggregates graph-local arrivals and policy weights; $L$ collects all linear contributions. The sensitivity to slowly varying load follows by implicit differentiation:
\begin{align}
\frac{\partial v_i^*}{\partial I_i^*} \;=\; -\,\frac{1}{f'_{\mathrm{sat}}(v_i^*) + L}.
\label{eq:eq_dc_gain}
\end{align}
Hence the DC gain is well-defined. Under the monotonicity condition \eqref{eq:eq_monotone_cond}, the denominator is strictly negative on $[0,1]$, so $\partial v_i^*/\partial I_i^*>0$ as expected: increasing offered load increases the operating queue level. The gain shrinks as saturation strengthens $\,(\kappa\uparrow)$ or as service and damping increase $\,(\lambda,\chi,ab/(a+\mu)\uparrow)$, matching the operational goal of keeping the operating point insensitive to modest load drift.

\subsection{Derivation note: $\delta I \approx G\,\delta v$}
\label{app:dI-G}

% Master definition (edit symbols to match your paper)
Let the coupling current be $I(v)=g\,W\,f_{\mathrm{sat}}(v)$ (applied with the stated delay).
Linearising about $v^*$ gives
\begin{equation}
\delta I \;=\; g\,W\,f'_{\mathrm{sat}}(v^*)\,\delta v \;+\; \mathcal{O}(\|\delta v\|^2).
\end{equation}
Taking an induced norm yields the scalar bound
\begin{equation}
\|\delta I\| \;\le\; g\,\|W\|\,|f'_{\mathrm{sat}}(v^*)|\,\|\delta v\|
\;=:\; G\,\|\delta v\|,
\end{equation}
so we write $\delta I \approx G\,\delta v$ as shorthand, with
\begin{equation}
G \;:=\; g\,\|W\|\,|f'_{\mathrm{sat}}(v^*)|.
\end{equation}

\subsection{Jacobian and linear stability (details)}
\label{app:Jstability-local}

Linearising \textit{NOS} node \eqref{eq:nos-v-eq}–\eqref{eq:nos-u-eq} at a subthreshold equilibrium $(v^*,u^*)$ gives the single-node Jacobian
\begin{align}
J_i \;=\;
\begin{bmatrix}
f'(v^*) + \beta - \lambda - \chi & -1\\
ab & -(a+\mu)
\end{bmatrix} ,
\label{eq:Ji_app}
\end{align}
with trace and determinant
\begin{align}
T_i \;=\; \mathrm{tr}(J_i) \;=\; f'(v^*) + \beta - \lambda - \chi - (a+\mu) ,
\label{eq:trace}\\
\Delta_i = \det(J_i) \;= ab -\; \bigl(f'(v^*) + \beta - \lambda - \chi\bigr)(a+\mu) .
\label{eq:det}
\end{align}
The Routh–Hurwitz conditions for a two-dimensional system yield local asymptotic stability if and only if
\begin{align}
T_i \;<\; 0 \qquad \text{and} \qquad \Delta_i \;>\; 0 .
\label{eq:RH}
\end{align}
\paragraph{From local Jacobian to a network-level proxy.}
To connect \eqref{eq:RH} to the graph-level proxy used in the main text, consider homogeneous coupling where small drive perturbations satisfy
$\delta I \approx G\,\delta v$ with $G=gW$ and Perron eigenvalue $\rho(W)$.
Define
\begin{align}
\bar d \;:=\; f'(v^*)+\beta-\lambda-\chi .
\label{eq:bard_def_app}
\end{align}
Stacking states $(v,u)\in\mathbb{R}^{2N}$ gives the block Jacobian
\begin{align}
\mathcal{J} \;=\;
\begin{bmatrix}
\bar d\,\mathbf{I}_N + gW & -\mathbf{I}_N\\
ab\,\mathbf{I}_N & -(a+\mu)\mathbf{I}_N
\end{bmatrix}.
\label{eq:blockJ_app}
\end{align}
Along any eigenmode of $W$ with eigenvalue $\omega$, the dynamics reduce to an effective $2\times 2$ Jacobian
\begin{align}
J_{\mathrm{eff}}(\omega)
\;=\;
\begin{bmatrix}
\bar d + g\omega & -1\\
ab & -(a+\mu)
\end{bmatrix}.
\label{eq:Jeff_app}
\end{align}
The least stable mode uses $\omega=\rho(W)$, so write $k:=g\rho(W)$.
Routh--Hurwitz on \eqref{eq:Jeff_app} yields the two conditions
\begin{align}
k \;<\; (a+\mu)-\bar d,
\qquad
k \;<\; \frac{ab}{a+\mu}-\bar d .
\label{eq:RH_network_app}
\end{align}
Now define the net drain
\begin{align}
\Lambda \;:=\; \lambda+\chi+\frac{ab}{a+\mu}-\beta .
\label{eq:Lambda_app}
\end{align}
Using \eqref{eq:bard_def_app}, the determinant branch in \eqref{eq:RH_network_app} is exactly
\begin{align}
k + f'(v^*) \;<\; \Lambda,
\qquad\text{i.e.}\qquad
g\,\rho(W) \;<\; \Lambda - f'(v^*),
\label{eq:additive_proxy_app}
\end{align}
which is the additive form used for the network stability proxy.
\paragraph{Remark (normalised reporting).}
Any ratio form (for example $(k+f'(v^*))/\Lambda$) is only a normalised headroom score.
The stability condition is the additive inequality \eqref{eq:additive_proxy_app}.

\paragraph{Networking reading.}
The trace condition \eqref{eq:RH} asks that net drain $(a+\mu)+(\lambda+\chi-\beta)$ exceeds the small-signal growth from $f'(v^*)$. The determinant condition requires that the product of recovery gain $ab$ and time-scale $(a+\mu)$ dominates the same growth term. Together they state that service plus damping and the recovery loop must remove perturbations faster than excitability amplifies them. When either inequality tightens, small oscillations and slow return-to-baseline emerge, which operators observe as ringing in queue telemetry.

\paragraph{Specialisations.}
\emph{Quadratic small-signal.} For small $|v^*|$, use $f'(v^*)\approx 2\alpha v^*$ and test \eqref{eq:RH} with that slope.\\
\emph{Bounded excitability.} For $f_{\mathrm{sat}}(v)=\dfrac{\alpha v^2}{1+\kappa v^2}$,
\begin{align}
f'(v^*) \;=\; \frac{2\alpha v^*}{\bigl(1+\kappa (v^*)^2\bigr)^2} ,
\label{eq:fsatprime}
\end{align}
\begin{align}
T_i \;=\; \frac{2\alpha v^*}{(1+\kappa (v^*)^2)^2} \;+\; \beta - \lambda - \chi \;-\; (a+\mu) ,
\label{eq:Ti_sat}
\end{align}
\begin{align}
\Delta_i \;=\; ab \;-\; \Bigl(\frac{2\alpha v^*}{(1+\kappa (v^*)^2)^2} \;+\; \beta - \lambda - \chi\Bigr)(a+\mu) .
\label{eq:Delta_sat}
\end{align}
As $\kappa$ grows, the effective slope \eqref{eq:fsatprime} decreases, enlarging the stability region defined by \eqref{eq:RH}. A full statement and supporting derivations for the stability claims are given in below.

\begin{lemma}[Monotone stabilisation by saturation (fixed operating point)]
\label{lem:monotone-saturation}
Let $f_{\mathrm{sat}}(v)=\dfrac{\alpha v^2}{1+\kappa v^2}$ with $\alpha>0$ and $\kappa>0$.
For any fixed $v\ge 0$,
\begin{align}
\frac{\partial}{\partial \kappa}\,f'_{\mathrm{sat}}(v)
\;=\;
-\,\frac{4\alpha v^{3}}{\bigl(1+\kappa v^{2}\bigr)^{3}}
\;\le\; 0,
\label{eq:dfdkappa_pointwise}
\end{align}
with equality only at $v=0$.
Consequently, for the Jacobian \eqref{eq:Ji_app} evaluated at a fixed operating point $v^*>0$,
increasing $\kappa$ strictly decreases the trace $T_i$ and strictly increases the determinant $\Delta_i$,
thereby enlarging the region $\{T_i<0,\ \Delta_i>0\}$ when re-evaluated at that same $v^*$.
\end{lemma}

\begin{proof}[Proof sketch]
Differentiate $f'_{\mathrm{sat}}(v)=\dfrac{2\alpha v}{(1+\kappa v^2)^2}$ with respect to $\kappa$ to obtain
\eqref{eq:dfdkappa_pointwise}. For fixed $v^*$, the $\kappa$-dependence of \eqref{eq:Ji_app} enters only
through $f'_{\mathrm{sat}}(v^*)$, hence
\begin{align}
\begin{split}
\frac{\partial T_i}{\partial \kappa}
\;=\;
\frac{\partial f'_{\mathrm{sat}}(v^*)}{\partial \kappa}
\;<\;0 \quad (v^*>0),
\\
\frac{\partial \Delta_i}{\partial \kappa}
\;=\;
-(a+\mu)\frac{\partial f'_{\mathrm{sat}}(v^*)}{\partial \kappa}
\;>\;0,
\label{eq:dTDelta_dkappa}
\end{split}
\end{align}
since $a+\mu>0$ and \eqref{eq:dfdkappa_pointwise} is strict for $v^*>0$.
\end{proof}

\paragraph{Remark (equilibrium shift with $\kappa$).}
If the operating point is taken as the equilibrium branch $v^*=v^*(\kappa)$ rather than held fixed,
then the relevant quantities use total derivatives. In particular,
\begin{align}
\frac{d}{d\kappa}\, f'_{\mathrm{sat}}(v^*(\kappa),\kappa)
\;=\;
\frac{\partial f'_{\mathrm{sat}}}{\partial \kappa}(v^*,\kappa)
\;+\;
f''_{\mathrm{sat}}(v^*,\kappa)\,\frac{dv^*}{d\kappa},
\label{eq:total_deriv_fprime}
\end{align}
so $\frac{dT_i}{d\kappa}=\frac{d}{d\kappa}f'_{\mathrm{sat}}(v^*)$ and
$\frac{d\Delta_i}{d\kappa}=-(a+\mu)\frac{d}{d\kappa}f'_{\mathrm{sat}}(v^*)$.

Moreover, $dv^*/d\kappa$ follows from the equilibrium relation (Eq.~\eqref{eq:eq_dc_gain}).
Writing $F(v,\kappa):=f_{\mathrm{sat}}(v,\kappa)+L v + C=0$ with
$L:=\beta-\lambda-\chi-\frac{ab}{a+\mu}$, the implicit-function theorem gives
\begin{align}
\frac{dv^*}{d\kappa}
\;=\;
-\,\frac{\partial_\kappa F(v^*,\kappa)}{\partial_v F(v^*,\kappa)}
\;=\;
-\,\frac{\partial_\kappa f_{\mathrm{sat}}(v^*,\kappa)}{f'_{\mathrm{sat}}(v^*,\kappa)+L},
\label{eq:dvdkappa_IFT}
\end{align}
where $\partial_\kappa f_{\mathrm{sat}}(v,\kappa)= -\dfrac{\alpha v^4}{(1+\kappa v^2)^2}<0$ for $v>0$.
Thus the pointwise monotonicity \eqref{eq:dfdkappa_pointwise} holds unconditionally at fixed $v^*$, while
monotonic behaviour along $v^*(\kappa)$ depends on the additional term in \eqref{eq:total_deriv_fprime}.
In the paper we use \eqref{eq:dfdkappa_pointwise} as an engineering fact: for a given observed subthreshold
$v^*$, larger $\kappa$ reduces the small-signal slope and widens the Routh--Hurwitz margins when re-evaluated
at that operating point.

\begin{corollary}[Network threshold increases with saturation (fixed operating point)]
\label{cor:net-threshold-sat}
Under homogeneous coupling $G=gW$ with Perron eigenvalue $\rho(W)$, define
\begin{align}
\begin{split}
\bar d \;:=\; f'_{\mathrm{sat}}(v^*)+\beta-\lambda-\chi,
\\
k^\star \;=\; \min\Bigl\{(a+\mu)-\bar d,\ \frac{ab}{a+\mu}-\bar d\Bigr\},
\qquad
g_\star \;\approx\; \frac{k^\star}{\rho(W)} .
\end{split}
\label{eq:kstar_gstar_defs}
\end{align}
Equivalently, with $\Lambda$ from \eqref{eq:Lambda_app}, the determinant branch is
$k < \Lambda - f'_{\mathrm{sat}}(v^*)$, matching \eqref{eq:additive_proxy_app}.
Holding $v^*$ fixed, \eqref{eq:dfdkappa_pointwise} implies $\partial \bar d/\partial\kappa<0$, hence
$\partial k^\star/\partial\kappa>0$ on the active branch, and therefore $\partial g_\star/\partial\kappa>0$.
\end{corollary}

\paragraph{Operational guidance.}
When tuning for a given site, estimate $v^*$ from the modal subthreshold queue level and compute the small-signal slope $f'(v^*)$ using the fitted $(\alpha,\kappa)$. If \eqref{eq:RH} is tight, increasing $\kappa$ (earlier saturation) or increasing $\lambda$ and $\chi$ adds headroom. If tightening occurs only in $\Delta_i$, adjust the recovery loop $(a,b,\mu)$ to increase $ab$ or the composite time-scale $(a+\mu)$; both choices shorten the memory of recent stress and prevent ringing after bursts.
Decay times and DC sensitivity that guide window length and reset tuning are summarised in below Eqs.~\eqref{eq:tau_lin} and \eqref{eq:dc_gain_repeat}.
\paragraph{Eigenvalues and decay times.}
Let $\lambda_{1,2}$ be the eigenvalues of the local Jacobian $J_i$ defined in \ref{app:Jstability-local}. When the Routh--Hurwitz conditions in \eqref{eq:RH} hold and $T_i^2>4\Delta_i$, the node is an overdamped node with two real negative modes. When $T_i^2<4\Delta_i$, it is underdamped with complex conjugate modes whose common decay rate is $-\tfrac{1}{2}T_i$. The dominant linear return time is
\begin{align}
\tau_{\text{lin}} \;\approx\; \frac{1}{\,\min\{-\Re\lambda_1,\,-\Re\lambda_2\}\,},
\label{eq:tau_lin}
\end{align}
which provides an engineering handle for truncated BPTT windows and for selecting the reset constant $\rho$ so numerical timescales match device drain and scheduler epochs. In practice, $\tau_{\text{lin}}$ estimated from small perturbations of $I_i$ around an operating point aligns well with the observed relaxation of the queue proxy $v_i$.

\paragraph{Input--output small-signal gain.}
Let $\widehat{I}$ be a small constant perturbation of $I_i$ and $\widehat{v}$ the induced steady-state change in $v_i$. Using the equilibrium implicit-function relation from \eqref{eq:eq_dc_gain}, the DC gain is
\begin{align}
\frac{\widehat{v}}{\widehat{I}} \;=\; -\,\frac{1}{\,f'_{\mathrm{sat}}(v^*)+L\,}.
\label{eq:dc_gain_repeat}
\end{align}
Larger service $\lambda$ and stronger damping $\chi$ decrease this gain, making the operating point less sensitive to slow offered-load drift. This matches queueing intuition and yields a direct calibration target: estimate $\widehat{v}/\widehat{I}$ from step tests or natural experiments in telemetry, then verify that the inferred $L$ is consistent with the Jacobian conditions in \eqref{eq:RH}.

\subsection{Delay-aware Perron-mode stability (DDE extension)}
\label{app:delay_stability}

For homogeneous coupling with a representative delay $\tau$ on the dominant mode,
the linearised Perron-mode dynamics become a $2\times 2$ delay system:
\begin{align}
\dot{\delta v}(t) &= \bar d\,\delta v(t) - \delta u(t) + k\,\delta v(t-\tau),\\
\dot{\delta u}(t) &= ab\,\delta v(t) - (a+\mu)\,\delta u(t),
\end{align}
where $\bar d=f'(v^*)+\beta-\lambda-\chi$ and $k=g\rho(W)$.
The characteristic equation is
\begin{equation}
\bigl(s-\bar d-k e^{-s\tau}\bigr)\bigl(s+(a+\mu)\bigr)+ab=0.
\label{eq:char_dde}
\end{equation}
Setting $s=i\Omega$ gives the boundary of stability in the standard frequency-domain form.
Define
\[
N(\Omega)=\bigl(i\Omega+(a+\mu)\bigr)\bigl(i\Omega-\bar d\bigr)+ab.
\]
Then \eqref{eq:char_dde} implies
\begin{equation}
k\,e^{-i\Omega\tau}=\frac{N(\Omega)}{i\Omega+(a+\mu)}.
\label{eq:k_phase}
\end{equation}
Hence the critical coupling at frequency $\Omega$ is
\begin{equation}
k(\Omega)=\frac{|N(\Omega)|}{\sqrt{\Omega^2+(a+\mu)^2}},
\label{eq:k_mag}
\end{equation}
and admissible delays satisfy the phase relation
\begin{equation}
\tau=\frac{1}{\Omega}\Bigl(\arg(i\Omega+(a+\mu))-\arg(N(\Omega))+2\pi m\Bigr),\quad m\in\mathbb{Z}.
\label{eq:tau_phase}
\end{equation}
For a given $\tau$, one finds $\Omega$ satisfying \eqref{eq:tau_phase} (for some integer $m$) and evaluates $k_{\mathrm{crit}}(\tau)=\min k(\Omega)$ over the feasible crossings.
When $\tau=0$, the $\Omega\to 0$ limit of \eqref{eq:k_mag} recovers the undelayed determinant branch
$k<\frac{ab}{a+\mu}-\bar d$, which is equivalent to
$k+f'(v^*)<\Lambda$ with $\Lambda=\lambda+\chi+\frac{ab}{a+\mu}-\beta$.
Empirically, increasing $\tau$ shifts $k_{\mathrm{crit}}(\tau)$ downward, matching the boundary shift observed in Fig.~\ref{fig:NOS_panels}(a).

\subsection{Network coupling and global stability (details)}
\label{app:global_stability}

In isolation a \textit{NOS} unit is a load–service balance set by its local parameters. In a network the steady input is shaped by neighbours and policy weights, so existence and robustness of equilibria depend on both the operating point and the coupling matrix. Let $W=[w_{ij}]$ be the graph coupling and let $S_j^*$ denote the effective steady presynaptic drive at node $j$ (after any per–link delays and gates). The steady input at node $i$ is
\begin{align}
\begin{split}
I_i^* \;=\; \sum_{j} w_{ij}\,S_j^*, 
\\ 
\max_i I_i^* \;\le\; \|W\|_{\infty}\,\|S^*\|_{\infty},
\\
\|I_i^*\|_2\le\;\|W\|_2\,\|S^*\|_2 \ \ 
\label{eq:nos-Istar-bounds}
\end{split}
\end{align}
with $\|W\|_2=\sqrt{\rho(W^\top W)}$, and the spectral radius obeys $\rho(W)\le \|W\|_{\infty}$, so $\|W\|_{\infty}$ gives a conservative topology-aware cap on steady drive.
From the scalar equilibrium reduction (cf. \eqref{eq:eq_u}–\eqref{eq:disc}), define
\begin{align}
L \;:=\; \beta - \lambda - \chi - \frac{ab}{a+\mu}, 
\qquad
C \;:=\; \gamma + \chi v_{\mathrm{rest}} + I_i^*,
\label{eq:nos-LC}
\end{align}
and write the per-node equilibrium condition as
\begin{align}
f_{\mathrm{sat}}(v_i^*) + L\,v_i^* + C \;=\; 0.
\label{eq:nos-eq-scalar}
\end{align}

\paragraph{Existence bounds.}
Two complementary sufficient conditions are useful in operations.

\emph{Saturated cap (conservative).} Since $f_{\mathrm{sat}}(v)\le \alpha/\kappa$ for all $v$,
\begin{align}
C \;\le\; \frac{\alpha}{\kappa}
\;\rightarrow\;
\text{a nonnegative equilibrium exists when } L\le 0.
\label{eq:sat_cap}
\end{align}
This bound uses only the saturation ceiling and is independent of the local operating point.

\emph{Quadratic small-signal.} Approximating $f_{\mathrm{sat}}(v)\approx \alpha v^2$ near light load $v=0$ gives
\begin{align}
\alpha (v_i^*)^2 + L v_i^* + C \;=\; 0,
\qquad 
D_i \;=\; L^2 - 4\alpha C,
\label{eq:nos-quad-eq}
\end{align}
so an equilibrium exists if
\begin{align}
C \;\le\; \frac{L^2}{4\alpha}.
\label{eq:quad_small_signal}
\end{align}
We refer to \eqref{eq:quad_small_signal} as the \emph{quadratic small-signal} bound. Combining \eqref{eq:quad_small_signal} with the norm bound in \eqref{eq:nos-Istar-bounds} yields
\begin{align}
\gamma + \chi v_{\mathrm{rest}} + \|W\|_{\infty}\,\|S^*\|_{\infty}
\;\le\; \frac{L^2}{4\alpha},
\label{eq:sufic}
\end{align}
with a spectral alternative obtained by replacing $\|W\|_{\infty}\|S^*\|_{\infty}$ with $\|W\|_2\|S^*\|_2$.
Inequality \eqref{eq:sufic} makes the graph explicit: heavier row sums (dense fan–in or large policy weights) shrink headroom for the same steady $S^*$.

\paragraph{Operational margin.}
A convenient scalar indicator is
\begin{align}
\Delta_{\mathrm{op}} \;:=\; \frac{L^2}{4\alpha} - \bigl(\gamma + \max_i I_i^*\bigr),
\label{eq:nos-op-margin}
\end{align}
which is positive when the quadratic small-signal existence test passes with slack. In practice, keeping $\Delta_{\mathrm{op}}>0$ under nominal load provides headroom for burstiness and guides weight normalisation: scale $W$ so that the largest row sum keeps $\max_i I_i^*$ below the right-hand side of \eqref{eq:sufic}. When delays $\tau_{ij}$ are present, \eqref{eq:sufic} remains a conservative existence test; sharper delay-aware guarantees require the block Jacobian in the next subsection.

\paragraph{Operational margin under load and damping.}
The operational margin \eqref{eq:nos-op-margin} can be plotted as a function of the subthreshold damping $\chi$ and the maximum steady input $I_{\max}=\max_i I_i^*$, as in Fig.~\ref{fig:op_heat}. The red contour $\Delta_{\mathrm{op}}=0$ separates admissible operating points from those that violate the quadratic small-signal bound. From a networking viewpoint, $\chi$ measures the strength of subthreshold queue relaxation, while $I_{\max}$ is the heaviest sustained offered load after graph coupling. Two effects are visible: (i) the margin decreases monotonically with $I_{\max}$, so higher offered load erodes headroom; (ii) larger $\chi$ shifts the boundary outward, permitting higher sustained load before crossing $\Delta_{\mathrm{op}}=0$. Stronger damping therefore improves resilience but does not remove the inherent load–margin trade-off. $\|W\|_{\infty}$ aggregates worst–case inbound influence at a node, so it reflects heavy in–degree or imbalanced policies; $\|W\|_2$ captures coherent multi–hop reinforcement. Both provide simple levers: reweight heavy rows, sparsify fan–in, or gate selected edges until $\Delta_{\mathrm{op}}$ is positive with a chosen margin.

\paragraph*{Illustrative computation.}
Using the legacy “starter” values (kept here for continuity) 
$\alpha=0.02$, $\beta=0.5$, $\gamma=0.05$, $\lambda=0.2$, 
$a=0.05$, $b=0.5$, $\mu=0.01$, $\chi=0.05$, and $I_{\max}=0.10$,
\(
\frac{ab}{a+\mu}=\frac{0.025}{0.06}\approx 0.4167,\)
\(
L=0.5-0.2-0.05-0.4167\approx -0.1667,
\)
\(
\frac{L^2}{4\alpha}\approx \frac{0.02778}{0.08}=0.34722,
\quad 
\Delta_{\mathrm{op}} \approx 0.34722-(0.05+0.10)=0.19722>0.
\)
At $I_{\max}=0.30$ the margin is $\Delta_{\mathrm{op}}\approx 0$, and raising $\chi$ from $0.05$ to $0.30$ changes $L$ from $-0.1667$ to $-0.4167$, yielding $\Delta_{\mathrm{op}}\approx 1.82$—a substantial safety increase. 
\emph{Note:} these numbers are purely illustrative for the quadratic surrogate; in the main experiments we use the admissible ranges in Table~\ref{tab:nos_ranges}. The linear dependence on $I_{\max}$ and the monotone effect of $\chi$ hold in either case.

\begin{figure}[H]
\centering
\includegraphics[width=\columnwidth]{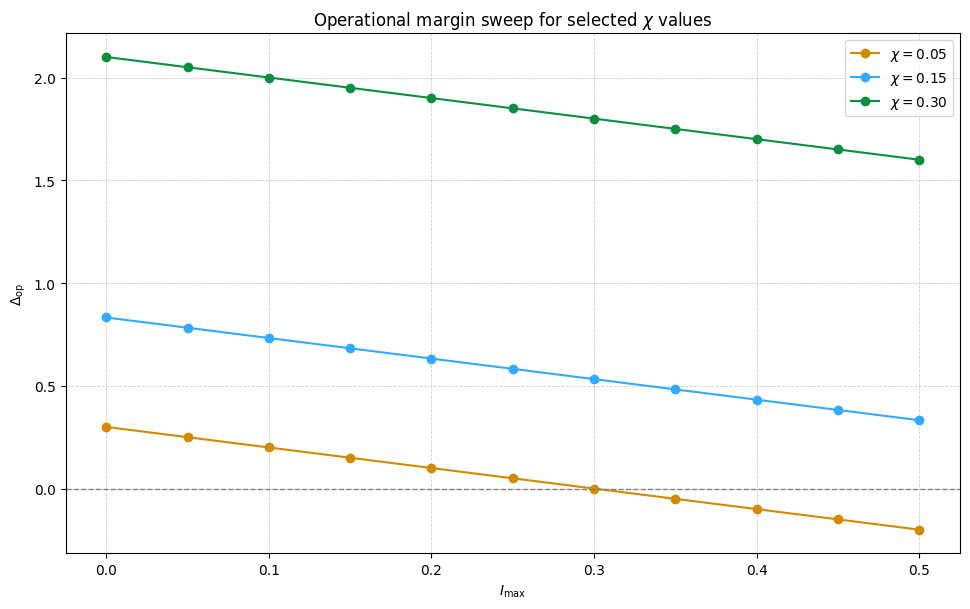}
\caption{One-dimensional sweeps of the operational margin $\Delta_{\mathrm{op}}$ against $I_{\max}$ for selected $\chi$. Each trace is affine with slope $-1$, illustrating the linear decay of margin with load. Intercepts grow with $\chi$, consistent with the global stability boundary in Fig.~\ref{fig:op_heat}.}
\label{fig:op_margin}
\end{figure}

\paragraph*{Supporting sweeps.}
Figure~\ref{fig:op_margin} shows $\Delta_{\mathrm{op}}$ against $I_{\max}$ for several $\chi$. The slope is $-1$ by construction; increasing $\chi$ shifts the intercept upward (analogous to higher service slack). For the legacy starter set above the margin crosses zero near $I_{\max}\approx 0.30$, consistent with
\(
\Delta_{\mathrm{op}} = \frac{L^2}{4\alpha} - (\gamma + I_{\max}),\qquad \text{$L$ fixed.}
\)
Increasing $\chi$ increases $|L|$ and hence the scalar margin.

\subsection{Block Jacobian, network spectrum, and stability threshold (details)}
\label{app:stability-threshold}

We assess global small–signal stability by linearising the coupled \textit{NOS} network about a subthreshold equilibrium $(v^*,u^*)$. Let $\delta v$ and $\delta u$ be perturbations of state, and let small variations of presynaptic drive satisfy $\delta I \approx G\,\delta v$, where $G$ is the input–output sensitivity induced by topology (for homogeneous scaling, $G=gW$). Linearisation of \eqref{eq:nos-v-eq}–\eqref{eq:nos-u-eq} then yields the $2N\times 2N$ block Jacobian
\begin{align}
\begin{split}
\mathcal{J} \;=\;
\begin{bmatrix}
D + G & -\mathbf{I}_N \\
a B & -(a+\mu)\mathbf{I}_N
\end{bmatrix},
\\
D=\operatorname{diag}\!\Bigl(f'(v_i^*)+\beta-\lambda-\chi\Bigr),\quad
B=b \mathbf{I}_N.
\label{eq:blockJ}
\end{split}
\end{align}
Instability occurs when any eigenvalue of $\mathcal{J}$ crosses into the open right half-plane.

\paragraph{Perron-mode reduction under homogeneous coupling.}
For homogeneous coupling $G=gW$ with $g\ge 0$ and let $\rho(W)$ be the spectral radius with Perron eigenpair $(\rho(W),\phi)$. Projecting \eqref{eq:blockJ} onto $\phi$ yields an effective $2\times 2$ Jacobian
\begin{align}
\begin{split}
J_{\mathrm{eff}}(k) \;=\;
\begin{bmatrix}
\bar d + k & -1\\
ab & -(a+\mu)
\end{bmatrix},
\\
k \;:=\; g\,\rho(W), 
\\
\bar d \;:=\; \frac{1}{\|\phi\|_2^2}\sum_i \phi_i^2\bigl[f'(v_i^*)+\beta-\lambda-\chi\bigr].
\label{eq:Jeff}
\end{split}
\end{align}
In the Perron–mode reduction we define \(\bar d := \frac{1}{\|\phi\|_2^2}\sum_i \phi_i^2\bigl[f'(v_i^*)+\beta-\lambda-\chi\bigr]\), which is exact for heterogeneous equilibria. Throughout the subsequent analysis we write \(\bar d \approx f'_{\mathrm{sat}}(v^*)+\beta-\lambda-\chi\) for homogeneous or mildly heterogeneous whenever \(v_i^*\approx v^*\) across nodes (or the Perron vector is sufficiently delocalised). All threshold formulas (e.g. \eqref{eq:kstar-gstar} and \eqref{eq:collapse_gstar}) remain valid with either form by substituting the appropriate \(\bar d\) for the deployment at hand.
The Routh–Hurwitz conditions for $J_{\mathrm{eff}}(k)$ give the critical scalar coupling
\begin{align}
k^\star
\;=\;
\min\Bigl\{(a+\mu)-\bar d,\ \frac{ab}{a+\mu}-\bar d\Bigr\},
\qquad
g_\star \;\approx\; \frac{k^\star}{\rho(W)}.
\label{eq:kstar-gstar}
\end{align}

Using $\Lambda := \lambda + \chi + \frac{ab}{a+\mu} - \beta$ (cf.~\eqref{eq:nos-reading-lambda}) and $\bar d \approx f'_{\mathrm{sat}}(v^*)+\beta-\lambda-\chi$, the determinant branch of \eqref{eq:kstar-gstar} gives
\begin{align}
g\,\rho(W) \;<\; \Lambda - f'_{\mathrm{sat}}(v^*) ,
\label{eq:net-small-signal}
\end{align}
equivalently $g\,\rho(W)+f'_{\mathrm{sat}}(v^*)<\Lambda$, which matches the stability proxy introduced in the scaling discussion. While the Perron-mode reduction yields tight thresholds in practice, conservative certificates can be obtained via Gershgorin discs and norm bounds; see \emph{methods}~\ref{sec:gershgorin} for the resulting inequalities and their comparison to measured $g_\star$ in Table~\ref{tab:rho_gstar}.

\begin{figure}[H]
\centering
\includegraphics[width=\columnwidth]{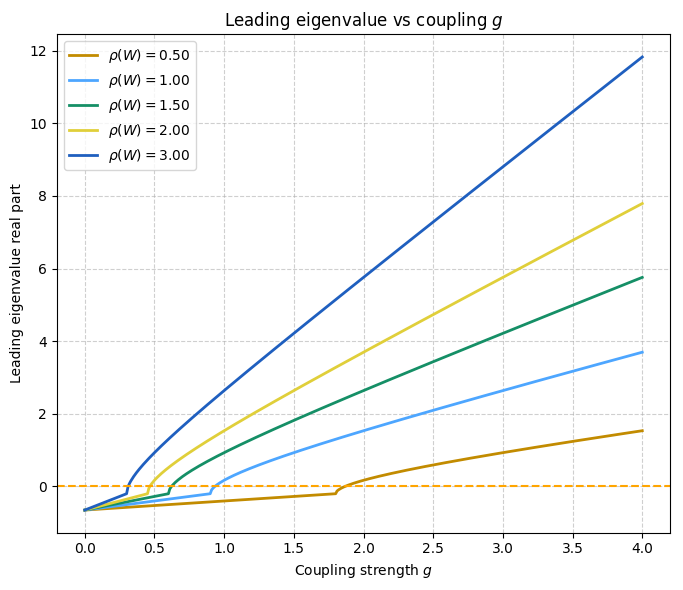}\hfill
\includegraphics[width=.\columnwidth]{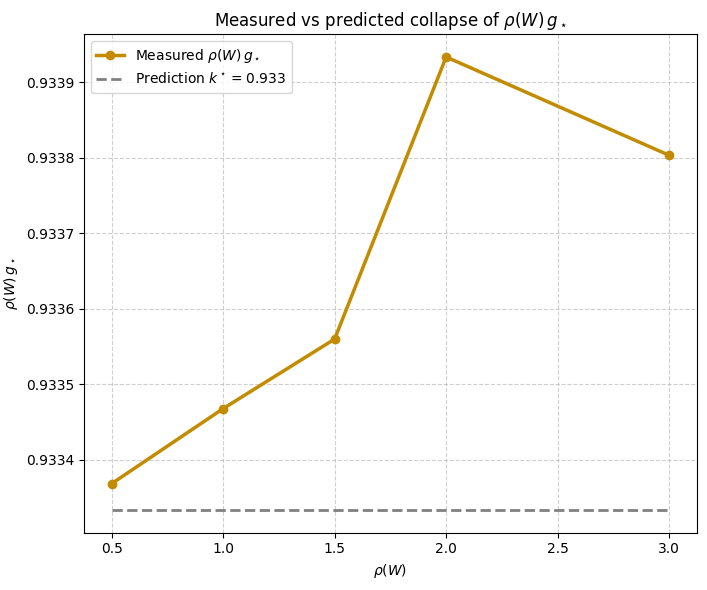}
\caption{Spectral verification for homogeneous coupling $G=gW$.
Left: real part of the leading eigenvalue of the block Jacobian versus $g$ for several $\rho(W)$; the zero crossing defines $g_\star$.
Right: collapse of $\rho(W)\,g_\star$ to the scalar threshold $k^\star$ predicted by the Perron–mode reduction in \eqref{eq:kstar-gstar}.
The small spread reflects finite size and random matrix variability.}
\label{fig:spec_collapse}
\end{figure}

\paragraph{Networking interpretation and control.}
Figure~\ref{fig:spec_collapse} shows that the first loss of stability is governed by the Perron mode of $W$. As coupling $g$ increases, the leading eigenvalue crosses the imaginary axis at $g_\star$, marking the onset of coherent queue growth along the dominant influence path. The collapse of $\rho(W)\,g_\star$ onto $k^\star$ confirms the separation of concerns in \eqref{eq:kstar-gstar}: $\rho(W)$ captures how strongly topology reinforces load; $k^\star$ depends only on node physics through $(\bar d,a,b,\mu)$. Operationally, the levers are direct. Reduce $\rho(W)$ by reweighting heavy fan–in or pruning selected edges. Reduce $g$ through gain scheduling on stressed regions or policy throttles that weaken $I\mapsto v$ sensitivity. Raise the drain margin $\Lambda$ by increasing service $\lambda$, adding subthreshold damping $\chi$, or shortening recovery via $a$ (with $b,\mu$ setting the trade). Because $f'_{\mathrm{sat}}(v^*)$ decreases as saturation strengthens (larger $\kappa$), bounded excitability expands the stable range in \eqref{eq:net-small-signal}.

\paragraph{Finite-size and heterogeneity effects.}
The rescaled thresholds are not perfectly flat. They sit slightly above $k^\star$ across $\rho(W)$, with larger deviations in smaller or more heterogeneous graphs. Two mechanisms explain this percent-level bias. First, nearby non-Perron modes and row–sum variability shift the leading root when the Perron vector localises on hubs, a common feature of heavy-tailed topologies. Second, discrete sweeps of $\Re\lambda_{\max}(g)$ and mild non-normal amplification nudge the numerical crossing. These are finite-size corrections, not a change of scaling, and they diminish with larger $N$ or more homogeneous weights. For conservative guarantees during early deployment, replace $\rho(W)$ by norm bounds or apply Gershgorin discs to $\mathcal{J}$ to obtain stricter, topology-aware thresholds. For deployments requiring analytic guarantees, \emph{methods}~\ref{sec:gershgorin} derives Gershgorin and $\|W\|_\infty$-based bounds and reports their conservatism relative to measured $g^\star$ (Table~\ref{tab:rho_gstar}). A detailed discussion on the bifurcation structure for \textit{NOS}, and its operational interpretation, is given in \ref{sec:bifurcationOP}.

\subsection{Conservative stability via Gershgorin and norm bounds}
\label{sec:gershgorin}

Let $G=gW$ with $W_{ii}=0$ and $\rho(W)$ the spectral radius after normalisation (we use $\rho(W)=1$ in experiments). Linearising \eqref{eq:nos-v-eq}–\eqref{eq:nos-u-eq} at $(v^*,u^*)$ gives a $2N\times 2N$ block Jacobian $\mathcal{J}$. For the top block rows $1{:}N$,
\(
\bar d_i := f'_{\mathrm{sat}}(v_i^*) + \beta - \lambda - \chi,
\)
is the diagonal entry, and the off-diagonals are $g W_{ij}$ for $j\neq i$, plus the coupling to $u_i$ with derivative $-1$. Thus the Gershgorin centre and radius are
\(
c_i=\bar d_i,\qquad R_i = 1 + |g| \sum_{j\neq i} |W_{ij}|.
\)
A sufficient condition for those discs to lie strictly in the left half-plane is
\(
\Re(c_i)+R_i<0\quad\text{for all }i,
\)
which yields
\(
|g| \;<\; \min_i \frac{-\Re(\bar d_i)-1}{\sum_{j\neq i}|W_{ij}|}.
\)
For the bottom block rows $N{+}1{:}2N$, the centre is $-(a+\mu)$ and the radius is $|ab|$, so a sufficient condition is 
\(
-(a+\mu)+|ab|<0.
\)
Using $\sum_{j\neq i}|W_{ij}|\le\|W\|_\infty$ gives the norm variant
\(
|g| \;<\; \frac{-\min_i \Re(\bar d_i) - 1}{\|W\|_{\infty}},
\qquad\text{and}\qquad |ab|<a+\mu.
\)
These conditions are sufficient (not necessary) and are conservative relative to observed thresholds.

\begin{table*}[htbp]
\centering
\caption{Instability thresholds versus network spectral radius for homogeneous coupling $G=gW$. 
The measured $g_\star$ is the first zero crossing of the leading eigenvalue real part. 
The product $\rho(W)\,g_\star$ agrees with the Perron-mode prediction $k^\star$. 
Gershgorin bounds are conservative; negative values mean the bound cannot certify stability for any $g\ge 0$. 
The heuristic $g_{\mathrm{heur}}=k^\star/\|W\|_\infty$ matches the Perron scaling but lacks a guarantee. 
The last column checks $|ab|<(a+\mu)$.}
\label{tab:rho_gstar}
\begin{tabular}{cccccccc}
\toprule
$\rho(W)$ & $g_\star$ (meas.) & $\rho(W) g_\star$ & $g_{\text{bound}}$ (Gersh.) & $g_{\text{heur}}$ & $k^\star$ (scalar) & $(\rho g_\star)/k^\star$ & $|ab|<(a+\mu)$ \\
\midrule
0.5 & 1.881799 & 0.940900 & $-0.526097$ & 0.554642 & 0.940878 & 1.000023 & True \\
1.0 & 0.940924 & 0.940924 & $-0.263049$ & 0.277321 & 0.940878 & 1.000049 & True \\
1.5 & 0.627452 & 0.941177 & $-0.175366$ & 0.184881 & 0.940878 & 1.000319 & True \\
2.0 & 0.470484 & 0.940969 & $-0.131524$ & 0.138660 & 0.940878 & 1.000097 & True \\
3.0 & 0.314042 & 0.942125 & $-0.087683$ & 0.092440 & 0.940878 & 1.001326 & True \\
\bottomrule
\end{tabular}
\end{table*}

Table~\ref{tab:rho_gstar} complements the figures: the measured thresholds obey the Perron-mode scaling ($\rho(W)g_\star\approx k^\star$), while Gershgorin discs provide only conservative certificates. From a networking standpoint, the stability boundary is captured by interpretable quantities: the effective small-signal slope (through $\bar d_i$), service and recovery $(\lambda,a,\mu,b)$ setting $k^\star$, and the graph’s spectral radius encoding topology.
This bridges low-level queue dynamics with high-level network design, showing how structural and service parameters jointly control the emergence of congestion instabilities.

\section{Bifurcation Structure, Operational Interpretation, and Finite-Size Behaviour}
\label{sec:bifurcationOP}
Loss of stability as the mean input $I$ increases (or as effective coupling $k=g\,\rho(W)$ grows) is produced by two local mechanisms. They can be diagnosed either from the scalar equilibrium map $F(v)=f_{\mathrm{sat}}(v)+Lv+C(I)$ (fold tests) or from the $2\times2$ Jacobian at the operating point (Routh–Hurwitz tests).

\begin{enumerate}
  \item \textbf{Saddle--node (SN).} The stable and unstable equilibria coalesce and disappear when the discriminant in \eqref{eq:disc} satisfies $D_i=0$. Beyond this point there is no steady operating level: $v$ drifts upward until resets and leakage dominate. In network terms, offered load or coherent reinforcement pushes a node beyond its admissible queue set; the model predicts sustained congestion with intermittent resets rather than recovery to a fixed level.
  \item \textbf{Hopf.} The Jacobian trace crosses zero while the determinant remains positive ($T_i=0$, $\Delta_i>0$), creating a small oscillatory mode. In network terms, queues enter a self-excited cycle whose phase can entrain neighbors through $W$, producing rolling congestion waves that propagate along high-gain paths.
\end{enumerate}

\paragraph{Parameter trends and admissibility.}
Continuation of equilibria $v^*$ versus $I$ reproduces queueing intuition and clarifies which onset is available. Increasing the service $\lambda$ or the subthreshold damping $\chi$ shifts both SN and Hopf onsets to larger $I$ (more headroom); increasing excitability $\alpha$ shifts them to smaller $I$ (less headroom). The recovery coupling $b$ controls whether a Hopf branch is admissible: the linear feedback $ab$ must exceed the decay $(a+\mu)^2$ for an oscillatory onset to exist; otherwise only an SN occurs. Figure~\ref{fig:VI_bif} shows $v^*(I)$ under sweeps of $\lambda$, $\alpha$, and $b$; marker coordinates are listed in Table~\ref{tab:continuation-markers}. Near-coincident markers correspond to the codimension-two neighborhood $ab\approx(a+\mu)^2$, where a small change in recovery or damping switches the dominant failure mode from smooth tipping (SN) to oscillatory bursts (Hopf).

\noindent\textit{Empirical collapse.}
Across graphs, the measured threshold obeys
\begin{align}
\rho(W)\,g_\star \;\approx\; 
\min\!\Bigl\{(a+\mu)-\bar d,\ \frac{ab}{a+\mu}-\bar d\Bigr\},
\label{eq:collapse_gstar}
\end{align}
with small upward deviations when the Perron vector localises or row-sums vary; cf. Fig.~\ref{fig:spec_collapse}.

\paragraph{Networking interpretation.}
The two onsets have distinct operational signatures. An SN route produces a clean loss of a steady queue level and sustained saturation. Telemetry shows a monotone climb in $v$ with resets, increased delay variance, and persistent ECN/marking without a clear period. A Hopf route produces narrowband oscillations: queues and drop/mark rates cycle with a characteristic period $\sim 2\pi/\sqrt{\Delta_i-(T_i/2)^2}$ at onset, and neighboring nodes oscillate in phase along the Perron mode of $W$. These patterns matter for mitigation. To extend stable throughput, raise $\lambda$ or $\chi$, or reduce $\alpha$ through stronger saturation (larger $\kappa$). To avoid self-sustained burstiness, keep $ab$ comfortably larger than $(a+\mu)^2$ by shortening recovery time ($a$), reducing recovery sensitivity ($b$) when needed, or adding passive decay ($\mu$). Each lever maps to standard controls: scheduler and line-rate settings ($\lambda$), AQM/leak policies ($\chi$), congestion ramp shaping ($\alpha,\kappa$), and backoff or pacing aggressiveness ($a,b,\mu$).

\paragraph{Finite-size effects and coherence.}
In finite graphs the transition is slightly rounded but sharpens with size. The measured coupling threshold aligns with the Perron-mode prediction $k^\star$ (cf. \eqref{eq:kstar-gstar}), with small upward deviations when the Perron vector localises on hubs or when secondary modes lie close in the spectrum. These are percent-level corrections that diminish in larger or more homogeneous topologies. Practically, one can monitor the rescaled margin $\rho(W)\,g - k^\star$: values near zero predict imminent entrainment of queues and the appearance of coherent oscillations along the dominant influence path. The \emph{methods}~\ref{sec:synchrony} quantifies network coherence via an order parameter $R(t)$, confirming that the onset clusters near $k^\star$ and sharpens with $N$ (Fig.~\ref{fig:R_vs_k}, Table~\ref{tab:R_vs_k}).

\begin{corollary}[Spectral collapse of the coupling threshold]
For homogeneous coupling $G=gW$, the instability threshold extracted from the block
Jacobian satisfies \eqref{eq:collapse_gstar},
where $\bar d=f'_{\mathrm{sat}}(v^*)+\beta-\lambda-\chi$. 
Percent-level upward deviations arise in finite or heterogeneous graphs when the Perron vector localises 
or row-sum variability is high, consistent with Fig.~\ref{fig:spec_collapse}.
\end{corollary}

\begin{figure*}[h]
\centering
\includegraphics[width=0.95\textwidth]{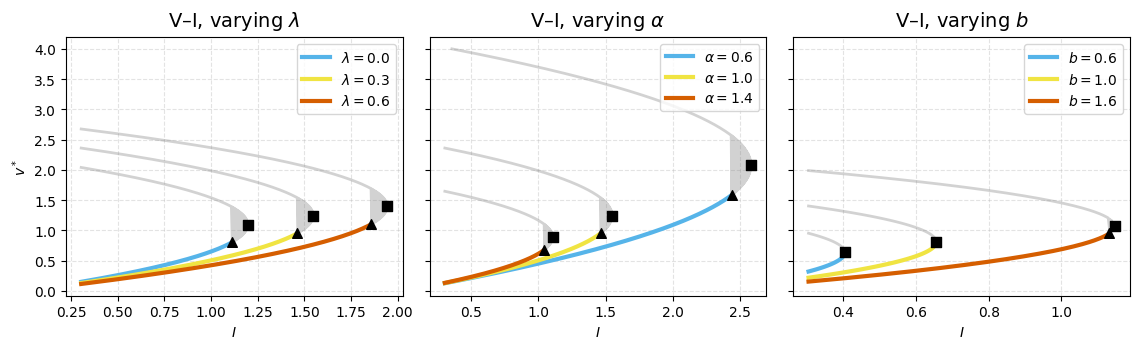}
\caption{Equilibrium $v^*$ versus input $I$ with saddle--node ($\blacksquare$) and Hopf ($\blacktriangle$) markers. 
(\textbf{Left}) Increasing service rate $\lambda$ shifts both onsets to higher $I$. 
(\textbf{Middle}) Increasing excitability $\alpha$ advances both onsets. 
(\textbf{Right}) Varying recovery coupling $b$ changes the balance between $ab$ and $a+\mu$: when $ab>(a+\mu)^2$ an oscillatory Hopf branch is admissible, otherwise only an SN appears. 
Stable branches are shown in colour; unstable in faint grey. The exact marker coordinates used here are reported in Table~\ref{tab:continuation-markers}.}
\label{fig:VI_bif}
\end{figure*}

\begin{table*}[t]
\centering
\caption{Continuation markers used in Fig.~\ref{fig:VI_bif}. 
Listed are the input levels $I_{\mathrm{SN}}$ and $I_{\mathrm{H}}$ and the corresponding equilibria $v^*_{\mathrm{SN}}, v^*_{\mathrm{H}}$. 
Rows with NaN in the Hopf columns occur where $ab\le(a+\mu)^2$, so the oscillatory onset is not admissible under the chosen parameters.}
\label{tab:continuation-markers}
\begin{tabular}{lrrrrrll}
\toprule
Panel & $\lambda$ & $\alpha$ & $b$ & $I_{\mathrm{SN}}$ & $v_{\mathrm{SN}}^*$ & $I_{\mathrm{H}}$ & $v_{\mathrm{H}}^*$ \\
\midrule
lambda sweep & 0.0 & 1.0 & 2.0 & 1.198 & 1.095 & 1.111 & 0.800 \\
lambda sweep & 0.3 & 1.0 & 2.0 & 1.549 & 1.245 & 1.462 & 0.950 \\
lambda sweep & 0.6 & 1.0 & 2.0 & 1.945 & 1.395 & 1.858 & 1.100 \\
alpha sweep  & 0.3 & 0.6 & 2.0 & 2.582 & 2.074 & 2.437 & 1.583 \\
alpha sweep  & 0.3 & 1.0 & 2.0 & 1.549 & 1.245 & 1.462 & 0.950 \\
alpha sweep  & 0.3 & 1.4 & 2.0 & 1.106 & 0.889 & 1.044 & 0.679 \\
b sweep      & 0.3 & 1.0 & 0.6 & 0.404 & 0.636 & NaN & NaN \\
b sweep      & 0.3 & 1.0 & 1.0 & 0.656 & 0.810 & NaN & NaN \\
b sweep      & 0.3 & 1.0 & 1.6 & 1.146 & 1.071 & 1.132 & 0.950 \\
\bottomrule
\end{tabular}
\end{table*}

%%%%%

\section{Equilibrium and boundedness}
\label{subsec:nos-equilibrium}

At steady state we eliminate \(u_i\) using \(u_i=\tfrac{ab}{a+\mu}v_i\) in \eqref{eq:nos-u-eq} and obtains a scalar balance:
\begin{align}
F(v^*) &:= f_{\mathrm{sat}}(v^*) + L\, v^* + C \;=\; 0,
\label{eq:nos-F}
\end{align}
where \(L\) and \(C\) collect the effective linear and constant drives implied by \eqref{eq:nos-v-eq}. The next result provides a sufficient, checkable criterion.

\begin{lemma}[Existence and uniqueness]
\label{lem:nos-exist-unique}
Let \(\mathcal{I}\subset\mathbb{R}\) be the admissible queue interval, for example \([0,1]\).
If
\begin{align}
\sup_{v\in\mathcal{I}} \bigl(f'_{\mathrm{sat}}(v)+L\bigr) &< 0,
\label{eq:nos-monotone}
\end{align}
then \(F\) is strictly \emph{decreasing} on \(\mathcal{I}\) and has at most one root.
If, in addition, there exist \(v_{-},v_{+}\in\mathcal{I}\) with \(F(v_-)\,F(v_+)\le 0\),
then a unique equilibrium \(v^*\in\mathcal{I}\) exists.

A sufficient global condition for \eqref{eq:nos-monotone} (using \( \sup_v f'_{\mathrm{sat}}(v)
\le \frac{3\sqrt{3}}{8}\frac{\alpha}{\sqrt{\kappa}} \) for \(\kappa>0\)) is
\begin{align}
L &< -\,\frac{3\sqrt{3}}{8}\,\frac{\alpha}{\sqrt{\kappa}} .
\label{eq:nos-L-sufficient}
\end{align}
\end{lemma}

\noindent\textit{Discussion.}  
Condition \eqref{eq:nos-monotone} states that net drain dominates the steepest admissible excitability slope. The bound \eqref{eq:nos-L-sufficient} follows from \eqref{eq:nos-fsat-deriv-max}. It is conservative, easy to verify from telemetry, and stable against modest estimation error.

\begin{corollary}[Algebraic condition and networking view]
\label{cor:nos-algebra}
For the coefficients in \eqref{eq:nos-v-eq}–\eqref{eq:nos-u-eq}, a sufficient condition for a unique equilibrium in \([0,1]\) is
\begin{align}
\beta - \lambda - \chi - \frac{ab}{a+\mu} &> -\,\frac{3\sqrt{3}}{8}\,\frac{\alpha}{\sqrt{\kappa}},
\label{eq:nos-cor-raw}
\end{align}
or, equivalently,
\begin{align}
\lambda + \chi + \frac{ab}{a+\mu} &> \frac{3\sqrt{3}}{8}\,\frac{\alpha}{\sqrt{\kappa}} - \beta .
\label{eq:nos-cor-service}
\end{align}
Operationally, the left side aggregates service and damping, while the right side is the worst-case excitability minus any linear relief \(\beta\). At full buffer, service must dominate arrivals; at empty buffer, leakage should not push the queue negative. Hence a single operating point exists within [0, 1].

\end{corollary}

\paragraph{Engineering rule of thumb.}
\label{par:nos-rule}
Choose \(\kappa\) so that the knee of \(f_{\mathrm{sat}}\) lies near the full-buffer scale, for example \(1/\sqrt{\kappa}\approx 1\). Fit \(\alpha\) from small-signal curvature so that \(f_{\mathrm{sat}}(v)\approx \alpha v^2\) around typical loads. Then tune \(\lambda\) and \(\chi\) to satisfy
\begin{align}
\lambda + \chi + \frac{ab}{a+\mu} \;\ge\; \frac{3\sqrt{3}}{8}\,\frac{\alpha}{\sqrt{\kappa}} - \beta \;+\; \varepsilon,
\label{eq:nos-rule}
\end{align}
with a small safety margin \(\varepsilon>0\) to absorb burstiness. This keeps the operating point in the unique, stable region and reduces sensitivity to trace noise.

\section{Non-dimensionalisation and scaling}
\label{app:nos-scaling}

Operational traces arrive with different bin widths, device rates, and buffer sizes. To make tuning portable across such heterogeneity, we rescale state and time by fixed references \(V\) and \(T\), and work with
\begin{equation}
v(t) = V\,\tilde v(\tilde t), \qquad t = T\,\tilde t,
\label{eq:nos-scale-def}
\end{equation}
so that \(\tilde v\) and \(\tilde t\) are dimensionless queue level and time. This scaling lets parameters fitted on sites with different sampling and buffer sizes be compared on a common basis and leaves the network-level stability proxy
\(
g\,\rho(W)\,+f'_{\mathrm{sat}}(v^*) < \Lambda
\)
invariant. 
Choose \(V\) as the full-buffer level used in the queue normalisation (for example, bytes or packets mapped to \(v\in[0,1]\)), and choose \(T\) as the dominant local timescale: the scheduler epoch, the service drain time \(1/\lambda\), or the telemetry bin width when that is the binding constraint. With \eqref{eq:nos-scale-def}, the left-hand side of \eqref{eq:nos-v-eq} satisfies
\begin{align}
\frac{dv}{dt} &= \frac{V}{T}\,\frac{d\tilde v}{d\tilde t},
\label{eq:nos-scale-lhs}
\end{align}
and each term on the right-hand side of \eqref{eq:nos-v-eq} is re-expressed as a dimensionless group. The bounded excitability transforms as
\begin{align}
\begin{split}
f_{\mathrm{sat}}(v) \;=\; \frac{\alpha\,v^2}{1+\kappa v^2}
\;\rightarrow\;
\tilde f_{\mathrm{sat}}(\tilde v) \\ \;=\; \frac{T}{V}\,f_{\mathrm{sat}}(V\tilde v) \;=\; 
\frac{\tilde\alpha\,\tilde v^2}{\,1+\tilde\kappa\,\tilde v^2\,},
\\
\tilde\alpha \;=\; \alpha T V,\;\; \tilde\kappa \;=\; \kappa V^2 .
\label{eq:nos-scale-fsat}
\end{split}
\end{align}

Linear and constant terms scale as
\begin{align}
\begin{split}
\tilde\beta \;=\; \beta T, \quad
\tilde\lambda \;=\; \lambda T, \quad
\tilde\chi \;=\; \chi T, \\
\tilde\gamma \;=\; \frac{\gamma T}{V}, \quad
\tilde v_{\mathrm{rest}} \;=\; \frac{v_{\mathrm{rest}}}{V}.
\label{eq:nos-scale-linear}
\end{split}
\end{align}
The recovery dynamics \eqref{eq:nos-u-eq} become
\begin{align}
\begin{split}
\frac{du}{dt} \;=\; a(bv-u) - \mu u
\;\rightarrow\;
\frac{d\tilde u}{d\tilde t} \;=\; 
\tilde a\bigl(\tilde b\,\tilde v - \tilde u\bigr) \\ - \tilde\mu\,\tilde u,\;
\tilde a \;=\; aT,\;\; \tilde\mu \;=\; \mu T,\; 
\tilde b \;=\; b,\;\; \tilde u \;=\; \frac{u}{V}
\label{eq:nos-scale-u}
\end{split}
\end{align}
and the graph-local input \eqref{eq:nos_start_input} and its delayed form \eqref{eq:nos-input-delayed} scale as
\begin{align}
\begin{split}
\tilde I_i(\tilde t) \;=\; \frac{T}{V}\,I_i(t), \;\;\;\;
\tilde \tau_{ij} \;=\; \frac{\tau_{ij}}{T},
\\
\tilde w_{ij} \;=\; w_{ij} \;\; \text{(dimensionless by construction)}.
\label{eq:nos-scale-input}
\end{split}
\end{align}
For the shot-noise drive \eqref{eq:nos-shot}, amplitudes and smoothing follow
\begin{align}
\begin{split}
\tilde \eta_i(\tilde t) = \frac{T}{V}\,\eta_i(t) \;=\; \sum_{n=1}^{N_i(t)} \tilde A_{i,n}\,
\tilde\kappa_s\bigl(\tilde t - \tilde t_{i,n}\bigr),
\\
\tilde A_{i,n} \;=\; \frac{T}{V}\,A_{i,n},\;\;
\tilde\kappa_s(\tilde t) \;=\; e^{-\tilde t/\tilde\tau_s} H(\tilde t),\;\;
\tilde\tau_s \;=\; \frac{\tau_s}{T}.
\label{eq:nos-scale-shot}
\end{split}
\end{align}

With \eqref{eq:nos-scale-def}–\eqref{eq:nos-scale-shot}, raw telemetry from sites that sample at \(1\)\,ms and sites that sample at \(10\)\,ms are mapped to the same dimensionless dynamics. The same holds for devices with different buffers: picking \(V\) from the local full-buffer level preserves the meaning of \(\tilde v\in[0,1]\). Operators can therefore compare fitted parameters across topologies and hardware: \(\tilde\lambda\) reflects service relative to the chosen time base; \(\tilde a\) reflects recovery rate in scheduler units; \(\tilde\alpha\) and \(\tilde\kappa\) capture the curvature and knee of the congestion ramp independent of absolute queue size. Delays and gates adopt the same rule: \(\tilde\tau_{ij}\) is RTT-free propagation and processing in units of \(T\), while \(w_{ij}\) stays a policy or bandwidth proportion that is already dimensionless.

\paragraph{Reading guide and forward links.}
The local slope bound in \eqref{eq:nos-fsat-deriv-max} and the service–damping aggregate in \eqref{eq:nos-cor-service} feed directly into the small-signal analysis in \ref{app:Jstability-local} and the network coupling results in \S\ref{sec:global_stability}. Let \(\rho(W)\) be the spectral radius of the coupling matrix and \(g\) any global gain applied to \(W\). Linearisation of \eqref{eq:nos-v-eq}–\eqref{eq:nos-u-eq} about an operating point \(v^*\) yields a block Jacobian whose dominant network term is proportional to \(g\,\rho(W)\, +f'_{\mathrm{sat}}(v^*)\). Compare this against the net drain
\begin{align}
\Lambda &:= \lambda + \chi + \frac{ab}{a+\mu} - \beta ,
\label{eq:nos-reading-lambda}
\end{align}
which is the left side of \eqref{eq:nos-cor-service} minus the right side at the operating point. A convenient stability proxy is
\begin{align}
g\,\rho(W) &< \Lambda - f'_{\mathrm{sat}}(v^*) ,
\label{eq:nos-reading-threshold}
\end{align}
equivalently $g\,\rho(W)+f'_{\mathrm{sat}}(v^*)<\Lambda$. Under the rescaling \eqref{eq:nos-scale-def}, both $T\,f'_{\mathrm{sat}}$ and $T\,\Lambda$ are dimensionless, so \eqref{eq:nos-reading-threshold} is invariant and can be checked in either set of units. The proxy exposes the same levers used in operations. One can reduce \(\rho(W)\) by reweighting or sparsifying couplings, reduce \(g\) through gain scheduling, or raise \(\Lambda\) by increasing service \(\lambda\), damping \(\chi\), or the recovery rate \(a\) (with the trade-off set by \(b\) and \(\mu\)). It is useful to track the \emph{operational margin}
\begin{align}
\Delta_{\mathrm{net}} &:= \Lambda - f'_{\mathrm{sat}}(v^*) - g\,\rho(W) .
\label{eq:nos-reading-margin}
\end{align}

which we will use in \S\ref{sec:global_stability} to explain bifurcation onsets, headroom under topology changes, and how weight normalisation can enforce a fixed margin across deployments. Practical rules to meet a target margin are collected in \emph{methods}~\ref{sec:nos-design-guidance}.

Operational interpretations and data-driven initialisation of \textit{NOS} parameters from observed network metrics are summarised in Table~\ref{tab:param_combined} in \emph{methods}~\ref{defvals}. The consequences of the chosen scaling for parameter admissibility, together with recommended calibration conventions, follow the same dimensionless groups.

\section{Finite-size sharpening and synchrony order parameter}
\label{sec:synchrony}
Finite networks smooth the bifurcation transitions seen in the scalar reduction. We quantify collective timing with a Kuramoto-type order parameter
\(
R(t)\;=\;\frac{1}{N}\Bigl|\sum_{j=1}^N e^{\mathrm{i}\phi_j(t)}\Bigr|,
\)
where spike phases $\phi_j(t)$ are obtained by linear phase interpolation between successive spikes of node $j$ (phase resets to $0$ at a spike, advances to $2\pi$ at the next). Figure~\ref{fig:R_vs_k} plots the time average $\langle R\rangle$ against the effective coupling $k=g\,\rho(W)$ for several network sizes $N$. The curves show the expected S-shaped rise: for small $k$ the system is desynchronised and $\langle R\rangle$ is low; once $k$ passes a critical value the order parameter increases rapidly and then saturates.

The dashed line marks the Perron-mode prediction $k^\star$ from the linear analysis of \textit{NOS}. Across $N$, measured onsets cluster near $k^\star$. Finite-size effects smooth the transition for small $N$ and sharpen it as $N$ grows, while the turning point remains anchored by $k^\star$.

\paragraph{Networking interpretation.}
Synchrony corresponds to queueing cycles that align across nodes: burst build-up and drain events become temporally coordinated. Below $k^\star$ these cycles are weakly correlated and buffers clear largely independently. Near and above $k^\star$, coupling is strong enough for delayed spikes to propagate and reinforce, producing network-wide burst trains. Because $k=g\,\rho(W)$, stability can be preserved either by reducing the local coupling gain $g$ (e.g. weight attenuation or stronger subthreshold leak) or by reducing the spectral radius $\rho(W)$ (e.g. reweighting or sparsifying high-gain pathways). Both actions move the operating point away from the synchrony threshold.

\begin{figure}[H]
\centering
\includegraphics[width=\columnwidth]{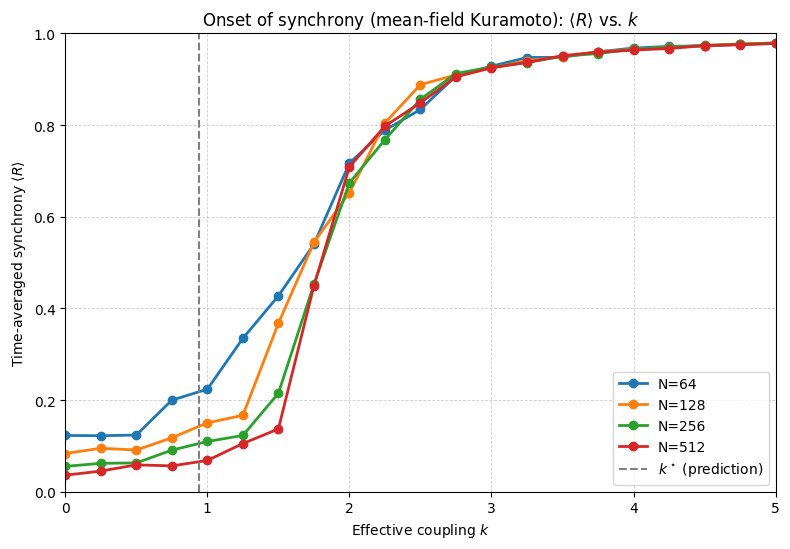}
\caption{Onset of synchrony in \textit{NOS}. Time-averaged order parameter $\langle R\rangle$ versus effective coupling $k=g\,\rho(W)$ for several network sizes $N$. The dashed line marks the Perron-mode prediction $k^\star$. The transition sharpens with $N$, while the onset remains close to $k^\star$.}
\label{fig:R_vs_k}
\end{figure}

\begin{table}[t]
\centering
\caption{Representative values of $\langle R\rangle$ at selected couplings $k$ and sizes $N$ (extracted from Fig.~\ref{fig:R_vs_k}). The increase of $\langle R\rangle$ with $k$ and the steeper change for larger $N$ reflect finite-size sharpening near $k^\star$.}
\label{tab:R_vs_k}
\begin{tabular}{ccc}
\toprule
$N$ & $k$ & $\langle R\rangle$ \\
\midrule
 64  & 0.0 & 0.100 \\
 64  & 1.0 & 0.207 \\
 64  & 2.0 & 0.822 \\
128  & 0.0 & 0.059 \\
128  & 1.0 & 0.148 \\
128  & 2.0 & 0.769 \\
256  & 0.0 & 0.061 \\
256  & 1.0 & 0.082 \\
256  & 2.0 & 0.755 \\
\bottomrule
\end{tabular}
\end{table}

\section{Zero-shot forecasting baselines (no labels)}
\label{subsec:zeroshot}

\paragraph{Setup (like-for-like, no training).}
All methods receive the same sliding window of arrival counts (last $L$ bins per node) and must predict the next queue sample $q_{t+1}$ \emph{without supervised training}.
We include: (\textbf{Phys.-Fluid}) the label-free fluid update $q_{t+1}=\mathrm{clip}(q_t+a_t-\mu\,\Delta t)$; 
(\textbf{MovingAvg}) a smoothed fluid variant; 
(\textbf{TGNN-smooth}) a single-hop temporal–graph smoother applied to arrivals before the fluid step; 
(\textbf{LIF-leaky}) a leaky-integrator forecaster; and 
(\textbf{NOS (calibr.)}) a single \emph{NOS} unit per node with \emph{arrival-only} calibration (no queue labels): a global input gain and per-node output scales are fitted to match light-load analytic means from M/M/1. 
In the run summarized here we obtained $\texttt{gain\_I}=0.60$ and the first three per-node scales $\{11.646,\,12.226,\,12.277\}$; the corresponding analytic light-load means for those nodes were $\{0.655,\,0.738,\,0.745\}$ (reported for reference only).

\paragraph{What the figure shows.}
Figure~\ref{fig:zeroshot_forecasts} overlays the \emph{true} queue on node~0 (black) with the five zero-shot forecasts. 
Phys.-Fluid (blue) tracks the envelope of bursts because it integrates the \emph{same} arrivals that drive the queueing simulator; 
MovingAvg (orange) lags and underestimates peaks; 
TGNN-smooth (green) damps spikes via spatial smoothing; 
LIF-leaky (red) behaves as a low-pass filter; 
NOS (purple) fires promptly at burst onsets and then resets, producing compact predictions around excursions.

\begin{figure*}[h]
  \centering
  \includegraphics[width=.98\textwidth]{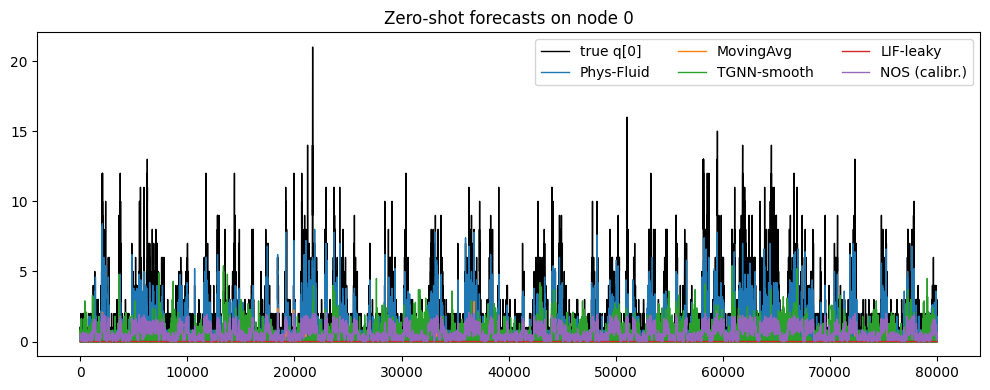}
  \caption{Zero-shot next-step forecasts on a single node (no supervised training). 
  All methods consume the same arrival windows; only \emph{arrival-only} calibration is used.
  \emph{NOS} produces compact, onset-aligned responses; the fluid forecaster follows backlog accumulation.}
  \label{fig:zeroshot_forecasts}
\end{figure*}

\paragraph{Numerical summary.}
Table~\ref{tab:zeroshot} reports point error (MAE) and event skill (AUROC / AUPRC for top-10\% bursts), computed on the same held-out segment. 
Phys.-Fluid attains the smallest MAE ($0.675$)---expected, since the simulator itself is fluid-like at the bin scale. 
NOS has a larger MAE ($0.928$) because it is \emph{bounded} and self-resets (it does not accumulate backlog), yet it delivers strong event skill (AUROC $0.894$, AUPRC $0.536$), on par with Phys.-Fluid and substantially above simple smoothers. 
Operationally, this means \emph{NOS} is well-tuned to \emph{detect and time} congestion onsets without drifting into long false positives between bursts. 
TGNN-smooth and LIF-leaky reduce noise but lose timing contrast (near-chance event skill).

\begin{table}[t]
\centering
\caption{Zero-shot next-step forecasting (no supervised training). 
All methods use the same arrival windows; \emph{NOS} uses only arrival-based calibration.}
\label{tab:zeroshot}
\begin{tabular}{lccc}
\toprule
Method & MAE & AUROC & AUPRC \\
\midrule
Physics Fluid & 0.6748 & 0.834 & 0.555 \\
Moving Avg    & 0.8782 & 0.552 & 0.194 \\
TGNN-smooth   & 0.9251 & 0.507 & 0.128 \\
LIF-leaky     & 0.8891 & 0.500 & 0.126 \\
NOS (calibr.) & 0.9278 & 0.894 & 0.536 \\
\bottomrule
\end{tabular}
\end{table}

\paragraph{Networking interpretation.}
In the absence of labels, a fluid predictor provides a tough MAE baseline because it integrates exactly the arrivals that create the queue, essentially reproducing backlog drift. 
However, operators often care more about \emph{timely onset detection} than about matching every backlog micro-fluctuation. 
Here NOS’ bounded, event-driven dynamics produce high AUROC/AUPRC---accurate timing of burst onsets and compact responses---while avoiding prolonged marking during recoveries. 
This complements the open/closed-loop results: use fluid logic for gross throughput accounting, and deploy \emph{NOS} to deliver low-latency, low-ringing congestion signals aligned with control actions.

%%%%%%%%% S start cited already
\section{Evaluation protocol for forecasting baselines}
\label{sec:evaluation-protocol}

This section details the protocol used for the comparisons in ~\ref{sec:ML-comparison}. For metrics we use networks of size
$N{=}250$; the range plots show node~0 for visual clarity. All baselines are trained label-free on next-step forecasting with a contiguous train/validation/test split. Inputs are standardised per node using statistics fitted on the train split only. Residuals are signed and $z$-scored with train-only moments; event starts are the first samples whose residual $z$-scores cross a fixed per-node threshold selected on the validation split but computed using train calibration. Episodes have a minimum duration to suppress single-sample blips; matching uses a fixed window around each ground-truth start. Start-latency is the sample difference between the truth
start and the first model start within the window, converted to milliseconds using the sampling period $\Delta t$. Forecasting error (MAE, RMSE) is computed on the original scale. The tGNN uses temporal message passing constrained to the given adjacency; capacities and early-stopping are held comparable across baselines.

All learned baselines, namely MLP, RNN, GRU, and tGNN, are trained
self-supervised on next-step forecasting, and events are inferred from forecast error rather than labels. Train-only calibration prevents leakage and keeps thresholds comparable across models and topologies. The evaluation can be extended with extreme-value thresholds on residual tails, rolling-origin splits to test drift, precision–recall curves, episode-wise latency in milliseconds, and an out-of-distribution topology block, without changing the core protocol.

\section{\textit{NOS} simulation pseudocode}
\label{pseudocode}

For reproducibility we include the simulation routine used in our experiments. We simulate \textit{NOS} in discrete time with a fixed step $dt$ (forward Euler). Per-edge propagation delays are implemented using delay buffers that store scheduled spike arrivals. Each time step consists of four stages: (i) deliver scheduled arrivals from buffers, (ii) compute the total input from delivered arrivals plus exogenous drive, (iii) update $(v,u)$ using bounded excitability with two leak terms, and (iv) apply thresholding with optional jitter and a soft exponential reset. Throughout, $S^{\mathrm{del}}_{j\to i}(t)\in\{0,1\}$ denotes the spike arrival delivered to node $i$ on edge $(j\to i)$ at time step $t$.

We use both leak terms in the membrane update: $-\lambda v$ represents a direct drain proportional to the current load, while $-\chi(v-v_{\mathrm{rest}})$ pulls the state back toward the resting level. A bounded excitability nonlinearity $f_{\mathrm{sat}}(\cdot)$ and an optional clamp $v\in[v_{\mathrm{rest}},v_{\max}]$ prevent runaway growth and reflect finite buffer capacity. The exogenous drive $\eta_i(t)$ may be implemented as shot noise (arrival impulses) or as additive noise; the pseudocode treats it abstractly.

\begin{algorithm}[ht!]
\caption{Discrete-time simulation of \emph{NOS} dynamics with per-edge delays and soft resets}
\label{alg:NOS}
\begin{algorithmic}[1]
\STATE \textbf{Initialise:} for each node $i$, set $v_i \gets v_{\mathrm{rest}}+\epsilon$, $u_i \gets u_0$
\STATE \textbf{Initialise:} for each directed edge $(j\to i)$, create a delay buffer of length $\lceil \tau_{ij}/dt\rceil$
\FOR{each time step $t$}
  \STATE \textbf{(1) Deliver delayed spikes:} set $I^{\mathrm{syn}}_i \gets 0$ for all $i$
  \FOR{each directed edge $(j\to i)$}
    \STATE Pop delivered arrival $S^{\mathrm{del}}_{j\to i}(t)$ from the buffer for $(j\to i)$
    \STATE Accumulate synaptic input: $I^{\mathrm{syn}}_i \gets I^{\mathrm{syn}}_i + w_{ij}\,g(q_{ij},t)\,S^{\mathrm{del}}_{j\to i}(t)$
  \ENDFOR
  \STATE \textbf{(2) Total input:} for each node $i$, set $I_i(t) \gets I^{\mathrm{syn}}_i + \eta_i(t)$

  \STATE \textbf{(3) State update (forward Euler):} for each node $i$
  \(
  v_i^{\text{old}} \gets v_i,\qquad u_i^{\text{old}} \gets u_i
  \)
  \(
  v_i \gets v_i^{\text{old}} + dt\cdot\Bigl(f_{\mathrm{sat}}(v_i^{\text{old}})+\beta v_i^{\text{old}}+\gamma - u_i^{\text{old}} + I_i(t)
  - \lambda v_i^{\text{old}} - \chi\bigl(v_i^{\text{old}}-v_{\mathrm{rest}}\bigr)\Bigr)
  \)
  \(
  u_i \gets u_i^{\text{old}} + dt\cdot\Bigl(a\bigl(b v_i^{\text{old}} - u_i^{\text{old}}\bigr) - \mu u_i^{\text{old}}\Bigr)
  \)
  \STATE (Optional) clamp: $v_i \gets \mathrm{clip}(v_i, v_{\mathrm{rest}}, v_{\max})$, $u_i \gets \mathrm{clip}(u_i, u_{\min}, u_{\max})$

  \STATE \textbf{(4) Threshold and reset:} for each node $i$
  \(
  v_{\mathrm{th},i}(t) \gets v_{\mathrm{th}} + \sigma\,\xi_i(t)\qquad (\xi_i(t)\sim\mathcal{N}(0,1),\ \sigma=0 \text{ for deterministic runs})
  \)
  \IF{$v_i \ge v_{\mathrm{th},i}(t)$}
    \STATE Emit spike at time $t$ and record it as $S_i(t)=1$
    \FOR{each outgoing edge $(i\to k)$}
      \STATE Push an arrival into the buffer for $(i\to k)$ scheduled at $t+\tau_{ik}$
    \ENDFOR
    \STATE Soft exponential reset:
    \(
      v_i \gets c + (v_i-c)\,e^{-r_{\mathrm{reset}} dt}, \qquad u_i \gets u_i + d
    \)
  \ENDIF
\ENDFOR
\end{algorithmic}
\end{algorithm}

\begin{verbatim}
# NOS simulation with per-edge delay buffers (forward Euler)

# Parameters: dt, T
# States: v[i], u[i]
# Buffers: delay_buffer[i->k] stores scheduled arrivals (0/1) in integer bins

# Initialisation
for each node i:
  v[i] = v_rest + small_noise()
  u[i] = u0

for each directed edge (j -> i):
  L_ij = ceil(tau_ij / dt)     
  delay_buffer[j->i] = RingBuffer(length=L_ij)

# Main loop
for t in time_steps:  # physical time is t*dt
  # (1) deliver scheduled arrivals and build synaptic input
  for each node i:
    I_syn[i] = 0

  for each directed edge (j -> i):
    S_del = delay_buffer[j->i].pop()              
    I_syn[i] += W[i,j] * gate(q_ij, t) * S_del

  # (2) exogenous drive
  for each node i:
    I_ex[i] = eta(i, t)                           
    I[i] = I_syn[i] + I_ex[i]

  # (3) Euler update from old state (uses both leak terms)
  for each node i:
    v_old = v[i]
    u_old = u[i]

    v_new = v_old + dt * ( f_sat(v_old) + beta*v_old + gamma - u_old
                           + I[i] - lambda*v_old - chi*(v_old - v_rest) )
    u_new = u_old + dt * ( a*(b*v_old - u_old) - mu*u_old )

    v[i] = clip(v_new, v_rest, v_max)             
    u[i] = clip(u_new, u_min, u_max)              
  for each node i:
    v_th_eff = v_th + sigma * randn()             
    if v[i] >= v_th_eff:
      emit_spike(i, t*dt)
      for each outgoing edge (i -> k):
        delay_buffer[i->k].push(1)                
      v[i] = c + (v[i] - c) * exp(-r_reset * dt)
      u[i] = u[i] + d

# Record observables as needed (rates, ISI CV, avalanches, synchrony).
\end{verbatim}

\bibliographystyle{plain}
\bibliography{references}

\end{document}